\documentclass{article}

\usepackage[final,nonatbib]{neurips_2022}

\usepackage{algorithm}

\usepackage[utf8]{inputenc} % allow utf-8 input
\usepackage[T1]{fontenc}    % use 8-bit T1 fonts
\usepackage{hyperref}       % hyperlinks
\usepackage{url}            % simple URL typesetting
\usepackage{booktabs}       % professional-quality tables
\usepackage{amsfonts}       % blackboard math symbols
\usepackage{nicefrac}       % compact symbols for 1/2, etc.
\usepackage{microtype}      % microtypography
\usepackage{graphicx}
\usepackage{amsthm}
\usepackage{amsmath}
\usepackage{amssymb}
\usepackage{bbm}
\usepackage[utf8]{inputenc}
\usepackage[english]{babel}
\usepackage{subcaption}
\usepackage[noend]{algorithmic}

\usepackage{makecell}

\usepackage[dvipsnames]{xcolor}

\usepackage[toc,page]{appendix}

\usepackage{commath}

\usepackage{enumitem}

\usepackage{cleveref}
\crefformat{footnote}{#2\footnotemark[#1]#3}

\newtheorem{theorem}{Theorem}

\newtheorem{lemma}{Lemma}
\newtheorem{proposition}{Proposition}

\newtheorem{definition}{Definition}

\newcommand{\squishlisttwo}{
 \begin{list}{$\bullet$}
  { \setlength{\itemsep}{1pt}
     \setlength{\parsep}{0pt}
    \setlength{\topsep}{0pt}
    \setlength{\partopsep}{0pt}
    \setlength{\leftmargin}{1em}
    \setlength{\labelwidth}{1.5em}
    \setlength{\labelsep}{0.5em} } }
\newcommand{\squishend}{
  \end{list}  }

\let\svthefootnote\thefootnote
\newcommand\blankfootnote[1]{%
  \let\thefootnote\relax\footnotetext{#1}%
  \let\thefootnote\svthefootnote%
}

\title{Sample-Then-Optimize Batch Neural \\
Thompson Sampling}

\author{%
  Zhongxiang Dai$^{\dagger}$, Yao Shu$^{\dagger}$, Bryan Kian Hsiang Low$^{\dagger}$, Patrick Jaillet$^{\S}$\\
  Dept. of Computer Science, National University of Singapore, Republic of Singapore$^{\dagger}$\\
  Dept. of Electrical Engineering and Computer Science, MIT, USA$^{\S}$\\
  \texttt{\{daizhongxiang,shuyao,lowkh\}@comp.nus.edu.sg$^{\dagger}$,jaillet@mit.edu$^{\S}$}\\
}

\begin{document}

\maketitle

\begin{abstract}
\emph{Bayesian optimization} (BO), which uses a \emph{Gaussian process} (GP) as a surrogate to model its objective function, is popular for black-box optimization. However, due to the limitations of GPs, BO underperforms in some problems such as those with categorical, high-dimensional or image inputs. To this end, recent works have used the highly expressive \emph{neural networks} (NNs) as the surrogate model and derived theoretical guarantees using the theory of \emph{neural tangent kernel} (NTK). However, these works suffer from the limitations of the requirement to invert an extremely large parameter matrix and the restriction to the sequential (rather than batch) setting. To overcome these limitations, we introduce two algorithms based on the \emph{Thompson sampling} (TS) policy named \emph{Sample-Then-Optimize Batch Neural TS} (STO-BNTS) and STO-BNTS-Linear. To choose an input query, we only need to train an NN (resp.~a linear model) and then choose the query by maximizing the trained NN (resp.~linear model), which is equivalently sampled from the GP posterior with the NTK as the kernel function. 
As a result, our algorithms sidestep the need to invert the large parameter matrix yet still preserve the validity of the TS policy. Next, we derive regret upper bounds for our algorithms with batch evaluations, and use insights from batch BO and NTK to show that they are \emph{asymptotically no-regret} under certain conditions. Finally, we verify their empirical effectiveness using practical AutoML and reinforcement learning experiments.
\end{abstract}

\section{Introduction}
\label{sec:introduction}
% \let\thefootnote\relax\footnotetext{*Correspondence to: Yao Shu <shuyao@comp.nus.edu.sg>.}
% \blfootnote{Correspondence to: Yao Shu <shuyao@comp.nus.edu.sg>.}
\blankfootnote{Correspondence to: Yao Shu <shuyao@comp.nus.edu.sg>.}
\emph{Bayesian optimization} (BO), also called \emph{Gaussian process} (GP) bandits,
% since it is a bandit algorithm using a GP as a surrogate to model the reward function, 
has become a celebrated method for optimizing expensive-to-compute black-box functions, primarily thanks to its practical sample efficiency and theoretically guaranteed convergence~\cite{dai2020federated,dai2019,frazier2018tutorial,shahriari2016taking}.
However, there are important problem settings where BO either underperforms or is not even applicable without sophisticated modifications, such as problems with categorical~\cite{deshwal2021bayesian}, high-dimensional~\cite{kandasamy2015high}, or images inputs.
%sy:
% However, there exists several out-of-scope problem settings where BO underperforms or even can not be applied naturally, such as the problems with categorical inputs~\cite{deshwal2021bayesian}, high-dimensional inputs~\cite{kandasamy2015high}, or image inputs.
These issues have arisen mainly because GPs (i.e., the surrogate used by BO to model the objective function) are not able to effectively model these types of input space, which therefore calls for the use of alternative surrogate models in BO.
%sy:
% These issues arise mainly because the surrogate model (i.e., GPs)  used by BO to model the objective function are not able to handle these types of inputs effectively, which therefore calls for alternative surrogate models in BO.
To this end, \emph{neural networks} (NNs) serve as a natural candidate owing to their remarkable expressivity~\cite{lecun2015deep}.
NNs have repeatedly proven their ability to model extremely complicated real-world functions such as those involving categorical, high-dimensional or image inputs, whereas the development of GPs to effectively model these functions still represents active areas of research.
In this regard,~\cite{zhou2020neural} have adopted NNs as the surrogate model in contextual bandit problems and employed the theoretical framework of \emph{neural tangent kernel} (NTK)~\cite{jacot2018neural} to construct a principled algorithm following the well-known policy of \emph{upper confidence bound} (UCB), hence introducing the \emph{Neural UCB} algorithm. 
More recently,~\cite{zhang2020neural} have extended Neural UCB to follow the \emph{Thompson sampling} (TS) policy and proposed the \emph{Neural TS} algorithm. 
Both Neural UCB and Neural TS are equipped with upper bounds on their cumulative regret and perform competitively in real-world contextual bandit experiments.

However, Neural UCB and Neural TS are still faced with several important limitations that may hinder their practical applications.
% both Neural UCB and Neural TS 
Firstly, these algorithms suffer from the
%an important drawback due to their
requirement to invert a $p\times p$ matrix in every iteration, in which $p$ is the number of parameters of the NN surrogate and is usually extremely large since the theory of NTK requires severe overparameterization.
% where $p$ is the parameter size of their NN surrogate model which can be extremely large in an overparameterized NN required by the NTK theory.
% of the NN.
In practice, a diagonal approximation is used to avoid the need to invert such a large $p\times p$ matrix~\cite{zhang2020neural,zhou2020neural}, which introduces approximation errors in their practical deployment that are unaccounted for in the theoretical analysis, hence causing a disparity between theory and practice.
%sy:
% Consequently, a diagonal approximation is applied in~\cite{zhou2020neural,zhang2020neural} to reduce the burden of inverting this extremely large $p\times p$ matrix in practice, which however introduces the approximation errors that are missing in their theoretical analysis and therefore causes a disparity between their theoretical analyses and practical applications.
Secondly, the theoretical analyses of these algorithms are restricted to the sequential setting and are hence not applicable in the batch setting where an entire batch of inputs are selected for querying.\footnote{The work of~\cite{gu2021batched} 
focuses on the \emph{contextual bandit} setting and aims to choose a batch of inputs \emph{given a batch of diverse contexts}. So, their methods are not applicable in our setting of BO where the contexts are fixed in all iterations, since they do not explicitly encourage input diversity which is a crucial problem in batch BO~\cite{contal2013parallel,desautels2014parallelizing}.}
% focuses on reducing the frequency of updating the NN parameters in the \emph{contextual} bandit setting, and hence does not resolve the issue of input diversity within a batch which is a crucial problem in the (non-contextual) setting of batch BO~\cite{contal2013parallel,desautels2014parallelizing}.}
% In this work, 
To overcome these two limitations,
% by drawing inspirations from sample-then-optimize optimization~\cite{matthews2017sample} and Bayesian deep ensembles~\cite{he2020bayesian}, 
we introduce two algorithms based on the TS policy named \emph{Sample-Then-Optimize Batch Neural Thompson Sampling} (STO-BNTS) and STO-BNTS-Linear, 
% we introduce the \emph{Sample-Then-Optimize Batch Neural Thompson Sampling} (STO-BNTS) framework, 
both of which \emph{(1)} sidestep the need to invert the $p\times p$ matrix and hence close the gap between theory and practice, and \emph{(2)} naturally support batch evaluations while preserving the theoretical guarantees.
% As a result, our resulting algorithms are also more practical, hence we can evaluate it in real-world hyper tuning experiments.

% \textbf{
% How our two algorithms overcome the first limitation, i.e., how do they sidestep the need to invert the big matrix yet still preserve the validity of the TS policy (i.e., still make sure that the acquisition function is sampled from a GP posterior).
% }

To avoid the inversion of the $p\times p$ matrix 
% (i.e., the first challenge) 
while still ensuring the validity of the TS policy, we draw inspirations from sample-then-optimize optimization~\cite{matthews2017sample} and Bayesian deep ensembles~\cite{he2020bayesian} to efficiently sample functions from the GP posterior (with the NTK as the kernel function) without inverting the $p\times p$ matrix.
Specifically, to choose an input query within a batch, our STO-BNTS (resp.~STO-BNTS-Linear) only needs to use the current observation history to train an \emph{NN surrogate} (resp.~a linear model defined w.r.t.~the random features embedding of the NTK associated with an \emph{NN surrogate}) using randomly initialized parameters, and then choose the input that maximizes the trained NN (resp.~trained linear model).
As a result, if the NN surrogate has an infinite width, the function of the trained NN (resp.~trained linear model) is equivalently sampled from the GP posterior with the NTK as the kernel function.
This ensures that our algorithms follow the TS policy and hence lays the foundation for our theoretical analyses.
Next, to address the second challenge of deriving theoretical guarantees for our algorithms in the batch setting, we generalize the theoretical analysis of sequential TS~\cite{chowdhury2017kernelized} to account for batch evaluations and derive a regret upper bound for both of our algorithms when the NN surrogate is infinite-width.
Then, we leverage insights from batch BO~\cite{desautels2014parallelizing} and NTK~\cite{kassraie2021neural}  to show that the regret upper bound is sub-linear (under some conditions), which implies that our algorithms are \emph{asymptotically no-regret}.
Next, when the NN surrogate is \emph{finite-width}, we derive a regret upper bound for our STO-BNTS-Linear by carefully accounting for the approximation error caused by the use of a finite (instead of infinite) NN, and show that 
% STO-BNTS-Linear remains asymptotically no-regret
the regret upper bound of STO-BNTS-Linear remains sub-linear
as long as the NN is wide enough.

Our contributions are summarized as follows:
% \vspace{-2mm}
% \begin{itemize}[noitemsep]
\squishlisttwo
% \begin{itemize}
% \itemsep-0.8mm
    \item Our STO-BNTS and STO-BNTS-Linear algorithms sidestep the 
    % approximation and 
    inversion of the $p\times p$ matrix required by Neural UCB and Neural TS, which closes their gap between theory and practice.
    % and avoids the errors in practice due to the diagonal approximation of the matrix.
    \item Our algorithms naturally support batch evaluations with theoretical guarantees.
    % naturally applicable to the batch setting thanks to the inherent randomness of the TS policy.
    \item Our algorithms are equipped with an upper bound on their cumulative regret when the NN surrogate is infinite-width, and are asymptotically no-regret (i.e., their regret upper bound is sub-linear) under certain conditions. 
    % Moreover, our STO-BNTS-Linear still enjoys a regret upper bound when the NN surrogate is finite-width and remains asymptotically no-regret as long as the NN is wide enough.
    Moreover, when the NN surrogate is finite-width, our STO-BNTS-Linear still enjoys a regret upper bound, which remains sub-linear as long as the NN is wide enough.
    \item We
%    show 
    demonstrate our empirical effectiveness 
%    of our algorithms 
    in real-world experiments including \emph{automated machine learning} (AutoML) and \emph{reinforcement learning} (RL) tasks, as well as a task on optimization over images. To the best of our knowledge, our experiments (Sec.~\ref{sec:exp}) are the first empirical study 
%    in the literature 
    to show the advantage of neural bandit over GP bandit algorithms in practical AutoML and RL tasks.
% \end{itemize}
\squishend

\vspace{-1.5mm}
\section{Background}
\label{sec:background}
\vspace{-1.5mm}
%\subsection{Neural Networks and Neural Tangent Kernel}
\textbf{Neural Networks and Neural Tangent Kernel.}
In this work, we adopt the same construction of a neural network (NN) as~\cite{arora2019exact}.
We use $f(\mathbf{x}; \theta)$ to denote the scalar output of an $(L+1)$-layer 
%neural network (NN) 
NN
with parameters $\theta \in \mathbb{R}^p$ and input $\mathbf{x}$, and use $\nabla_{\theta}f(\mathbf{x}; \theta')$ to represent the 
% gradient of the NN w.r.t.~the parameters $\theta$ at $\theta'$.
gradient of the NN evaluated at $\theta=\theta'$.
For simplicity, we assume every layer of the NN has the same width and represent the width as $m$.
We denote our initializaiton scheme for $\theta$ as $\text{init}(\cdot)$ which simply independently samples every NN parameter from the standard Gaussian distribution.
% Refer to~\cite{arora2019exact} for more details on NNs construction.
%sy: how about the reference for this formulation?
% We formulate an $(L+1)$-layer neural network (NN) with the ReLU activation function as
% \begin{equation}
% \begin{split}
% &f^{(1)} = \mathbf{W}^{(1)}\mathbf{x},\\
% &f^{(l)}=\sqrt{2/m}\mathbf{W}^{(l)}\text{ReLU}(f^{(l-1)}),\forall l=2,\ldots,L,\\
% &f(\mathbf{x};\theta)=\sqrt{2/m}\mathbf{W}^{(L+1)}\text{ReLU}(f^{(L)}),
% \end{split}
% \label{eq:nn:definition}
% \end{equation}
% where $\mathbf{x}$ stands for the input to NN, $m$ denotes the width of every layer in NN and $\theta=\{\mathbf{W}^{(l)}\}_{l=1}^{L+1}\in\mathbb{R}^p$ represents the collection of all NN parameters.
% We use $f(\mathbf{x}; \theta)$ to represent the scalar output of an NN with parameters $\theta$ and input $\mathbf{x}$, and use $\nabla_{\theta}f(\mathbf{x}; \theta')$ to denote the 
% % gradient of the NN w.r.t.~the parameters $\theta$ at $\theta'$.
% gradient of NN evaluated at $\theta=\theta'$.
% We denote our initializaiton scheme as $\text{init}(\cdot)$ which simply samples every network parameter from the standard Gaussian distribution independently.
The NTK~\cite{jacot2018neural}
% which has been proposed to characterize the training dynamics of NNs through the lens of kernel learning, 
%sy: the following seems not correct
provides an explicit connection between NNs trained via gradient descent and kernel regression using NTK as the kernel function~\cite{arora2019exact}. 
% For example, it has been shown that a neural network (in the infinite width limit) trained till convergence by gradient descent is the same as the GP posterior mean calculated using the NTK covariance function.
The NTK matrix, denoted as $\Theta$, has been shown to stay constant during the course of training as the width $m$ of the NN approaches infinity~\cite{jacot2018neural}.
Moreover, $\Theta$ can be approximated by an \emph{empirical NTK} $\widetilde{\Theta}$~\cite{arora2019exact} calculated using a finite-width NN:
% $f(\mathbf{x};\theta)$:
$\widetilde{\Theta}(\mathbf{x},\mathbf{x}') \triangleq \langle \nabla_{\theta}f(\mathbf{x};\theta_0), \nabla_{\theta}f(\mathbf{x}';\theta_0) \rangle \approx \Theta(\mathbf{x}, \mathbf{x}')$, where $\theta_0\sim\text{init}(\cdot)$ denotes initial parameters and $\nabla_{\theta}f(\mathbf{x};\theta_0)$ is referred to as the \emph{neural tangent features}~\cite{zhang2020neural}.
% The approximation quality improves as the width of the neural network $f(\mathbf{x};\theta)$ is increased  (Proposition~\ref{prop:arora}).
% As a result, 
% the gradient at initialization 
% $\nabla_{\theta}f(\mathbf{x};\theta_0)$ can be viewed as the random features that can be used to approximate the NTK and are hence referred to as \emph{neural tangent features}~\cite{zhang2020neural}.
% To simplify the analysis, \textbf{we assume} w.l.o.g.~that $\Theta(\mathbf{x},\mathbf{x}') \leq 1,\forall \mathbf{x},\mathbf{x}'$. A more relaxed condition of $\Theta(\mathbf{x},\mathbf{x}') \leq K_0,\forall \mathbf{x},\mathbf{x}'$ for some absolute constant $K_0$ has been proved in previous works (which we will discuss more in the Appendix), which will only introduce a multiplicative constant of $K_0$ to the resulting regret upper bound.
We refer the readers to the works of~\cite{arora2019exact,jacot2018neural} for a more detailed background on NTK.

%\vspace{-1mm}
% \subsection{Batch Neural Bandit}
%\subsection{Problem Setting}
%\label{subsec:background:problem:setting}
\textbf{Problem Setting.}
We aim to maximize a black-box function $f: \mathcal{X} \rightarrow \mathbb{R}$, i.e., find $\mathbf{x}^* \in {\arg\max}_{\mathbf{x}\in\mathcal{X}}f(\mathbf{x})$, in which the domain $\mathcal{X}$ is a finite subset of the $d$-dimensional unit ball: $\mathcal{X}\subset\{\mathbf{x} | \norm{\mathbf{x}}_2\leq1\}$.
Of note, our theoretical results allow $\mathcal{X}$ to be very large because our regret upper bounds only depend on its cardinality $|\mathcal{X}|$ logarithmically.
Moreover, all our theoretical results can be easily extended to problems with continuous input domains with an additional assumption on the Lipschitz continuity of $f$ (Appendix \ref{app:sec:extension:continuous:domain}).
We focus on the noisy setting, i.e., for every queried $\mathbf{x}$, we observe a noisy output $y(\mathbf{x})=f(\mathbf{x})+\zeta$ where $\zeta \sim \mathcal{N}(0,\sigma^2)$.
%In our theoretical analysis, our primary assumption is that the objective function $f$ is sampled from a GP with the NTK $\Theta$ as the kernel function, which is a common assumption in the analysis of BO algorithms~\cite{srinivas2009gaussian,desautels2014parallelizing}.
%%sy:
%%In our theoretical analysis, our primary assumption is that the objective function $f$ is sampled from a GP with NTK $\Theta$ as the kernel function, which has been commonly applied in the analysis of BO algorithms~\cite{srinivas2009gaussian,desautels2014parallelizing}.
%In line with~\cite{kassraie2021neural}, we assume that $\Theta(\mathbf{x},\mathbf{x}') \leq K_0,\forall \mathbf{x},\mathbf{x}' \in \mathcal{X}$ for some $K_0>0$.
%In addition, we also assume that 
%% the function value is bounded, i.e., 
%$|f(\mathbf{x})| \leq B'$ for some $B'>0$, which is also a mild assumption 
%% in practice 
%since most practical objective functions have bounded values.
For simplicity, we focus on the setting of \emph{synchronous} batch BO with a batch size $B$
% a batch of $B\geq1$ input queries are selected and queried in every iteration, and 
where a new batch of $B$ inputs are selected only after all evaluations of the previous batch are completed~\cite{desautels2014parallelizing}.
However, our theoretical results also hold for asynchronous batch BO where a new input query is selected once any pending query is completed (Appendix~\ref{app:sec:infinite:width}).
%We analyze the \emph{cumulative regret} of our algorithms: $R_T=\sum^{T/B}_{t=1}\sum^{B}_{i=1}(f(\mathbf{x}^*)-f(\mathbf{x}^i_t))$ where $\mathbf{x}^i_t$ is the $i^{\text{th}}$ selected input in iteration $t$. An algorithm is \emph{asymptotically no-regret} if $R_T$ is sub-linear in $T$, since it implies that the simple regret $S_T=\min_{t,i}(f(\mathbf{x}^*)-f(\mathbf{x}^i_t)) \leq R_T/T$ goes to $0$ asymptotically.
We denote the $i^{\text{th}}$ selected input in iteration $t$ as $\mathbf{x}^i_t$.
We analyze the \emph{cumulative regret} of our algorithms: $R_T=\sum^{T/B}_{t=1}\sum^{B}_{i=1}(f(\mathbf{x}^*)-f(\mathbf{x}^i_t))$,
%where $\mathbf{x}^i_t$ is the $i^{\text{th}}$ selected input in iteration $t$, 
because if $R_T$ is shown to be sub-linear in $T$, then the simple regret $S_T=\min_{t,i}(f(\mathbf{x}^*)-f(\mathbf{x}^i_t)) \leq R_T/T$ goes to $0$ asymptotically, which implies that our algorithm is \emph{asymptotically no-regret}.

\vspace{-1mm}
\section{Sample-Then-Optimize Batch Neural Thompson Sampling}
\label{sec:algo}
\vspace{-1mm}
Our STO-BNTS and STO-BNTS-Linear algorithms are presented in Algos.~\ref{algo:1} and~\ref{algo:2}.
In both algorithms, the NN surrogate $f(\mathbf{x};\theta)$ can be either infinite-width or finite-width.
% \footnotetext{In our theoretical analysis, we assume that we train the NN surrogate $\beta_t f(\mathbf{x};\theta)$ (instead of the original NN of $f(\mathbf{x};\theta)$) and set the noise variance to be $\beta_t^2\sigma^2$.}
Both STO-BNTS and STO-BNTS-Linear follow the TS policy to 
% follow the same overall structure to 
select an input query $\mathbf{x}^i_t$: They firstly (\emph{a}) obtain a function $f^i_t(\mathbf{x};\theta^i_t)$ which is equivalently sampled from the GP posterior with the NTK as the kernel: $\mathcal{GP}(\mu_{t-1}(\cdot),\beta_t^2\sigma^2_{t-1}(\cdot,\cdot))$~\cite{he2020bayesian} 
(see Appendix~\ref{app:sec:justification:gp:sampling} for details), and then (\emph{b}) maximize the function to select the next query: $\mathbf{x}^i_t = {\arg\max}_{\mathbf{x}\in\mathcal{X}} f^i_t(\mathbf{x};\theta^i_t)$.
Step (\emph{a}) is achieved via the sample-then-optimize procedure, i.e., by firstly \emph{sampling} initial parameters ($\theta_0$ and $\theta_0'$) to construct a function $f^i_t(\mathbf{x};\theta)$, and then \emph{optimizing} the function using gradient descent 
% till convergence 
to obtain the resulting function of $f^i_t(\mathbf{x};\theta^i_t)$.

% Importantly, if the NN $f(\mathbf{x};\theta)$ is infinite-width, the function $f^i_t(\mathbf{x};\theta^i_t)$ obtained after the training in line 7 is \emph{a sample from the GP posterior with the NTK as the kernel}: 
% $\mathcal{GP}(\mu_{t-1}(\cdot),\beta_t^2\sigma^2_{t-1}(\cdot,\cdot))$~\cite{he2020bayesian} (Appendix~\ref{app:sec:justification:gp:sampling}). 

% Before training the NN using the current observation history $\mathcal{D}_{t-1}$, we follow the works of~\cite{he2020bayesian,osband2018randomized} to add a zero-mean Gaussian noise to the observations, i.e., $\widetilde{y}^i_{\tau} = y^i_{\tau} + \epsilon'$ where $\epsilon'\sim\mathcal{N}(0,\sigma^2)$. This allows us to sample functions from GP posteriors with \emph{noisy} observations.
% \begin{equation}
% \mathcal{L}_t(\theta,\mathcal{D}_{t-1}) = \sum^{t-1}_{\tau=1}\sum^B_{i=1}(\widetilde{y}^i_{\tau} - f(\theta))^2 + \sigma^2\norm{\theta - \theta_0}^2_2
% \label{eq:loss:function}
% \end{equation}

%\paragraph{STO-BNTS (Algo.~\ref{algo:1}).}
\textbf{STO-BNTS (Algo.~\ref{algo:1}).}
In every iteration $t$ of STO-BNTS, we firstly construct an NN $f(\mathbf{x};\theta)$ and multiply its output by $\beta_t = 2\log(\pi^2t^2|\mathcal{X}|/(3\delta))$, in which 
% $|\mathcal{X}|$ is the cardinality of the domain and 
$\delta\in(0,1)$ (Theorem~\ref{theorem:regret:exact:ntk}).\footnote{Note that $\beta_t$ is introduced only for the theoretical analysis, and hence we set $\beta_t=1$ in our experiments.}
% we choose a batch of $B\geq1$ input queries $\{\mathbf{x}^i_t\}_{i=1,\ldots,B}$. 
Next, to choose the $i^{\text{th}}$ query $\mathbf{x}^i_t$, we start by sampling initial parameters $\theta_0\sim\text{init}(\cdot)$ and $\theta_0'\sim\text{init}(\cdot)$ independently, and then set the parameters of $\theta_0'$ in the last layer to $0$ (lines 4-5). Next, we use the resulting $\theta_0$ and $\theta_0'$, as well as the 
% (either infinite-width or finite-width) 
NN $f(\mathbf{x},\theta)$, to construct a function $f^i_t(\mathbf{x},\theta)$ (line 6). 
Subsequently, in line 7, using the current history of observations (denoted as $\mathcal{D}_{t-1}$) as the training set, we train $f^i_t(\mathbf{x},\theta)$ (setting $\theta_0$ as the initial parameters) using gradient descent with the following loss function:
\begin{equation}
\mathcal{L}_t(\theta,\mathcal{D}_{t-1}) = \sum\nolimits^{t-1}_{\tau=1}\sum\nolimits^B_{j=1}(y^j_{\tau} - f^i_t(\mathbf{x}^j_{\tau};\theta))^2 + \beta_t^2\sigma^2\norm{\theta - \theta_0}^2_2,
\label{eq:loss:function}
\end{equation}
in which $\sigma^2$ is the observation noise variance (Sec.~\ref{sec:background}). 
After the training, we use the resulting function $f^i_t(\mathbf{x};\theta^i_t)$ as the \emph{acquisition function} to choose the $i^{\text{th}}$ query: $\mathbf{x}^i_t = {\arg\max}_{\mathbf{x}\in\mathcal{X}} f^i_t(\mathbf{x};\theta^i_t)$ (line 8). This procedure (lines 4-8) is repeated independently for $B\geq1$ times, after which a batch of $B$ queries $\{\mathbf{x}^i_t\}_{i=1,\ldots,B}$ are selected and then queried to produce the observations $\{y^i_{t}\}_{i=1,\ldots,B}$. 
Next, the newly collected input-output pairs 
$\{(\mathbf{x}^i_t, y^i_{t})\}_{i=1,\ldots,B}$
% $\{\mathbf{x}^i_t\}_{i=1,\ldots,B}$ and $\{y^i_{t}\}_{i=1,\ldots,B}$ 
are added to $\mathcal{D}_{t-1}$ and the algorithm proceeds to the next iteration.
Importantly, if the NN $f(\mathbf{x};\theta)$ is infinite-width, the function $f^i_t(\mathbf{x};\theta^i_t)$ obtained after the 
% gradient descent 
training in line 7 is \emph{a sample from the GP posterior with the NTK as the kernel}: 
$\mathcal{GP}(\mu_{t-1}(\cdot),\beta_t^2\sigma^2_{t-1}(\cdot,\cdot))$~\cite{he2020bayesian}
% $\mathcal{GP}(\mu_{\text{fb}[t]}(\cdot),\beta_t^2\sigma^2_{\text{fb}[t]}(\cdot,\cdot))$
(Appendix~\ref{app:sec:justification:gp:sampling}). 
This ensures the validity of the TS policy and is crucial for deriving the theoretical guarantee of STO-BNTS (Sec.~\ref{sec:theoretical:results}).

\setlength{\textfloatsep}{0.5cm}
\setlength{\floatsep}{0.10cm}
\begin{algorithm}[t]
% \hspace*{1mm} \textbf{Input:} NN $f_0(\mathbf{x};\theta)$.
\begin{algorithmic}[1]
	\FOR{$t=1,2,\ldots, T/B$}
            \STATE Construct NN $f(\mathbf{x};\theta)$ and multiply its output by $\beta_t$ (Theorem~\ref{theorem:regret:exact:ntk})
	   % \STATE Construct NN: $f(\mathbf{x};\theta)=\beta_t f_0(\mathbf{x};\theta)$
	    \FOR{$i=1,2,\ldots, B$}
            \STATE Sample $\theta_0 \sim \text{init}(\cdot)$ 
            \STATE Sample $\theta_0' \sim \text{init}(\cdot)$ and set the parameters of $\theta_0'$ in the last layer to $0$
            \STATE Set $f^i_t(\mathbf{x};\theta) = f(\mathbf{x};\theta) + \langle\nabla_{\theta}f(\mathbf{x};\theta_0), \theta_0'\rangle$
            \STATE Use observation history $\mathcal{D}_{t-1}$ to train $f^i_t(\mathbf{x};\theta)$
            % \footnotemark 
            with the loss function $\mathcal{L}_t(\theta,\mathcal{D}_{t-1})$~\eqref{eq:loss:function} (setting $\theta_0$ as the initial parameters) using gradient descent till convergence, to obtain $\theta^i_t={\arg\min}_{\theta}\mathcal{L}_t(\theta,\mathcal{D}_{t-1})$
            \STATE Choose $\mathbf{x}^i_t = {\arg\max}_{\mathbf{x}\in\mathcal{X}} f^i_t(\mathbf{x};\theta^i_t)$
        \ENDFOR
        \STATE Query the batch $\{\mathbf{x}^i_t\}_{i=1,\ldots,B}$ to yield the observations $\{y^i_t\}_{i=1,\ldots,B}$, and add them to $\mathcal{D}_{t-1}$
    \ENDFOR
\end{algorithmic}
\caption{STO-BNTS}
\label{algo:1}
\end{algorithm}
% \vspace{-3mm}

% \vspace{-2mm}
\begin{algorithm}[t]
\begin{algorithmic}[1]
	\FOR{$t=1,2,\ldots, T/B$}
            \STATE Construct NN $f(\mathbf{x};\theta)$ and multiply its output by $\beta_t$ (Theorem~\ref{theorem:regret:exact:ntk})
	    \FOR{$i=1,2,\ldots, B$}
            \STATE Sample $\theta_0' \sim \text{init}(\cdot)$, and define $f^{i}_t(\mathbf{x};\theta) = \langle \nabla_{\theta}f(\mathbf{x};\theta_0'), \theta \rangle$ \STATE Sample $\theta_0 \sim \text{init}(\cdot)$ 
            \STATE Use observation history $\mathcal{D}_{t-1}$ to train $f^{i}_t(\mathbf{x};\theta)$ with the loss function $\mathcal{L}_t(\theta,\mathcal{D}_{t-1})$~\eqref{eq:loss:function} (setting $\theta_0$ as the initial parameters)
            using gradient descent till convergence, to obtain $\theta^i_t={\arg\min}_{\theta}\mathcal{L}_t(\theta,\mathcal{D}_{t-1})$
            \STATE Choose the query $\mathbf{x}^i_t = {\arg\max}_{\mathbf{x}\in\mathcal{X}} f^{i}_t(\mathbf{x};\theta^i_t)$
        \ENDFOR
        \STATE Query the batch $\{\mathbf{x}^i_t\}_{i=1,\ldots,B}$ to yield the observations $\{y^i_t\}_{i=1,\ldots,B}$, and add them to $\mathcal{D}_{t-1}$
    \ENDFOR
\end{algorithmic}
\caption{STO-BNTS-Linear}
\label{algo:2}
\end{algorithm}
% \vspace{-2mm}

% \vspace{-0.35in}

%\paragraph{STO-BNTS-Linear (Algo.~\ref{algo:2}).}
\textbf{STO-BNTS-Linear (Algo.~\ref{algo:2}).}
Similar to STO-BNTS, at the beginning of iteration $t$, STO-BNTS-Linear firstly constructs an NN $f(\mathbf{x};\theta)$ and multiply its output by $\beta_t$.
To choose the $i^{\text{th}}$ query in iteration $t$, STO-BNTS-Linear firstly constructs a linear model $f^i_t(\mathbf{x};\theta)$ using the neural tangent features (i.e., $\nabla_{\theta}f(\mathbf{x};\theta_0')$ where $\theta_0' \sim \text{init}(\cdot)$ are initial parameters) as the input features (line 4). Next, we sample $\theta_0 \sim \text{init}(\cdot)$ (line 5) and use it as the initial parameters to train $f^i_t(\mathbf{x};\theta)$ using gradient descent with the same loss function~\eqref{eq:loss:function} as STO-BNTS, to produce $\theta^i_t$ (line 6). 
After that, the $i^{\text{th}}$ query is selected by maximizing the acquisition function $f^i_t(\mathbf{x};\theta^i_t)$: $\mathbf{x}^i_t = {\arg\max}_{\mathbf{x}\in\mathcal{X}} f^i_t(\mathbf{x};\theta^i_t)$ (line $7$).
%sy: not necessary to get two init parameters

% The procedure adopted by STO-BNTS-Linear to select a query 
Lines 4-6 of STO-BNTS-Linear
can be interpreted as a sample-then-optimize method~\cite{matthews2017sample} using the neural tangent features as the input features. As a result, same as STO-BNTS, if an infinite-width NN is used, the function $f^i_t(\mathbf{x};\theta^i_t)$ obtained after the 
% gradient descent 
training in line 6 is a sample from the GP posterior with the NTK as the kernel: $\mathcal{GP}(\mu_{t-1}(\cdot),\beta_t^2\sigma^2_{t-1}(\cdot,\cdot))$ 
% $\mathcal{GP}(\mu_{\text{fb}[t]}(\cdot),\beta_t^2\sigma^2_{\text{fb}[t]}(\cdot,\cdot))$ 
(Appendix~\ref{app:sec:justification:gp:sampling}). However, in contrast to STO-BNTS, if the NN is \emph{finite-width}, the function $f^i_t(\mathbf{x};\theta^i_t)$ derived after line 6 of Algo.~\ref{algo:2} still corresponds to a sample from the GP posterior with the \emph{empirical NTK} $\widetilde{\Theta}(\mathbf{x},\mathbf{x}')=\langle \nabla_{\theta}f(\mathbf{x};\theta_0'), \nabla_{\theta}f(\mathbf{x}';\theta_0') \rangle$ as the kernel. 
As a result, for infinite-width NNs, STO-BNTS-Linear and STO-BNTS enjoy the same sub-linear (under some conditions) upper bound on their cumulative regret (Sec.~\ref{subsec:theory:infinite}); however, for finite-width NNs, unlike STO-BNTS, STO-BNTS-Linear still enjoys a regret upper bound which is sub-linear as long as 
% the width of 
the NN is wide enough (Sec.~\ref{subsec:theory:finite}).

Although STO-BNTS-Linear (Algo.~\ref{algo:2}) has the theoretical advantage of a theoretically guaranteed convergence also for finite-width NNs (Sec.~\ref{subsec:theory:finite}), we expect STO-BNTS (Algo.~\ref{algo:1}) to perform better in practice.
This is because STO-BNTS explicitly trains an NN surrogate model in every iteration 
% using gradient descent 
and is hence able to \emph{directly exploit the strong representation power of NNs} to model the objective function $f$. 
In contrast, STO-BNTS-Linear derives its representation power entirely from the neural tangent features of NTK.
% to model non-linear functions. 
However, it has been shown~\cite{allen2019can,zhou2020neural} that \emph{neural tangent features of NTK can not completely realize the representation power of NNs}. 
As a result, STO-BNTS is expected to be more competitive in practice, especially in problems where the strong representation power of NNs is essential for accurately modeling the complex objective functions.
We empirically verify this practical advantage of STO-BNTS in our experiments (Sec.~\ref{subsec:exp:discussion}).

\vspace{-2mm}
\section{Theoretical Results}
\label{sec:theoretical:results}
\vspace{-2mm}
We firstly prove a regret upper bound for both STO-BNTS and STO-BNTS-Linear 
%when the NN surrogate model is infinite-width, 
using infinite-width NNs,
and show that both algorithms are asymptotically no-regret under certain conditions (Sec.~\ref{subsec:theory:infinite}).
Next, we derive a regret upper bound for STO-BNTS-Linear when the NN is finite-width, and show that 
% it remains 
the regret upper bound remains
% asymptotically no-regret 
sub-linear
as long as the NN is wide enough (Sec.~\ref{subsec:theory:finite}).

% When the neural network surrogate is infinite-width, we need to analyze the regret of a batch TS algorithm with the NTK covariance function, for which we can make use of proof techniques from batch Bayesian optimization~\cite{desautels2014parallelizing} (Sec.~\ref{subsec:theory:infinite}).
% When the neural network surrogate has a finite width ($m<\infty$), since every member of the ensemble corresponds to a sampled function from the GP posterior using the \emph{empirical NTK}, we can derive the regret upper bound by accounting for the approximation error between the exact and empirical NTKs (Sec.~\ref{subsec:theory:finite}).

% To cater to our theoretical analysis, we need to multiply the output of the NN by $\beta_t = 2\log(\pi^2t^2|\mathcal{X}|/\delta)$ when defining the NN surrogate model~\eqref{eq:nn:definition} (i.e., before initialization and training) and multiply the noise variance $\sigma^2$ by $\beta_t^2$ when running our algorithms (i.e., replace $\sigma^2$ by $\beta_t^2\sigma^2$ in~\eqref{eq:loss:function}).
% This ensures that every function we draw is sampled from the GP posterior $\mathcal{GP}(\mu_{\text{fb}[t]}(\cdot),\beta_t^2\sigma^2_{\text{fb}[t]}(\cdot,\cdot))$ with the NTK covariance functions (Appendix ?).

\vspace{-1.5mm}
\subsection{Infinite-width NNs}
\label{subsec:theory:infinite}
\vspace{-1.5mm}
% Here we use $\mathbf{y}_{1:\text{fb}[t]}$ to denote the output observations from iterations $1$ to $\text{fb}[t]$ 
Here we use $\mathbf{y}_{1:t}$ to denote the output observations from iterations $1$ to $t$ (i.e., $t\times B$ observations in total)
and use $\mathbf{y}_{A}$ to denote the vector of observations at a set of inputs $A\subset\mathcal{X}$.
% Assume that $f$ is sampled from a GP with an NTK covariance function $\Theta$, then t
% The cumulative regret of STO-BNTS and STO-BNTS-Linear can be upper-bounded in the following Theorem (proof in Appendix~\ref{app:sec:infinite:width}).
Theorem~\ref{theorem:regret:exact:ntk} (proof in Appendix~\ref{app:sec:infinite:width}) gives a regret upper bound for both STO-BNTS and STO-BNTS-Linear:

\begin{theorem}[Infinite-width NNs]
\label{theorem:regret:exact:ntk}
Assume that $f$ is sampled from a GP with the NTK $\Theta$ as the kernel function, that $|f(\mathbf{x})| \leq B' \, \forall \mathbf{x}\in\mathcal{X}$ for some $B'>0$, and that $\Theta(\mathbf{x},\mathbf{x}') \leq K_0 \, \forall \mathbf{x},\mathbf{x}' \in \mathcal{X}$ for some $K_0>0$.
Let $\delta\in(0,1)$ and $\beta_t = 2\log(\pi^2t^2|\mathcal{X}|/(3\delta))$. Then with probability of $\geq 1-\delta$,
% \footnote{$\widetilde{\mathcal{O}}$ ignores all logarithmic factors for simplicity.}
\vspace{-0.5mm}
\[
R_T = \widetilde{\mathcal{O}}(e^C \sqrt{T} (\sqrt{\gamma_T}+1))
\]
where $\widetilde{\mathcal{O}}$ ignores all logarithmic factors,
% for simplicity, 
 $\gamma_T$ is the max information gain about $f$ from any set of $T$ observations,
and $C$ is an absolute constant 
s.t.~$\max_{A\subset\mathcal{X},|A|\leq B-1}\mathbb{I}(f;\mathbf{y}_{A} | \mathbf{y}_{1:t}) \leq C,\forall t\geq 1$.
% s.t.~$\max_{A\subset\mathcal{X},|A|\leq B-1}\mathbb{I}(f;\mathbf{y}_{A} | \mathbf{y}_{1:\text{fb}[t]}) \leq C,\forall t\geq 1$.
\end{theorem}
%In our theoretical analysis, our primary assumption is that the objective function $f$ is sampled from a GP with the NTK $\Theta$ as the kernel function, which is a common assumption in the analysis of BO algorithms~\cite{srinivas2009gaussian,desautels2014parallelizing}.
%In line with~\cite{kassraie2021neural}, we assume that $\Theta(\mathbf{x},\mathbf{x}') \leq K_0,\forall \mathbf{x},\mathbf{x}' \in \mathcal{X}$ for some $K_0>0$.
%In addition, we also assume that $|f(\mathbf{x})| \leq B'$ for some $B'>0$, which is also a mild assumption since most practical objective functions have bounded values.
The assumption of $\Theta(\mathbf{x},\mathbf{x}') \leq K_0$ is natural and in line with~\cite{kassraie2021neural}, and $|f(\mathbf{x})| \leq B'$ is also a mild assumption since most practical objective functions have bounded values.
Our primary assumption that $f$ is sampled from a GP with the NTK as the kernel function is a common assumption in the analysis of BO algorithms~\cite{desautels2014parallelizing,srinivas2009gaussian}.
Of note, compared with the assumptions from previous works using other commonly used kernels such as the squared exponential (SE) kernel~\cite{desautels2014parallelizing,srinivas2009gaussian},
%and Mat\'ern kernel with a roughness parameter $\nu\geq1$~\cite{chowdhury2017kernelized,srinivas2009gaussian}, 
our assumption on $f$ holds for more complicated objective functions.
Specifically, the class of functions sampled from a GP with the NTK as the kernel subsumes more non-smooth and sophisticated objective functions compared with the other commonly used kernels.
As shown in~\cite{desautels2014parallelizing}, $C$ can be chosen to be an absolute constant that is independent of $B$ and $T$ as long as we initialize our algorithm using the uncertainty sampling method, which entails choosing the initial inputs by sequentially maximizing the GP posterior variance (more details in Appendix~\ref{app:gp:posterior:variance:for:ntk}).
% \footnote{For NTK, the GP posterior variance at any input $\mathbf{x}$ can be easily approximated by $\nabla_{\theta}f(\mathbf{x};\theta_0)^{\top}\big[ \sum^{t-1}_{\tau=0} \nabla_{\theta}f(\mathbf{x}_{\tau};\theta_0) \nabla_{\theta}f(\mathbf{x}_{\tau};\theta_0)^{\top} + \sigma^2 I \big]^{-1} \nabla_{\theta}f(\mathbf{x};\theta_0)$, where $\theta_0\sim\text{init}(\cdot)$ are initial parameters.}
Specifically, for any chosen constant $C>0$, as long as we run the uncertainty sampling initialization stage for $T_{\text{init}}$ iterations such that $((B-1)\gamma_{T_{\text{init}}}) / T_{\text{init}} \leq C$, 
% then the condition $\max_{A\subset\mathcal{X},|A|\leq B-1}\mathbb{I}(f;\mathbf{y}_{A} | \mathbf{y}_{1:\text{fb}[t]}) \leq C,\forall t\geq 1$ 
then 
% the condition
$\max_{A\subset\mathcal{X},|A|\leq B-1}\mathbb{I}(f;\mathbf{y}_{A} | \mathbf{y}_{1:t}) \leq C,\forall t\geq 1$ 
is guaranteed to be satisfied. Since it has been shown by the work of~\cite{kassraie2021neural} that $\gamma_T = \widetilde{\mathcal{O}}(T^{(d-1)/d})$ grows sub-linearly for the NTK,\footnote{
Note that in order to quote the results from~\cite{kassraie2021neural}, we need follow their assumption to assume that the domain $\mathcal{X}$ is a subset of the $d$-dimensional unit hyper-sphere: $\mathcal{X}\subset\{\mathbf{x} | \norm{\mathbf{x}}_2=1\}$, which is more strict than our main assumption (Sec.~\ref{sec:background}) that $\mathcal{X}$ is a subset of the unit hyper-ball: $\mathcal{X}\subset\{\mathbf{x} | \norm{\mathbf{x}}_2\leq1\}$.
% Note that in order to quote the results from~\cite{kassraie2021neural}, we need the domain $\mathcal{X}$ to be a subset of the $d$-dimensional unit hyper-sphere: $\mathcal{X}\subset\{\mathbf{x} | \norm{\mathbf{x}}_2=1\}$ instead of the unit hyper-ball: $\mathcal{X}\subset\{\mathbf{x} | \norm{\mathbf{x}}_2\leq1\}$ (Sec.~\ref{sec:background}).
} therefore, $((B-1)\gamma_{T_{\text{init}}})/T_{\text{init}}$ is decreasing as $T_{\text{init}}$ increases.
As a result, for any chosen $C$, we are able to choose a finite $T_{\text{init}}$ such that the condition $((B-1)\gamma_{T_{\text{init}}}) / T_{\text{init}} \leq C$ is satisfied.

% \paragraph{The Constant $C$}
% The term $C$ is an absolute constant such that $I(f;\mathbf{y}_{\text{fb}[t]+1:t-1} | \mathbf{y}_{1:\text{fb}[t]}) \leq C,\forall t\geq 1$, and as has been shown by the work of~\cite{desautels2014parallelizing}, we can make $C$ a constant that is independent of $B$ and $T$ by dividing our entire algorithm into two stages: an initialization stage of length $T_{\text{init}}$ during which we run the uncertainty sampling algorithm, and a main stage. Then, according to the analysis of~\cite{desautels2014parallelizing}, for any chosen constant $C$, we can run the uncertainty sampling initialization stage for $T_{\text{init}}$ iterations such that $((B-1)\gamma_{T_{\text{init}}}) / T_{\text{init}} \leq C$. Since it has been shown by the work of~\cite{kassraie2021neural} that $\gamma_T$ grows sub-linearly for the NTK covariance function, so, $((B-1)\gamma_{T_{\text{init}}}/T_{\text{init}}$ is decreasing in $T_{\text{init}}$. Therefore, for any chosen $C$, we can choose a finite $T_{\text{init}}$ such that the condition $((B-1)\gamma_{T_{\text{init}}}) / T_{\text{init}} \leq C$ is satisfied.

% For example, if we choose $C=d$, then the required number of initial iterations is approximately $T_{\text{init}}=\Theta\big((\frac{B-1}{d})^d\big)$.
For example, if we choose $C=1$, then the required number of initial iterations is approximately $T_{\text{init}}=\Theta((B-1)^d)$. 
As a result, the regret upper bound will be a summation of $2B'T_{\text{init}}$ (i.e., the regrets incurred during the initializatioin stage, because the regret at every step is upper-bounded by $2B'$) and the regret upper bound from Theorem~\ref{theorem:regret:exact:ntk} with $C=1$.
% As a consequence, by separately considering the regret incurred during and after the initialization stage, the resulting regret upper bound will be a summation of $2B'T_{\text{init}}$ (because the regret at every initialization step is upper-bounded by $2B'$) and the regret upper bound from Theorem~\ref{theorem:regret:exact:ntk} with $C=1$.
%sy: kind of confusing for the sentence after and
Since both $B'$ and $T_{\text{init}}$ are constants independent of $T$ (assuming that $B$ is independent of $T$), 
% which are independent of $T$, 
the asymptotic regret upper bound can be simplified into $R_T = \widetilde{\mathcal{O}}(\sqrt{T}(1+\sqrt{\gamma_T}))$. 
Plugging in 
% the expression of 
$\gamma_T = \widetilde{\mathcal{O}}(T^{(d-1)/d})$, the final regret upper bound becomes $R_T =\widetilde{\mathcal{O}}(T^{1/2}+T^{(2d-1)/(2d)})= \widetilde{\mathcal{O}}(T^{(2d-1)/(2d)})$ which is sub-linear in $T$ and hence implies that our STO-BNTS and STO-BNTS-Linear are both \emph{asymptotically no-regret} when the NN surrogate is infinite-width.

Moreover, when $T\gg B$ and $B$ is a constant which is independent of $T$, Theorem~\ref{theorem:regret:exact:ntk} gives us insights on the benefit of batch over sequential evaluations.
% In particular, our analysis above suggests that 
In this case, the regrets incurred during initialization (i.e., $\widetilde{\mathcal{O}}((B-1)^d)$) is negligible. 
Therefore, our analysis above suggests that our algorithms with batch evaluations ($B>1$) enjoys the same asymptotic regret upper bound as its sequential counterpart ($B=1$) since the resulting regret upper bound of $R_T = \widetilde{\mathcal{O}}(\sqrt{T}(1+\sqrt{\gamma_T}))$ does not depend on the batch size $B$.
This demonstrates the advantage of batch evaluations because when $B>1$, some of our evaluations can be run in parallel, which is not supported by the sequential setting with $B=1$.
% because compared with the sequential setting, some of our iterations can be run in parallel. 
% As a simple illustration, for the synchronous batch TS setting, our batch TS algorithm can run $B>1$ function evaluations in parallel in every iteration. Therefore, for a total of $TB$ iterations, 
As a simple illustration, for a large $T$, both the sequential and batch settings achieve a simple regret of the order $\widetilde{\mathcal{O}}(T^{-1/(2d)})$ after $T$ function evaluations. 
However, since our batch setting can evaluate every $B>1$ selected inputs in parallel in every iteration, our batch setting with $B>1$ achieves this simple regret after only $T/B$ iterations, which is only a fraction ($1/B$) of the $T$ iterations required by the sequential setting.
This also shows that we enjoy more benefit with a larger batch size $B$.
% and the speedup due to batch evaluations improves with a larger $B$.

% \paragraph{RKHS Assumption}
% Under the other assumption where we assume that $f$ lies in the RKHS associated with an NTK covariance function, the resulting regret bound will be $R_T = \widetilde{\mathcal{O}}(\sqrt{T\gamma_T}(1+\sqrt{\gamma_T}))=\widetilde{\mathcal{O}}(T^{(3d-2)/(2d)})$. This suggests that the algorithm is not sub-linear for $d\geq 2$. To resolve this issue, we can try to explore the same approach adopted by~\cite{kassraie2021neural}, i.e., by introducing a "Sup" variant of our algorithm.

\subsection{Finite-width NNs}
\label{subsec:theory:finite}
Here we prove 
%an upper bound on the cumulative regret of our 
a regret upper bound for 
STO-BNTS-Linear 
%when the NN surrogate model is \emph{finite-width}. 
with \emph{finite-width} NN.
The main technical challenge in the proof 
lies in the disparity
% is to carefully manage the mismatch
between the exact and empirical NTKs, i.e., the function $f$ is assumed to be sampled from a GP with the exact NTK $\Theta$ yet our acquisition function $f^i_t(\mathbf{x};\theta^i_t)$ (line 7 of Algo.~\ref{algo:2}) is obtained using the empirical NTK $\widetilde{\Theta}$ when the NN is finite-width.
To this end, we make use of the following theoretical guarantee on the approximation error between $\Theta$ and $\widetilde{\Theta}$.
% In our BNTS-NN algorithm where finite-width NNs are used as the surrogate models, every sampled function in our algorithm is sampled from the GP with the \emph{empirical NTK}, instead of the exact NTK as in the BNTS-NTK algorithm.
% Therefore, in the analysis of BNTS-NN, we need to additionally take into account the approximation error resulting from the use of the empirical NTK.
% The approximation quality of the empirical NTK has been proved:
\begin{proposition}[Theorem 3.1 of~\cite{arora2019exact}]
\label{prop:arora}
Choose $\varepsilon>0$ and $\delta\in(0,1)$. If the width of every layer in an NN satisfies $m = \Omega(\frac{L^6}{\varepsilon^4}\log(4L|\mathcal{X}|^2/\delta))$, 
then $\forall \mathbf{x},\mathbf{x}' \in \mathcal{X}$,
we have with probability $\geq 1-\delta/4$ that
\[
\left| \langle \nabla_{\theta}f(\mathbf{x},\theta_0), \nabla_{\theta}f(\mathbf{x}',\theta_0) \rangle - \Theta(\mathbf{x}, \mathbf{x}') \right| \leq (L+1) \varepsilon.
\]
% \[
% \left| \langle \nabla_{\theta}f(\mathbf{x},\widetilde{\theta}), \nabla_{\theta}f(\mathbf{x}',\widetilde{\theta}) \rangle - \Theta(\mathbf{x}, \mathbf{x}') \right| \leq (L+1) \varepsilon.
% \]
\end{proposition}
% \begin{proposition}[Theorem 3.1 of~\cite{arora2019exact}]
% \label{prop:arora}
% Choose $\epsilon>0$ and $\delta_{\text{ntk}}\in(0,1)$. If the width of every layer satisfied 
% % $m = \Omega(\frac{L^6}{\varepsilon^4}\log(L/\delta_{\text{ntk}}))$, 
% $m = \Omega(\frac{L^6}{\varepsilon^4}\log(L|\mathcal{X}|^2/\delta_{\text{ntk}}))$, 
% then 
% % $\forall \mathbf{x}, \mathbf{x}' \in \mathbb{R}^{d}$ such that $\norm{\mathbf{x}}_2\leq 1$ and $\norm{\mathbf{x}'}_2\leq 1$, 
% $\forall \mathbf{x},\mathbf{x}' \in \mathcal{X}$,
% we have with probability $\geq 1-\delta_{\text{ntk}}$ that
% \[
% \left| \langle \nabla_{\theta}f(\mathbf{x},\widetilde{\theta}), \nabla_{\theta}f(\mathbf{x}',\widetilde{\theta}) \rangle - \Theta(\mathbf{x}, \mathbf{x}') \right| \leq (L+1) \varepsilon.
% \]
% \end{proposition}
Proposition~\ref{prop:arora} ensures that we can reduce the upper bound $(L+1)\varepsilon$ on the approximation error between $\Theta$ and $\widetilde{\Theta}$ by increasing the width $m$ of the NN.
For example, let $m=C_{\text{ntk}}L^6\varepsilon^{-4}\log(4L|\mathcal{X}|^2/\delta)$ for some constant $C_{\text{ntk}}>0$, then the approximation quality of Proposition~\ref{prop:arora} can be expressed as $(L+1)\varepsilon = C_{\text{ntk}}(L+1)L^{3/2}\log^{1/4}(4L|\mathcal{X}|^2/\delta)m^{-1/4}$ which decreases as the width $m$ increases.
% Therefore, we can exploit Proposition~\ref{prop:arora} to bound the approximation error introduced by using the empirical NTK instead of the exact NTK when sampling the function $\widetilde{f}$.
% Let $m=C_{\text{ntk}}\frac{L^6}{\varepsilon^4}\log(4L|\mathcal{X}|^2/\delta)$ for some constant $C_{\text{ntk}}>0$, then the approximation quality (i.e., the right hand side) of proposition~\ref{prop:arora} can be expressed as $(L+1)\varepsilon = C_{\text{ntk}}(L+1)L^{3/2}\log^{1/4}(4L|\mathcal{X}|^2/\delta)m^{-1/4}$.
In our theoretical analysis, we assume that $(L+1)\varepsilon \leq 1$, which can be satisfied as long as $m$ is large enough.
The regret upper bound for STO-BNTS-Linear is given in Theorem~\ref{theorem:regret:approc:ntk} (proof in Appendix~\ref{app:sec:finite:width}):
% As a result, an upper bound on the cumulative regret of BNTS-NN can also be derived.
% An upper bound on the cumulative regret of STO-BNTS-Linear with finite-width NN surrogate models is given by the following Theorem.
\begin{theorem}[Finite-width NNs]
\label{theorem:regret:approc:ntk}
Let $\delta\in(0,1)$ and $\beta_t = 2\log(2\pi^2t^2|\mathcal{X}|/(3\delta))$. 
%Then with probability of $\geq 1-\delta$,
Then we have
% \[
% R_T = \widetilde{\mathcal{O}}\left(e^C \sqrt{T} (\sqrt{\gamma_T}+1) + T^3m^{-1/8}(L+1)^{5/4}\right),
% \]
\[
R_T = \widetilde{\mathcal{O}}(e^C \sqrt{T} (\sqrt{\gamma_T}+1) + T^3m^{-1/8}(L+1)^{5/4}),
\]
with probability of $\geq 1-\delta$. Here $\gamma_T$ and $C$ are the same as those defined in Theorem~\ref{theorem:regret:exact:ntk}.
\end{theorem}
The first term in the regret upper bound in Theorem~\ref{theorem:regret:approc:ntk} is the same as that of Theorem~\ref{theorem:regret:exact:ntk} for infinite-width NNs and can hence be made sub-linear by using uncertainty sampling as the initialization stage: $R_T = \widetilde{\mathcal{O}}(\sqrt{T}(1+\sqrt{\gamma_T}))= \widetilde{\mathcal{O}}(T^{(2d-1)/(2d)})$.
The second term in the regret upper bound represents the additional regrets incurred by the use of finite-width NNs, which can also be made sub-linear 
% as long as the width $m$ of the NN is large enough, i.e., if $m=\Omega(T^{24})$.
by choosing $m=\Omega(T^{24})$.
In other words, 
if the width $m$ of the NN is chosen to be large enough (i.e., if $m=\Omega(T^{24})$), then the cumulative regret of STO-BNTS-Linear scales sub-linearly in $T$.
\subsection{Discussion}
The assumption on the function $f$ for our theoretical analyses in Theorems~\ref{theorem:regret:exact:ntk} and~\ref{theorem:regret:approc:ntk} differs from those made by the previous works on neural contextual bandits~\cite{zhang2020neural,zhou2020neural}.
% Firstly, previous works on neural contextual bandits~\cite{zhou2020neural,zhang2020neural,kassraie2021neural} have all assumed that the NTK matrix of all observed contexts up to iteration $T$ is \emph{positive definite}.
% To this end, they have assumed that all contexts have the same $L_2$ norm (i.e., the space of contexts $\mathcal{X}$ is a subset of the $d$-dimensional \emph{unit hyper-sphere}) and no context is observed more than once.
% In contrast, we do not need the NTK matrix of the selected inputs to be positive definite.
% Therefore, we allow the space of inputs $\mathcal{X}$ to be any finite subset of the $d$-dimensional \emph{unit ball} (Sec.~\ref{sec:background}), which is more realistic for our setting where we intend to optimize an objective function over a fixed domain (e.g., Bayesian optimization).
% Secondly, 
% Regarding the smoothness assumption on the objective function $f$,
Specifically, these previous works 
% on neural contextual bandits 
have relied on the assumption of a positive definite NTK Gram matrix 
% (which is not required by our work) 
to approximate the value of $f$ (only evaluated at the observed contexts up to iteration $T$) using a function that is linear in the neural tangent features. 
When translated into our setting (which is equivalent to the contextual bandit setting where all contexts are fixed for all $t\geq1$), the assumption of a positive definite NTK Gram matrix can be easily violated as long as any input $\mathbf{x}$ is queried more than once, thus rendering this assumption unrealistic in our setting.
% corresponds to assuming that no input $\mathbf{x}$ is queried more than once, which is not realistic in our setting where the contexts are fixed throughout all iterations.
In contrast, we have assumed that $f$ is sampled from a GP with the NTK as the kernel, which is a common assumption in the analysis of BO~\cite{desautels2014parallelizing,srinivas2009gaussian} and allows us to derive a sub-linear regret upper bound.
As a result of the different assumptions, our regret upper bounds are not directly comparable with those from the previous works on neural contextual bandits~\cite{zhang2020neural,zhou2020neural}.
% The other commonly used assumption in the analysis of GP bandits is to assume that $f$ lies in an RKHS induced by a kernel $k$~\cite{chowdhury2017kernelized,kassraie2021neural}. Our theoretical analysis can also be straightforwardly extended to the RKHS setting (Appendix~\ref{app:proof:rkhs:assumption}). In particular, the regret upper bound in Theorem~\ref{theorem:regret:exact:ntk} (and the first term in Theorem~\ref{theorem:regret:approc:ntk}) would become $R_T = \widetilde{\mathcal{O}}(\sqrt{T\gamma_T}(1+\sqrt{\gamma_T}))=\widetilde{\mathcal{O}}(T^{(3d-2)/(2d)})$, which is unfortunately no longer sub-linear for $d\geq 2$.

\section{Experiments}
\label{sec:exp}
We compare our STO-BNTS and STO-BNTS-Linear with the baselines of Neural UCB~\cite{zhou2020neural} and Neural TS~\cite{zhang2020neural}, as well as GP-UCB and GP-TS which use GPs (with the SE kernel) instead of NNs as their surrogate models.
The original implementations of Neural UCB~\cite{zhou2020neural} and Neural TS~\cite{zhang2020neural} are only applicable to discrete domains. So, for a fair comparison in those tasks with continuous domains, we have modified their implementations to maximize their acquisition functions in the same way as our methods (i.e., through a combination of random search and L-BFGS-B, refer to Appendix~\ref{app:sec:experimental:details} for more details).
We firstly explore some interesting insights about our algorithms using a synthetic experiment in Sec.~\ref{sec:exp:synth}. Next, we apply our algorithms to real-world AutoML (Sec.~\ref{subsec:exp:automl}) and RL (Sec.~\ref{subsec:exp:rl}) problems, as well as an optimization task over images (Sec.~\ref{sec:exp:images}).
Finally, we discuss some interesting insights from our experiments in Sec.~\ref{subsec:exp:discussion}.
% Compared with BO, our method is expected to work better in terms of (a) capability of dealing with categorical inputs and (b) scalability. 
% Compared with those BO papers using Bayesian neural networks as the surrogate model, our method should also be more scalable since our NN is not Bayesian and is hence easier to train.
% Compared with NTS, our method is expected to be more scalable.
% So, we can craft experiments with (a) difficult-to-optimize functions such that it will require a large number of observations to show our advantage of scalability and (b) categorical inputs. For example, we can use hyperparameter tuning for XGBoost models (which invovles a number of categorical hyperparameters) and choose a large search space to make the problem difficult. There are also synthetic experiments from those papers on BO for categorical inputs.
% For simplicity, we use random search to choose the initial input queries instead of using uncertainty sampling method discussed in Sec.~\ref{subsec:theory:infinite}.
We plot the simple regret (or the best found observation till an iteration) when presenting the experimental results, which is the common practice in BO \cite{contal2013parallel,kandasamy2018parallelised}.
We have deferred some experimental details to Appendix~\ref{app:sec:experimental:details} due to 
% the 
space limitation.
Our code is available at \url{https://github.com/daizhongxiang/sto-bnts}.

\subsection{Synthetic Experiment}
\label{sec:exp:synth}
% We firstly use synthetic experiments to explore some interesting insights about the performances of our algorithms.
Here, we sample a smooth function from a GP with an SE kernel (with a length scale of $0.1$), defined on a discrete 1-dimensional domain within the range of $[0,1]$.
% For all methods in this section, we use the ERF activation function because the synthetic objective function is very smooth, and 
For all methods, we use an NN architecture with a depth of $L=8$ and a width of $m=64$ unless specified otherwise.

Figs.~\ref{fig:synth:illu} (a-c) illustrate 
% the behaviors of different algorithms by comparing their 
the acquisition functions $f^i_t(\mathbf{x};\theta^i_t)$ (line 8 of Algo.~\ref{algo:1} and line 7 of Algo.~\ref{algo:2}) of different algorithms.
%sy: maybe we need to give a citation for deep ensemble
The \emph{Deep Ensemble} method~\cite{lakshminarayanan2016simple} in Fig.~\ref{fig:synth:illu}a can be regarded as a reduced version of our STO-BNTS algorithm (Algo.~\ref{algo:1}) in which the term $\langle\nabla_{\theta}f(\mathbf{x};\theta_0), \theta_0'\rangle$ (i.e., the second term in line 6 of Algo.~\ref{algo:1}) is removed. As a result, Deep Ensemble does not enjoy the theoretical guarantees of our algorithms (Sec.~\ref{sec:theoretical:results}).
In Figs.~\ref{fig:synth:illu} (a-c), given the same training set (red stars), every method constructs a batch of $B=100$ acquisition functions. 
E.g., STO-BNTS repeats lines 4-7 of Algo.~\ref{algo:1} independently for $B=100$ times to produce acquisition functions $f^i_t(\mathbf{x};\theta^i_t)$ for $i=1,\ldots,100$, which are plotted as the blue lines in Fig.~\ref{fig:synth:illu}b.
Note that every acquisition function (blue line) is maximized to select an input query.
The figures show that compared with the naive baseline of Deep Ensemble, our STO-BNTS and STO-BNTS-Linear are able to display more exploratory behaviors in unexplored regions (e.g., the interval of $[0.2,0.4]$). This may be explained by our theoretical guarantees (Sec.~\ref{sec:theoretical:results}) which imply that both of our algorithms are able to perform exploration in a principled way and hence naturally handle the the exploration-exploitation trade-off. Moreover, it has also been justified by~\cite{he2020bayesian} that the addition of the term $\langle\nabla_{\theta}f(\mathbf{x};\theta_0), \theta_0'\rangle$ improves the ability of the NN to characterize the uncertainty of predictions, which corroborates our findings here.

The simple regrets of different algorithms are plotted in Fig.~\ref{fig:synth:illu}d.\footnote{To show the benefit of batch evaluations, in all experiments (including real-world experiments), we use the \emph{iterations} $t$ as the horizontal axis and in every iteration $t$, we report the largest $f(\mathbf{x}^i_t)$ ($y^i_t$ in real-world experiments) within a batch (when $B>1$) as the observation in this iteration.}
The first interesting observation is that our STO-BNTS (orange) significantly outperforms Deep Ensemble (red), which corroborates the insight discussed above positing that the addition of the term $\langle\nabla_{\theta}f(\mathbf{x};\theta_0), \theta_0'\rangle$ leads to more principled exploration and hence better performances. 
Due to its 
% overly exploitative behavior and the 
lack of exploration as illustrated in Fig.~\ref{fig:synth:illu}a, Deep Ensemble fails to reach zero regret in Fig.~\ref{fig:synth:illu}d.
Furthermore, the discrepancy between the green and purple curves shows that batch evaluations ($B=4$) lead to significant performance improvement.
Moreover, compared with the green curve for which $m=64$, \emph{using a wider NN} (gray curve, $m=512$) \emph{substantially improves the performance of STO-BNTS-Linear} yet employing a shallower NN (yellow curve, $m=16$) significantly degrades the performance. 
These observations agree with Theorem~\ref{theorem:regret:approc:ntk} which states that a larger width $m$ reduces the regret of STO-BNTS-Linear.
Similarly, the NN surrogate model of STO-BNTS should also be wide enough since the use of a narrower NN (light blue curve, $m=16$) also leads to a worse performance for STO-BNTS.
Lastly, GP-TS (blue) performs competitively in this experiment with a very smooth objective function. However, as we will show in the next two sections, in real-world experiments with more complicated objective functions, GP-based methods 
% GP-TS and GP-UCB 
are consistently outperformed by our algorithms.
\begin{figure}[t]
     \centering
     \begin{tabular}{cccc}
         \hspace{-5mm}
         \includegraphics[width=0.263\linewidth]{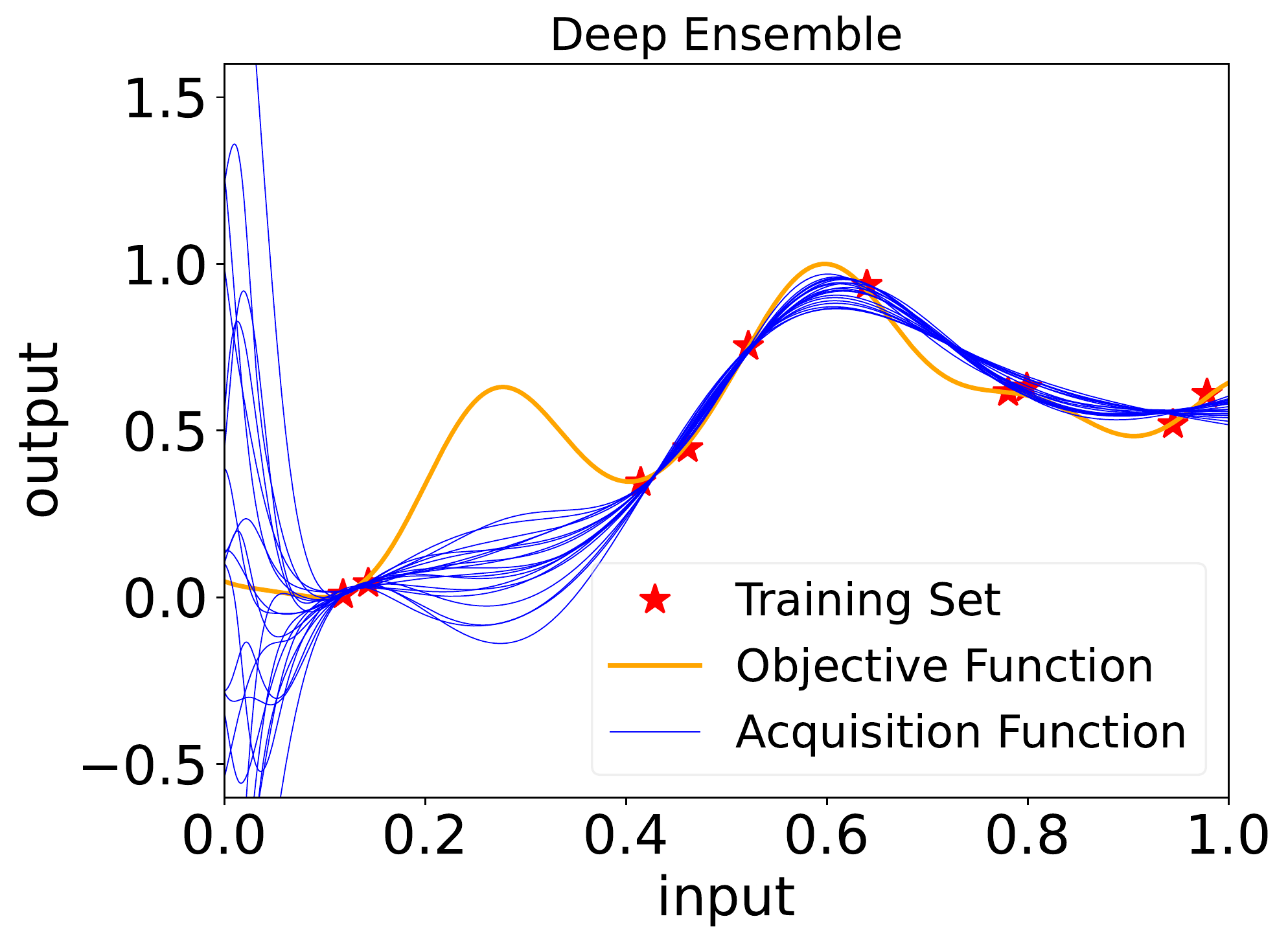} & \hspace{-6.9mm}
         \includegraphics[width=0.263\linewidth]{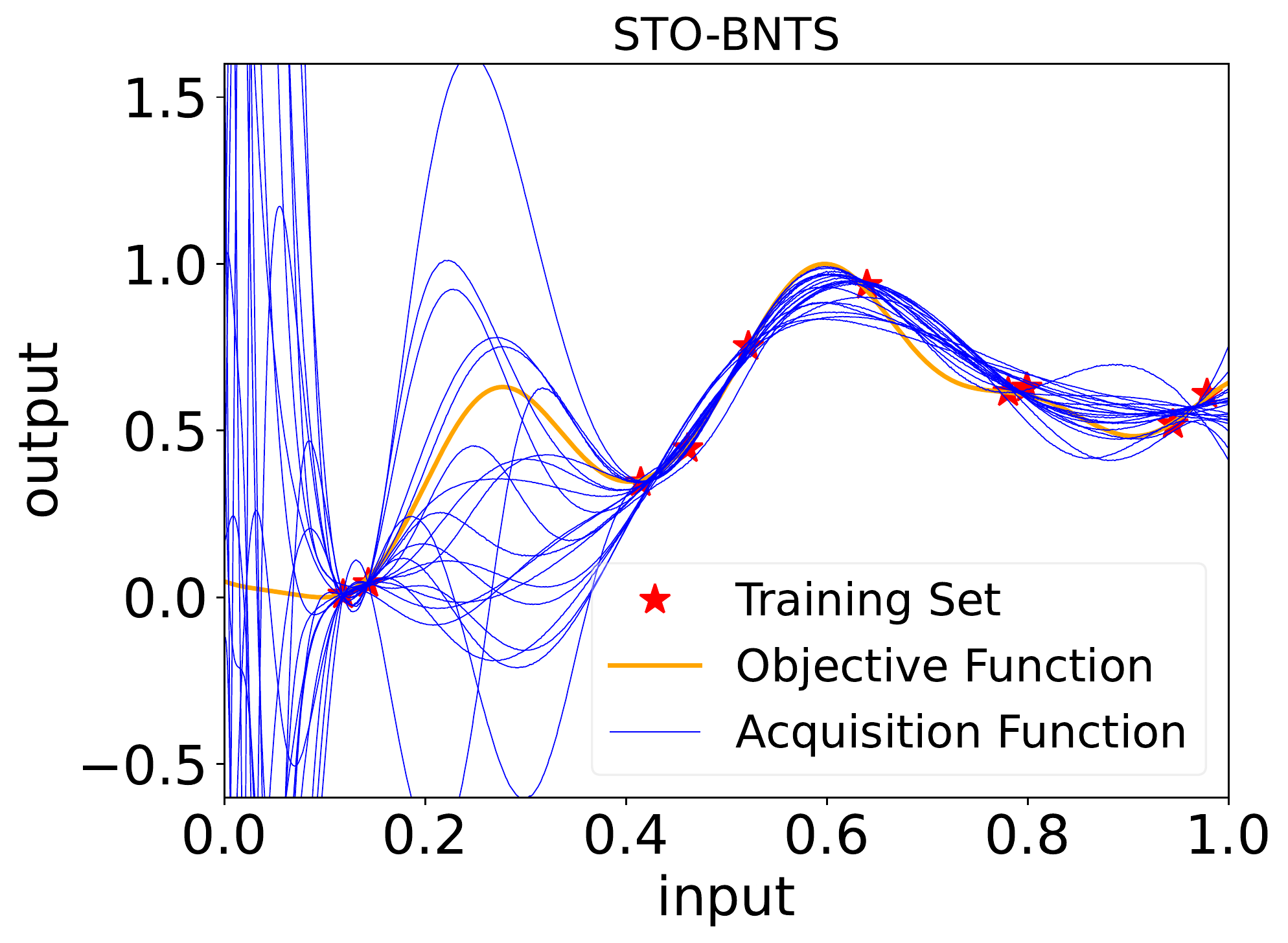} & \hspace{-6.9mm}
        %  \vspace{-1mm}
         \includegraphics[width=0.263\linewidth]{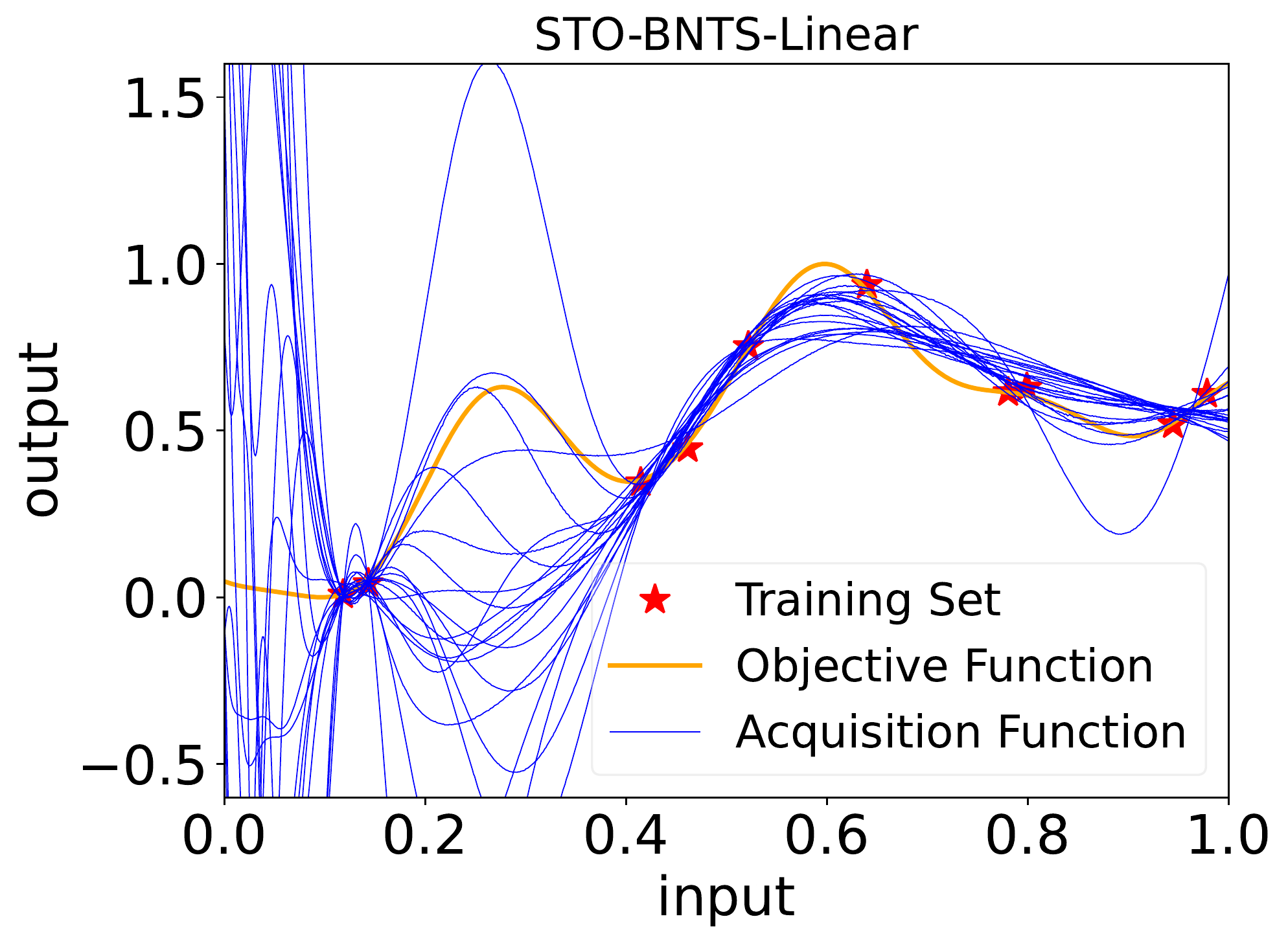} & \hspace{-6.9mm}
         \includegraphics[width=0.26\linewidth]{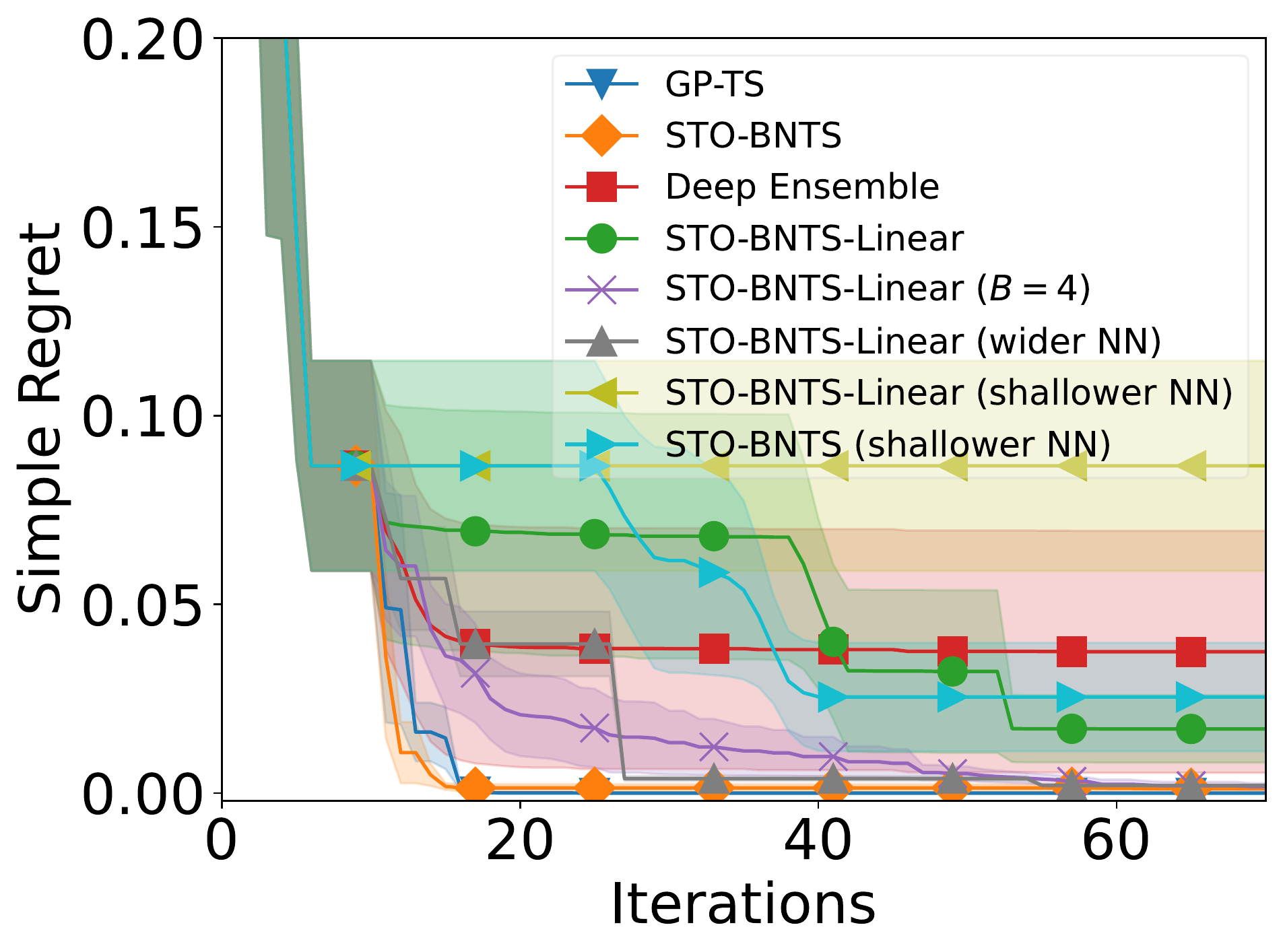}\\
         {(a)} & {(b)} & {(c)} & {(d)}
     \end{tabular}
 \vspace{-2.5mm}
     \caption{
    %  (a-c) Behaviors of different algorithms in terms of their sampled acquisition functions.
     (a-c) Acquisition functions and (d) performances in the synthetic experiment (Sec.~\ref{sec:exp:synth}).
    %  in terms of their sampled acquisition functions.
    %  (d) The simple regret of different methods ($B=1$ unless specified otherwise).
    }
     \label{fig:synth:illu}
 \vspace{-2mm}
\end{figure}
\subsection{Real-world Experiments on Automated Machine Learning (AutoML)}
\label{subsec:exp:automl}
% \vspace{-0.5mm}
% In this section, we perform experiments on three real-world AutoML problems, i.e., hyperparameter tuning tasks for ML models. 
Here, we 
%perform experiments on 
adopt
3 hyperparameter tuning tasks.
% for ML models.
We use
% these hyperparameter tuning 
tasks involving categorical hyperparameters to highlight the advantage of our NN-based over GP-based methods.
% NN-based algorithms in this type of real-world problems.
We use a diabetes diagnosis dataset to tune $6$ categorical hyperparameters of random forest (RF), and then use the MNIST dataset to tune $9$ hyperparameters ($4$ continuous and $5$ categorical) of XGBoost and $9$ hyperparameters ($2$ continuous and $7$ categorical) of convolutional neural networks (CNNs).
% The original implementations of Neural UCB~\cite{zhou2020neural} and Neural TS~\cite{zhang2020neural} are only applicable to discrete domains, so for a fair comparison in problems with continuous domains, we have modified their implementations to maximize their acquisition functions in the same way as the other methods under comparison (i.e., through a combination of random search and L-BFGS-B).
Fig.~\ref{fig:exp:automl} plots the results for the RF (a,b) and XGBoost (c,d) tasks,
and the results for CNN are shown in Fig.~\ref{fig:exp:automl:cnn} (Appendix~\ref{app:subsec:exp:automl}).

For each task, we firstly compare the performances of different variants of our algorithms in Figs.~\ref{fig:exp:automl}a, c and Fig.~\ref{fig:exp:automl:cnn}a,
% , and then compare with other baseline methods (Fig.~\ref{fig:exp:automl}b, d, f). 
% The results in Fig.~\ref{fig:exp:automl} 
which demonstrate a number of interesting insights.
% about our algorithms. 
Firstly, our STO-BNTS outperforms Deep Ensemble, which is consistent with the results in 
% the synthetic experiment (
Fig.~\ref{fig:synth:illu}d and further emphasizes the practical significance of the improved exploration performed by our STO-BNTS (Sec.~\ref{sec:exp:synth}).
Secondly, STO-BNTS-Linear is significantly outperformed by STO-BNTS in Figs.~\ref{fig:exp:automl} a and c, which can likely be attributed to its inability to fully leverage the representation power of NNs as we have discussed in Sec.~\ref{sec:algo} (see more discussions in Sec.~\ref{subsec:exp:discussion}).
Next, Figs.~\ref{fig:exp:automl}a, c and Fig.~\ref{fig:exp:automl:cnn}a also show that the performance of STO-BNTS tends to suffer when the NN surrogate model is either overly shallow ($L=1,m=512$, orange curves) or overly deep ($L=8,m=64$, red curves), 
that is, both the orange and red curves underperform significantly in two of the three experiments. 
This is likely because an overly shallow NN lacks the representation power to model complicated objective functions, whereas an overly deep NN may be prone to overfitting
since the size of the training set here is much smaller than typical deep learning applications. 
In contrast, the NN architecture of $L=2,m=256$ (blue curves) consistently performs well in all three experiments, therefore, we will use it as the default architecture in the experiments in the next section.
Moreover, the benefit of batch evaluations can also be further confirmed by comparing the blue ($B=1$) and green ($B=4$) curves in Figs.~\ref{fig:exp:automl}a, c and Fig.~\ref{fig:exp:automl:cnn}a.

Figs.~\ref{fig:exp:automl}b, d and Fig.~\ref{fig:exp:automl:cnn}b show that for the same NN architecture of $L=2,m=256$, our STO-BNTS is the best-performing method
% outperforms all baselines under comparison 
since it performs better than all baselines in Figs.~\ref{fig:exp:automl}b and d and comparably with them in Fig.~\ref{fig:exp:automl:cnn}b.
The undesirable performances of Neural UCB and Neural TS may be explained by the errors due to their diagonal matrix approximation (Sec.~\ref{sec:introduction}), and the underwhelming performances of GP-UCB and GP-TS may result from the ineffectiveness of GP in modeling categorical inputs (more on this in Sec.~\ref{subsec:exp:discussion}).
\begin{figure}[t]
     \centering
     \begin{tabular}{cccc}
         \hspace{-4mm}
         \includegraphics[width=0.25\linewidth]{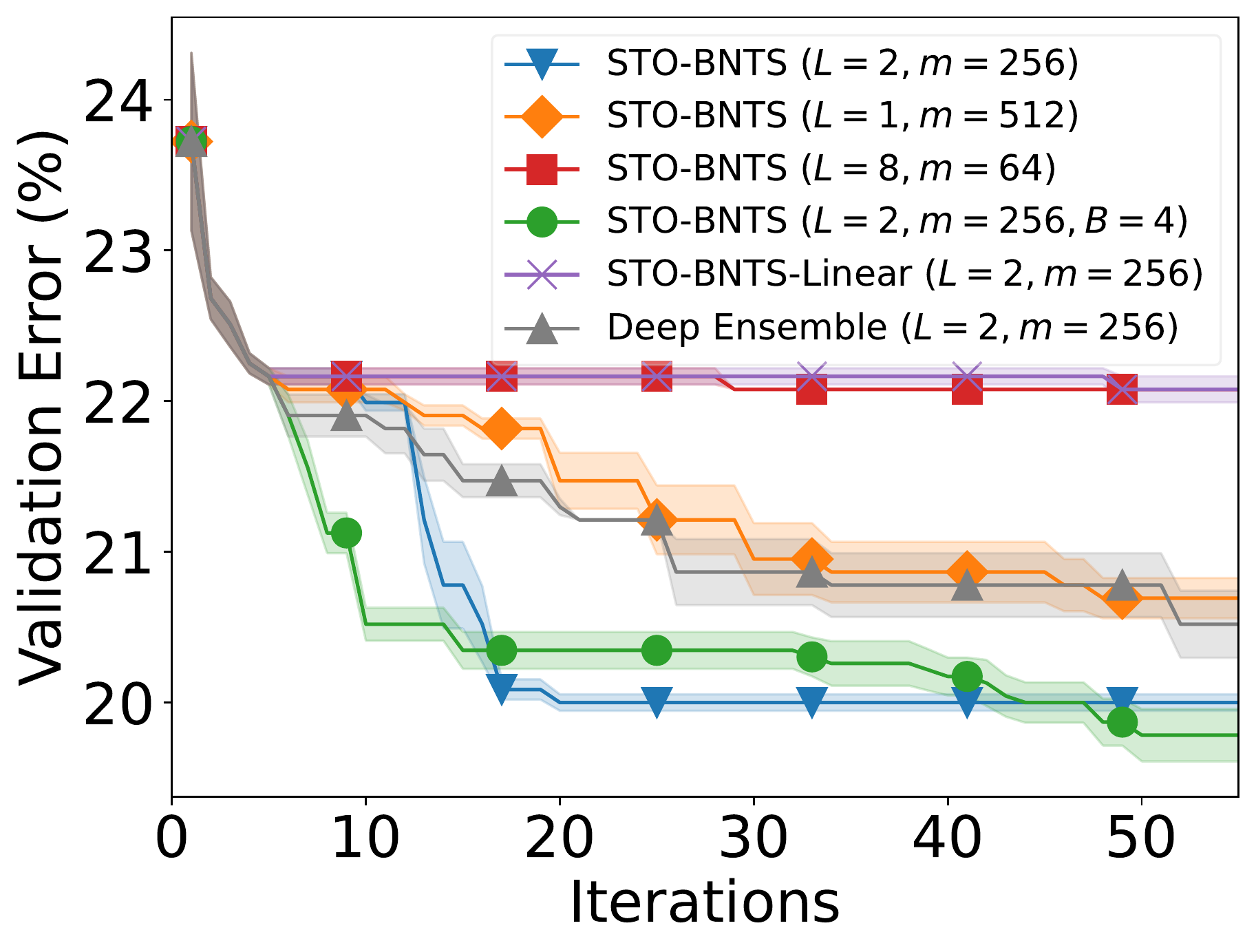} & \hspace{-5mm}
         \includegraphics[width=0.25\linewidth]{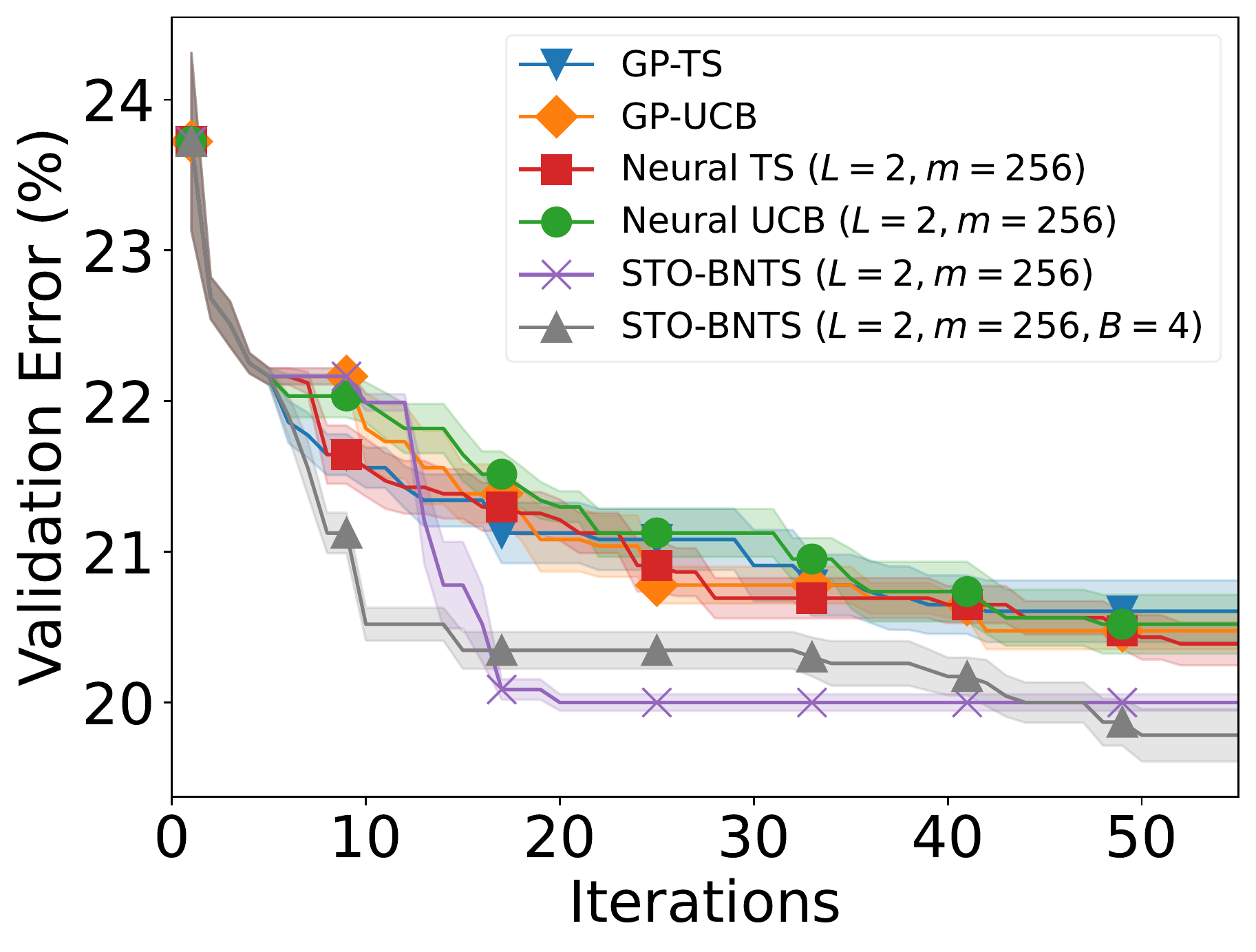} & \hspace{-5mm}
         \includegraphics[width=0.25\linewidth]{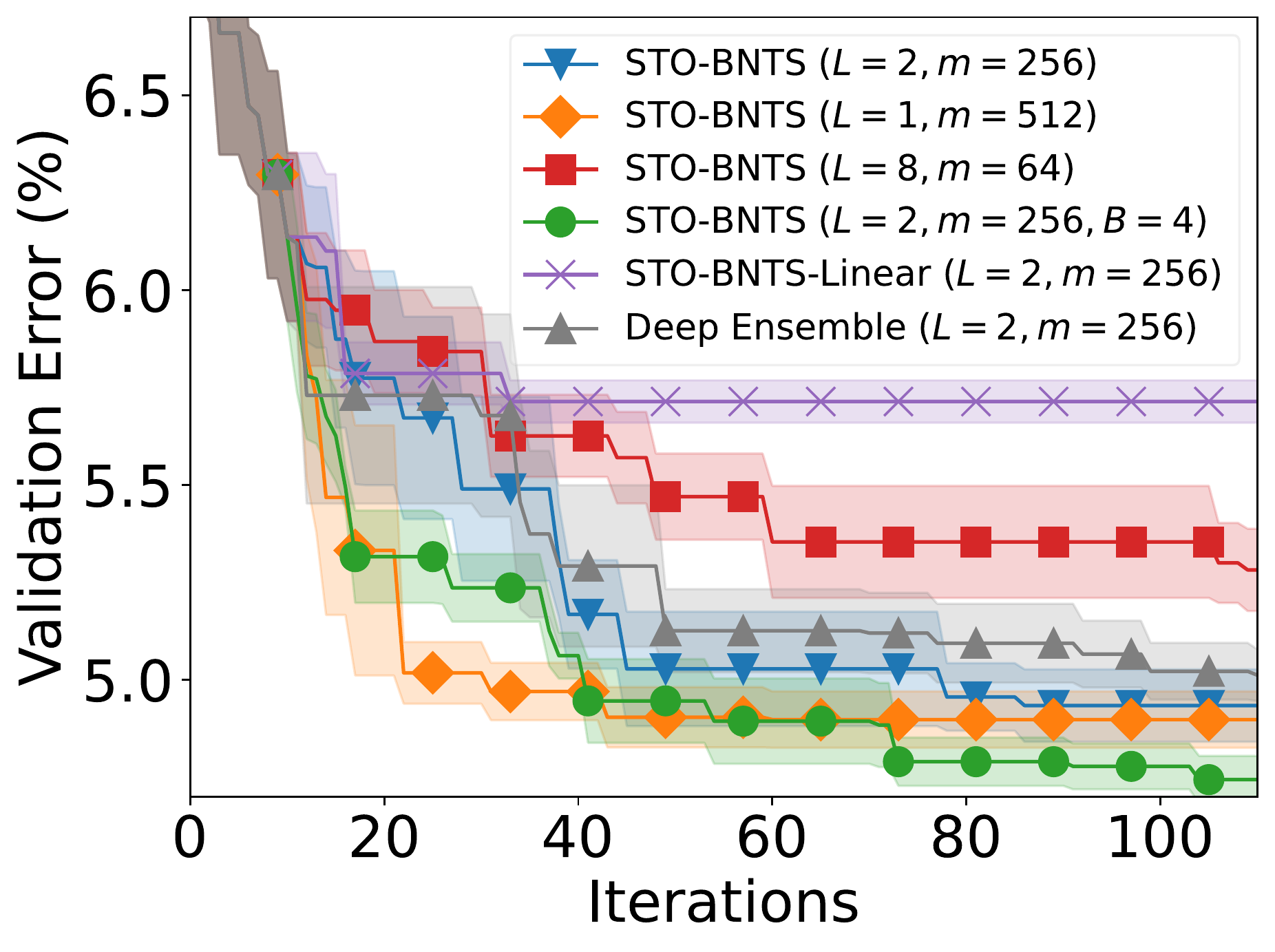} & \hspace{-5.5mm}
         \includegraphics[width=0.25\linewidth]{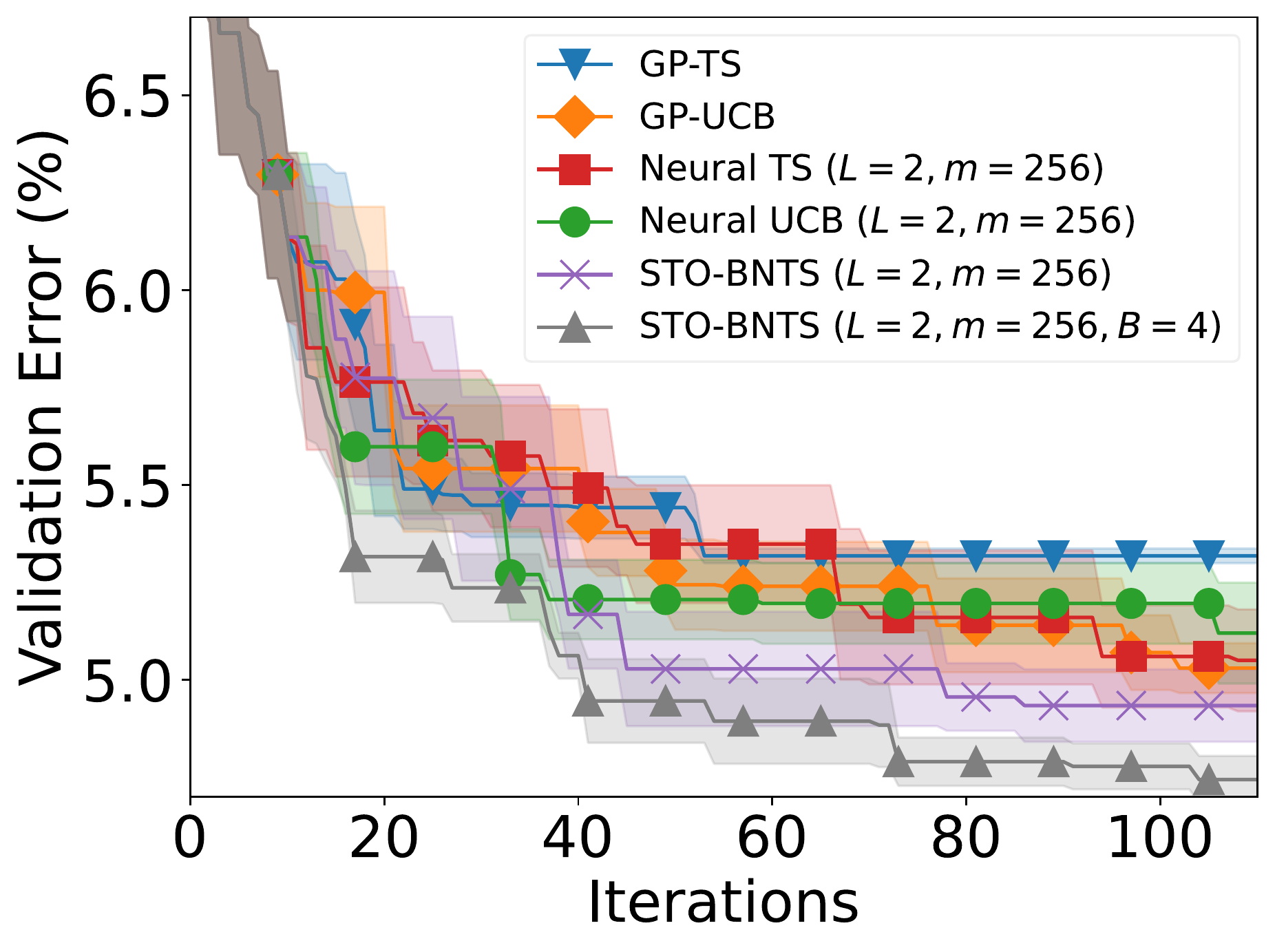}\\
         {(a)} & {(b)} & {(c)} & {(d)}
     \end{tabular}
 \vspace{-2.5mm}
     \caption{
    %  Results for AutoML experiments. 
     Validation errors for hyperparameter tuning of RF (a,b), XGBoost (c,d).
    %  , and CNN (e,f). 
    %  $B=1$ unless specified otherwise.
     $B=1$ by default.
     }
     \label{fig:exp:automl}
% \vspace{-0.1in}
\end{figure}

\vspace{-1.5mm}
\subsection{Real-world Experiments on Reinforcement Learning (RL)}
\label{subsec:exp:rl}
\vspace{-1.0mm}
\begin{figure*}[t]
     \centering
     \begin{tabular}{cccc}
         % \hspace{-4mm} \includegraphics[width=0.265\linewidth]{figures_2022_01_05/compare_lunar_one_plot.pdf} & \hspace{-6mm}
         % \includegraphics[width=0.251\linewidth]{figures_2022_01_05/compare_push_one_plot.pdf}& \hspace{-6mm}
         \hspace{-4mm} \includegraphics[width=0.265\linewidth]{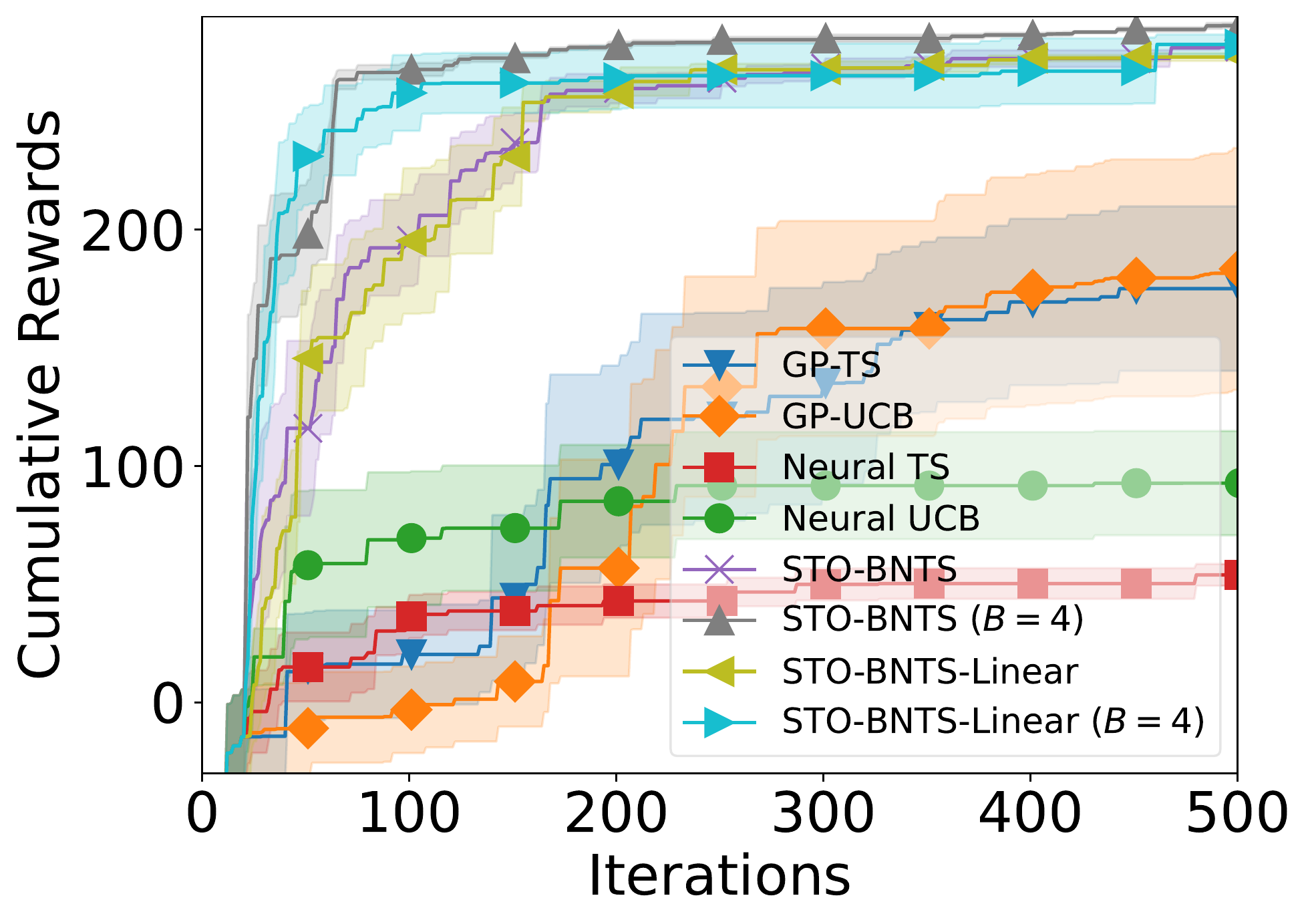} & \hspace{-6mm}
         \includegraphics[width=0.251\linewidth]{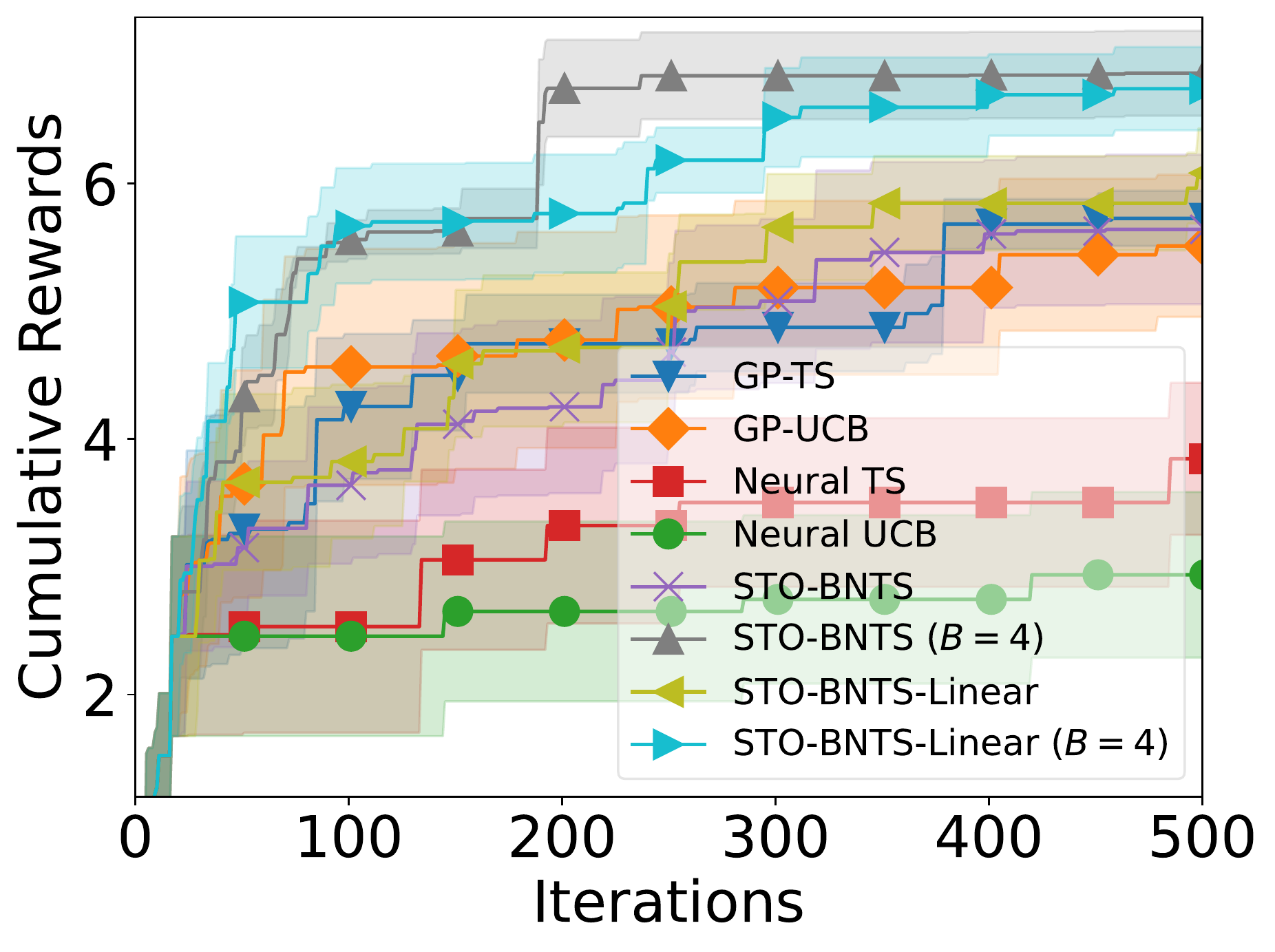}& \hspace{-6mm}
         \includegraphics[width=0.251\linewidth]{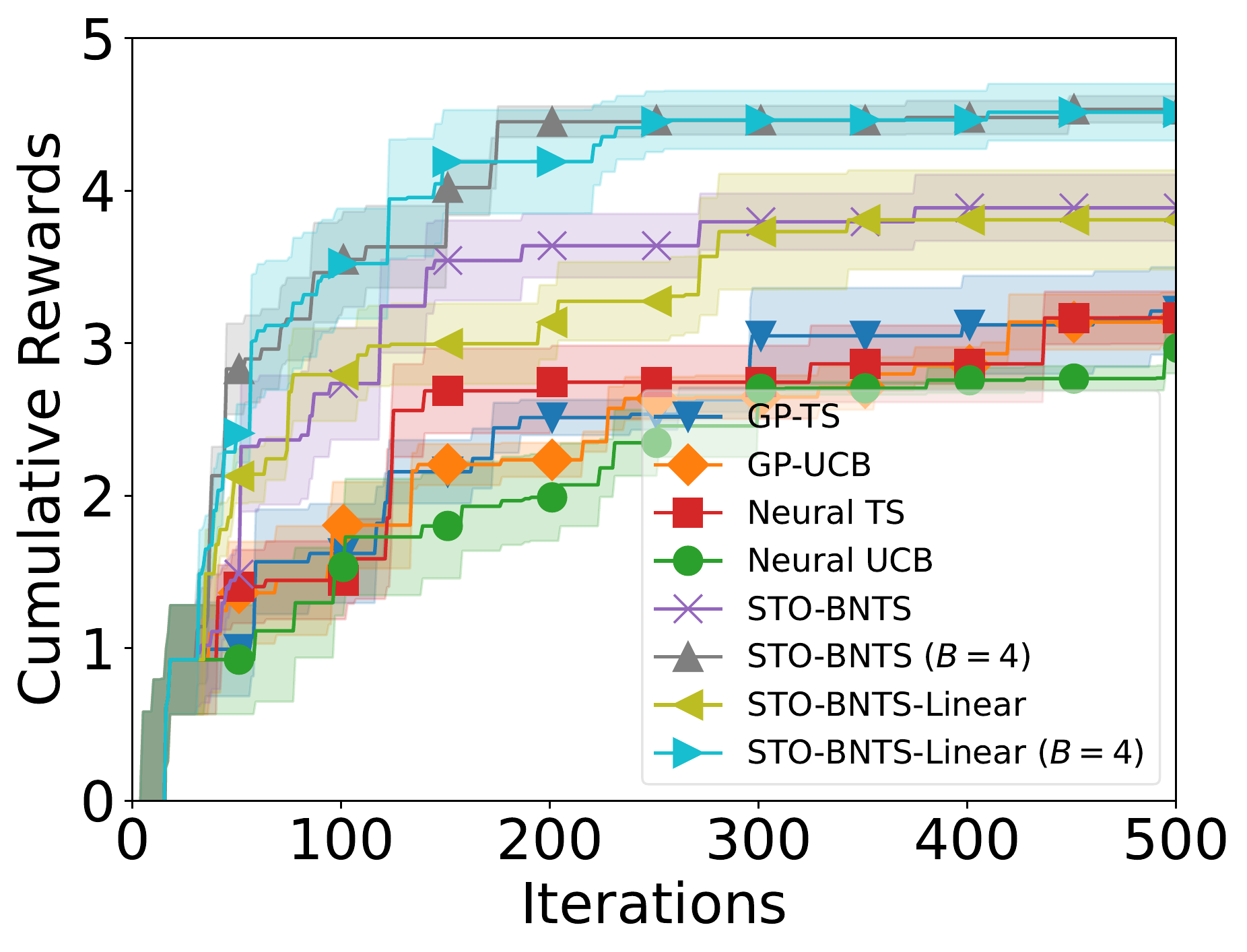}& \hspace{-6mm}
         \includegraphics[width=0.26\linewidth]{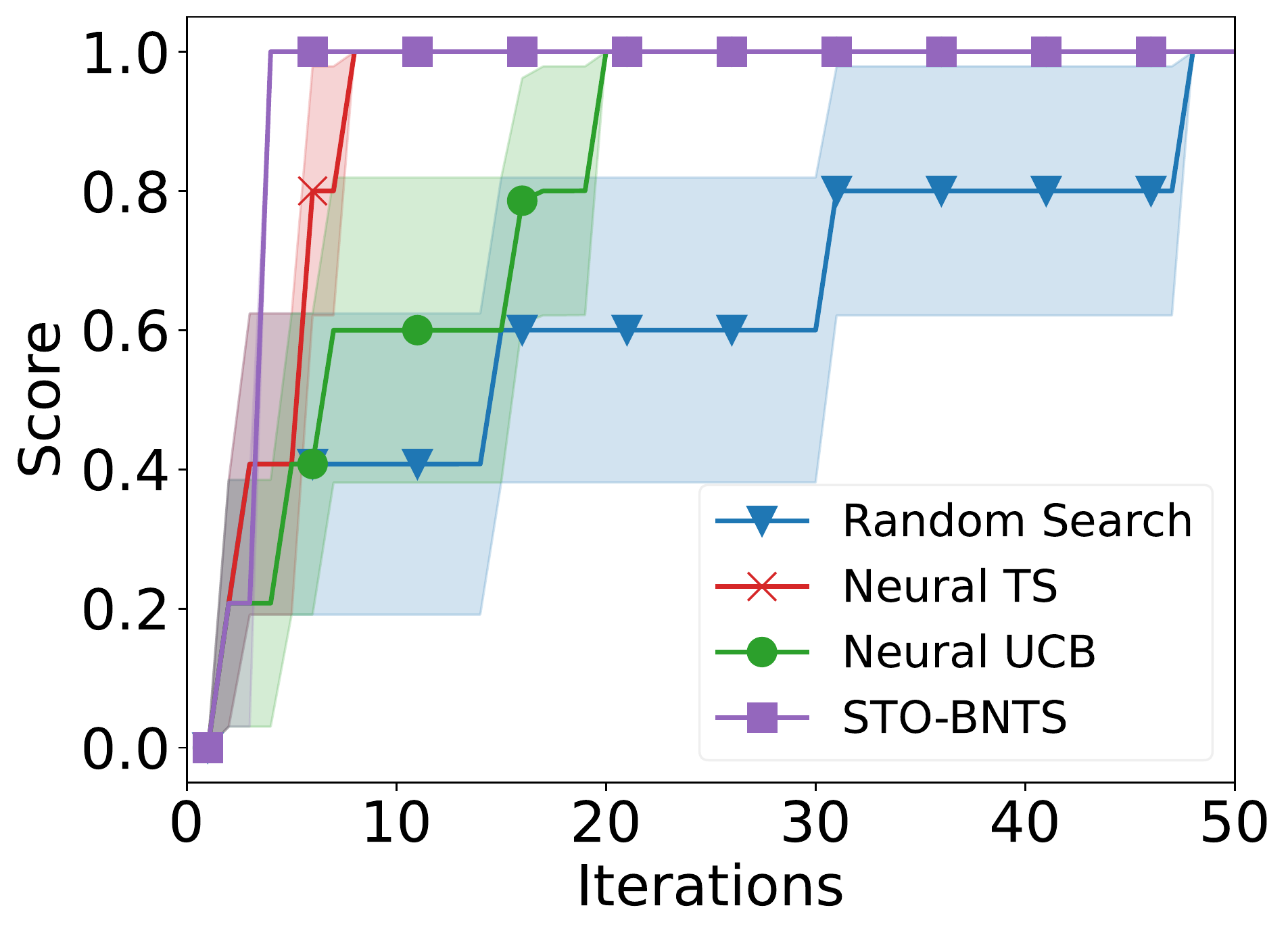}\\
         {(a)} & {(b)} & {(c)} & {(d)}
     \end{tabular}
\vspace{-2.5mm}
     \caption{
     Results for (a) $12$-D Lunar-Lander, (b) $14$-D robot pushing, (c) $20$-D rover trajectory planning, and (d) 
%     experiment on
      optimization over images in Sec.~\ref{sec:exp:images}. 
%     Every method uses 
     $B=1$ unless specified otherwise.
    %  Results for problems with continuous (or mixed) domains.
    %  (a,b) Validation errors for hyperparameter tuning of CNN on MNIST. (c,d) Cumulative rewards for the $12$-D Lunar-Lander task. (e,f) Rewards for the $14$-D robot pushing task.
     }
     \label{fig:exp:rl}
\vspace{-2mm}
\end{figure*}
Here we optimize the control parameters of 3 RL problems:
% with relatively high input dimensions.
% High-dimensional inputs space has been a recurring problem in BO, and one of the most crucial reasons why BO methods struggle in this setting is the ineffectiveness of GP in modeling high-dimensional input space~\cite{kandasamy2015high}.
we tune $d=12$ parameters of a heuristic controller for the Lunar-Lander task from OpenAI Gym~\cite{brockman2016openai}, $d=14$ parameters of a controller for a robot pushing task~\cite{wang2018batched}, and $d=20$ parameters for a rover trajectory planning problem~\cite{wang2018batched}.
We use $L=2,m=256$ for all methods.
The results 
% of different algorithms 
in these three RL tasks are plotted in Figs.~\ref{fig:exp:rl}a, b and c.
Our STO-BNTS and STO-BNTS-Linear (purple and yellow curves) consistently perform the best among all methods with sequential evaluations ($B=1$), and our methods with batch evaluations ($B=4$, gray and light blue curves) achieve further performance improvements over their sequential counterparts. 
Of note, despite underperforming 
% in the experiments 
in Sec.~\ref{subsec:exp:automl}, STO-BNTS-Linear achieves comparable performances with STO-BNTS in all three experiments here for both the sequential and batch evaluations, outperforming the other baselines.
The inefficacy of GP-TS and GP-UCB here may result from the inability of GPs to effectively model high-dimensional input space~\cite{kandasamy2015high} (Sec.~\ref{subsec:exp:discussion}).

\vspace{-1.5mm}
\subsection{Optimization over Images}
\label{sec:exp:images}
\vspace{-1.5mm}
Some real-world applications require optimizing over an input domain of \emph{images}. For example, an image recommender system sequentially recommends different images to a user in order to select the image with the best rating/score~\cite{zhang2019deep}.
In these applications, GP-based BO methods require sophisticated techniques such as convolutional kernels in order to model the inputs of images~\cite{van2017convolutional}, in contrast, our methods can be easily applied by simply replacing the NN surrogate model with a CNN.
Here, we simulate these applications by maximizing a score over the domain of images from the MNIST dataset. 
For all CNN-based methods, we use a CNN with a convolutional layer followed by a fully connected layer (both with $m=64$) as the surrogate model (more details in Appendix~\ref{app:sec:exp:images}).
The results (Fig.~\ref{fig:exp:rl}d) show that our STO-BNTS performs the best.
% We also discovered that given the same network architecture, 
Moreover, when using the same CNN architecture, STO-BNTS-Linear fails to achieve comparable performances to the other CNN-based methods since STO-BNTS-Linear is not able to fully leverage the representation power of CNN (Sec.~\ref{sec:algo}).
% (more discussions in Sec.~\ref{subsec:exp:discussion}).
However, again consistent with Theorem~\ref{theorem:regret:approc:ntk}, increasing the width of the CNN can improve the performance of STO-BNTS-Linear to be comparable with STO-BNTS (Fig.~\ref{fig:app:exp:images} in Appendix~\ref{app:sec:exp:images}).

\vspace{-1.2mm}
\subsection{Discussion}
\label{subsec:exp:discussion}
\vspace{-1.2mm}
% \paragraph{STO-BNTS vs.~STO-BNTS-Linear.}
% The results in Figs.~\ref{fig:exp:automl} and~\ref{fig:exp:rl} reveal some interesting insights regarding the relative strengths and weaknesses of our STO-BNTS (Algo.~\ref{algo:1}) and STO-BNTS-Linear (Algo.~\ref{algo:2}). Specifically, as we have discussed in Sec.~\ref{sec:algo}, since STO-BNTS-Linear does not explicitly take advantage of the representation power of NNs, it is expected to be less competitive than STO-BNTS in practice especially in problems where the strong expressive power of NNs is critical to accurately modeling the objective function. An example of such problems is hyperparameter tuning where the input space contains categorical input variables such as the three experiments in Sec.~\ref{subsec:exp:automl}. Interestingly, our experimental results corroborate this insight because STO-BNTS-Linear is consistently outperformed by STO-BNTS in Fig.~\ref{fig:exp:automl} yet performs comparably with 
% % or better than 
% STO-BNTS in Fig.~\ref{fig:exp:rl}.
Secs.~\ref{subsec:exp:automl},~\ref{subsec:exp:rl} and~\ref{sec:exp:images} show that our algorithms, 
%although \emph{not specifically designed for a particular type of problem} 
although \emph{without any special design for a specific type of problem} 
(e.g., problems involving categorical, high-dimensional or image inputs), are \emph{competitive in all these problems} thanks to the ability of NNs to model complicated real-world functions.
Compared with GP-based BO methods, our algorithms may incur more computation due to the need to train an NN surrogate model. However, the additional computation can be easily overshadowed by the cost of function evaluations since BO is usually used to optimize expensive-to-compute functions.

% \paragraph{STO-BNTS vs.~STO-BNTS-Linear.}
\textbf{STO-BNTS vs.~STO-BNTS-Linear.}
Our experimental results 
% in Secs.~\ref{subsec:exp:automl},~\ref{subsec:exp:rl} and~\ref{sec:exp:images} 
have demonstrated some interesting insights on the comparison
% on the relative strengths and weaknesses of 
between our STO-BNTS (Algo.~\ref{algo:1}) and STO-BNTS-Linear (Algo.~\ref{algo:2}). 
As we have discussed in Sec.~\ref{sec:algo}, since STO-BNTS-Linear is unable to explicitly take advantage of the representation power of NNs, it is expected to be less competitive than STO-BNTS in practice especially in problems where the strong representation power of NNs is crucial for accurately modeling the objective function.
Examples of such problems include those with categorical (Sec.~\ref{subsec:exp:automl}) or image (Sec.~\ref{sec:exp:images}) inputs.
Interestingly, our results in Secs.~\ref{subsec:exp:automl} and~\ref{sec:exp:images} indeed corroborate this insight by showing that STO-BNTS-Linear is consistently outperformed by STO-BNTS in these experiments.
However, note that in such problems, the performance of STO-BNTS-Linear can be significantly improved by further increasing the width of the NN (CNN) as we have shown in Fig.~\ref{fig:app:exp:images}.
% for the experiment in Sec.~\ref{sec:exp:images}.
In addition, the practical efficacy of STO-BNTS-Linear can also be seen from the experiments in Sec.~\ref{subsec:exp:rl}, in which STO-BNTS-Linear performs comparably with STO-BNTS and \emph{consistently outperforms all other baselines}.
This demonstrates the empirical competence of STO-BNTS-Linear in some real-world problems such as RL.
% in all experiments in Sec.~\ref{subsec:exp:rl}, STO-BNTS-Linear performs comparably with STO-BNTS and outperforms the other baselines, which indeed demonstrates the practical efficacy of STO-BNTS-Linear in some real-world applications.
Moreover, also note that compared with STO-BNTS, STO-BNTS-Linear enjoys the theoretical advantage of having a guaranteed convergence for finite-width NNs (Sec.~\ref{subsec:theory:finite}).

% Interestingly, this has been verified by our empirical results in Secs.~\ref{subsec:exp:automl} and~\ref{sec:exp:images} because the strong representation power of NNs and CNNs make them suitable for modeling functions whose input contains categorical variables (Sec.~\ref{subsec:exp:automl}) or are images (Sec.~\ref{sec:exp:images}), 
% An example of such problems is hyperparameter tuning where the input space contains categorical input variables such as the three experiments in Sec.~\ref{subsec:exp:automl}. Interestingly, our experimental results corroborate this insight because STO-BNTS-Linear is consistently outperformed by STO-BNTS in Fig.~\ref{fig:exp:automl} yet performs comparably with 
% or better than 
% STO-BNTS in Fig.~\ref{fig:exp:rl}.

% \textbf{explain why Neural UCB/Neural TS under-perform, and why GP-UCB/GP-TS under-perform}

% \vspace{-0.5mm}
% \paragraph{Baseline Methods.}
\textbf{Baseline Methods.}
The underwhelming performances of Neural UCB and Neural TS in our experiments are likely caused by the errors introduced by the diagonal matrix approximation that is used to avoid the inversion of the $p\times p$ matrix (Sec.~\ref{sec:introduction}).
The inadequate performances of GP-UCB and GP-TS may be explained by the ineffectiveness of GPs to model objective functions involving categorical~\cite{daxberger2019mixed,deshwal2021bayesian} or high-dimensional inputs~\cite{kandasamy2015high,mutny2019efficient}, which correspond to the experiments in Secs.~\ref{subsec:exp:automl} and~\ref{subsec:exp:rl}, respectively. 
Nevertheless, we expect GP-based methods to achieve better performances (than their performances here) in problems with lower-dimensional and purely continuous input space (e.g., GP-TS performs competitively in the synthetic experiment as shown in Fig.~\ref{fig:synth:illu}d).

% Comparing the performances of GP-TS in Fig.~\ref{fig:synth:illu}d with those in Figs.~\ref{fig:exp:automl} and~\ref{fig:exp:rl} reveals the relative strengths and weaknesses of of GP-based and NN-based methods. Specifically, the results imply that GP-based methods are likely more suitable for problems with smooth objective functions, whereas, in problems with complicated and non-smooth objective functions, methods which use NNs as the surrogate model are expected to perform better.

\textbf{Depth $L$ and Width $m$ of the NN Surrogate Model.}
Our experimental results have provided some guidelines on the choices of the depth $L$ and width $m$ of the NN surrogate model in our algorithms.
Regarding the depth $L$, our experiments in Sec.~\ref{subsec:exp:automl} have shown that an overly shallow NN usually hurts the performance due to its lack of expressive power, whereas an excessively deep NN is also likely to deteriorate the performance due to overfitting.
Therefore, we discourage the use of NNs which are either exceedingly shallow or overly deep, and recommend shallower NNs for simpler tasks 
% (i.e., simpler objective functions) 
to prevent overfitting and deeper NNs for more complicated tasks to gain enough representation power.
The width $m$ should be chosen to be large enough since our experiments in Sec.~\ref{sec:exp:synth} (Fig.~\ref{fig:synth:illu}d) and Sec.~\ref{sec:exp:images} (Fig.~\ref{fig:app:exp:images} in Appendix~\ref{app:sec:exp:images}) suggest that a larger width usually improves the performance.
Moreover, we have shown that the choice of $L=2,m=256$ consistently leads to competitive performances in a wide range of experiments (i.e., all experiments in Sec.~\ref{subsec:exp:automl} and Sec.~\ref{subsec:exp:rl}). 
% Therefore, we expect it to serve as a competitive baseline for many real-world problems.
Therefore, we recommend it as the default choice for real-world problems.

\vspace{-2.0mm}
\section{Related Works}
\label{sec:related:works}
\vspace{-2mm}
The NTK provides a theoretical tool to study the training dynamics of NNs by drawing connections with kernel methods~\cite{arora2019exact,cao2019generalization,dai2022federated,jacot2018neural,lee2019wide}, and it has been successfully applied to a number of practical problems such as neural architecture search \cite{nasi,hnas2022}, active learning \cite{wang2021neural}, data valuation \cite{wu2022davinz}, among others.
The 
pioneering 
work of~\cite{zhou2020neural} exploited the theory of NTK to introduce the Neural UCB algorithm for contextual bandits, which uses an NN 
%as the surrogate model 
to learn the reward (objective) function and leverages neural tangent features for exploration following the UCB principle.
% from linear bandit.
Neural UCB has been followed up by other works which mainly aimed to make Neural UCB more practical.
% \paragraph{NTK and Neural Contextual Bandits}
% The work of~\cite{zhou2020neural} has introduced the Neural UCB algorithm.
\cite{xu2020neural} proposed to improve the computational efficiency of Neural UCB through the additional assumption that the reward function is linear in the feature mapping of the last layer of the NN,~\cite{gu2021batched} also aimed to reduce the computational cost of Neural UCB by limiting the number of updates of the NN surrogate model, and~\cite{nabati2021online} proposed a method to reduce the memory requirement of generic neural bandit algorithms.
The recent work of~\cite{lisicki2021empirical} performed an empirical study of neural bandit algorithms, and discovered that they tend to perform competitively in problems that require learning complicated representations. 
% This agrees with our experimental results (Secs.~\ref{subsec:exp:automl} and~\ref{subsec:exp:rl}) because the objective functions in these experiments are complicated due to categorical and high-dimensional inputs.
% In addition, the idea of neural bandit has also been extended to handle other topics such as pure exploration~\cite{zhu2021pure}, active learning~\cite{awasthi2021neural}, visual-aware advertising through the use of convolutional NNs~\cite{ban2021convolutional}, etc.
Above all, the work that is most closely related to our paper is~\cite{zhang2020neural}, which introduced the Neural TS algorithm.
Following similar principles as Neural UCB, Neural TS learns the reward function using an NN surrogate model and constructs the exploration term through the neural tangent features, which are used as, respectively, the mean and variance of the Gaussian distribution from which the reward is sampled for running the TS routine.
GP-based BO methods \cite{bala20,dmitrii20b,topk,TMES} have achieved impressive performances in recent years, and have been extended to various problem settings such as 
high-dimensional BO \cite{eriksson2019scalable,NghiaAAAI18,kandasamy2015high},
multi-fidelity BO \cite{dai2019,kandasamy2017multi,ZhangUAI19,yehong17},
meta-BO \cite{metaBO,salinas2020quantile}, 
risk-averse BO \cite{nguyen2021conditional,nguyen2021value,SebICML22},
multi-agent/collaborative BO \cite{dai2020,sim2021collaborative},
non-myopic BO \cite{dmitrii20a,ling16},
% BO with delayed feedback \cite{verma2022bayesian}, 
among others.
More importantly, they have also been extended to the batch setting, based on either GP-UCB~\cite{contal2013parallel,daxberger2017distributed,desautels2014parallelizing,kathuria2016batched,verma2022bayesian} or GP-TS~\cite{hernandez2017parallel,kandasamy2018parallelised,nava2022diversified,verma2022bayesian}.

% \textbf{Batch Bandits and Batch BO, other BO papers which use NN or BNN as surrogates and don't have theoretical guarantee?}

% The neural TS (NTS) algorithm has been introduced by the recent work of~\cite{zhang2020neural}. The NTS algorithm~\cite{zhang2020neural} can be viewed as a TS algorithm for kernelized bandits (with the NTK kernel) using random features approximation, where the gradient of the neural network surrogate model (updated in every iteration via gradient descent training) is used as the random features.
% A potential limitation of the NTS algorithm is that due to the use of random features approximation, a matrix of size $p\times p$ needs to be inverted in every iteration, where $p$ is the total number of parameters of the neural network surrogate model and is hence required to be very large by the theory of NTK. Another potential limitation is that the theoretical analysis of NTS~\cite{zhang2020neural} is only applicable to the sequential (non-batch) setting.
% We address both of these limitations in this work through an interesting combination with Bayesian deep ensembles.

\vspace{-2.0mm}
\section{Conclusion}
\vspace{-2.0mm}
\label{sec:conclusion}
We propose STO-BNTS and STO-BNTS-Linear, 
%which use NNs as the surrogate to model the objective function.
both of which sidestep the requirement to invert a large parameter matrix in existing neural bandit algorithms, and naturally support batch evaluations while preserving their theoretical guarantees. 
%Both algorithms are asymptotically no-regret (under certain conditions) and perform effectively in real-world problems.
Both algorithms are asymptotically no-regret under certain conditions if the NN surrogate is infinite-width, and 
STO-BNTS-Linear 
% remains no-regret 
still enjoys sub-linear regret
for a finite-width NN if it is wide enough.
%We also show the empirical effectiveness of our algorithms in real-world problems.
A potential limitation is that our theoretical analysis for finite-width NNs (Sec.~\ref{subsec:theory:finite}) only holds for STO-BNTS-Linear but not for STO-BNTS, which we will explore in future works.
Another promising future topic is to leverage our ability to exploit the strong representation power of NNs to apply our algorithms to other challenging optimization tasks with sophisticated search spaces, such as chemical design \cite{griffiths2020constrained,korovina2020chembo}, neural architecture search \cite{nes,shu2019understanding}, etc.
Moreover, it is also an interesting future direction to extend our algorithms to handle other problem settings such as those which have been considered by BO (Sec.~\ref{sec:related:works}), e.g., multi-fidelity optimization \cite{ZhangUAI19,yehong17}, risk aversion/fault tolerance \cite{fan2021fault,nguyen2021conditional,nguyen2021value}, etc.
A potential negative societal impact is that our work may promote more adoption of deep learning methods and hence cause more electricity consumption.

%We propose STO-BNTS and STO-BNTS-Linear, which use NNs as the surrogate to model the objective function.
%Both algorithms sidestep the the requirement to invert a large parameter matrix in existing neural bandit algorithms, and naturally support batch evaluations while preserving their theoretical guarantees. Both algorithms are asymptotically no-regret under certain conditions if the NN surrogate is infinite-width, and STO-BNTS-Linear remains no-regret for a finite-width NN if it is wide enough.
%We also show the empirical effectiveness of our algorithms in real-world problems.
%A potential limitation is that our theoretical analysis for finite-width NNs (Sec.~\ref{subsec:theory:finite}) only holds for STO-BNTS-Linear but not for STO-BNTS, which we will explore in future works.

\begin{ack}
This research/project is supported by A*STAR under its RIE$2020$ Advanced Manufacturing and Engineering (AME) Industry Alignment Fund – Pre Positioning (IAF-PP) (Award A$19$E$4$a$0101$).
\end{ack}

\newpage

\bibliography{neural_TS}

\begin{thebibliography}{10}

\bibitem{allen2019can}
Z.~Allen-Zhu and Y.~Li.
\newblock What can {ResNet} learn efficiently, going beyond kernels?
\newblock In {\em Proc. {NeurIPS}}, 2019.

\bibitem{arora2019exact}
S.~Arora, S.~S. Du, W.~Hu, Z.~Li, R.~Salakhutdinov, and R.~Wang.
\newblock On exact computation with an infinitely wide neural net.
\newblock {arXiv}:1904.11955, 2019.

\bibitem{bala20}
S.~Balakrishnan, Q.~P. Nguyen, B.~K.~H. Low, and H.~Soh.
\newblock Efficient exploration of reward functions in inverse reinforcement
  learning via {Bayesian} optimization.
\newblock In {\em Proc. {NeurIPS}}, pages 4187--4198, 2020.

\bibitem{brockman2016openai}
G.~Brockman, V.~Cheung, L.~Pettersson, J.~Schneider, J.~Schulman, J.~Tang, and
  W.~Zaremba.
\newblock {OpenAI Gym}.
\newblock {arXiv}:1606.01540, 2016.

\bibitem{cao2019generalization}
Y.~Cao and Q.~Gu.
\newblock Generalization bounds of stochastic gradient descent for wide and
  deep neural networks.
\newblock In {\em Proc. NeurIPS}, volume~32, pages 10836--10846, 2019.

\bibitem{chen2016xgboost}
T.~Chen and C.~Guestrin.
\newblock Xgboost: {A} scalable tree boosting system.
\newblock In {\em Proceedings of the 22nd acm sigkdd international conference
  on knowledge discovery and data mining}, pages 785--794, 2016.

\bibitem{chowdhury2017kernelized}
S.~R. Chowdhury and A.~Gopalan.
\newblock On kernelized multi-armed bandits.
\newblock In {\em Proc. {ICML}}, pages 844--853, 2017.

\bibitem{contal2013parallel}
E.~Contal, D.~Buffoni, A.~Robicquet, and N.~Vayatis.
\newblock Parallel {Gaussian} process optimization with upper confidence bound
  and pure exploration.
\newblock In {\em Joint European Conference on Machine Learning and Knowledge
  Discovery in Databases}, pages 225--240. Springer, 2013.

\bibitem{dai2020}
Z.~Dai, Y.~Chen, B.~K.~H. Low, P.~Jaillet, and T.-H. Ho.
\newblock {R2-B2}: Recursive reasoning-based {Bayesian} optimization for
  no-regret learning in games.
\newblock In {\em Proc. {ICML}}, pages 2291--2301, 2020.

\bibitem{metaBO}
Z.~Dai, Y.~Chen, H.~Yu, B.~K.~H. Low, and P.~Jaillet.
\newblock On provably robust meta-{Bayesian} optimization.
\newblock In {\em Proc. {UAI}}, 2022.

\bibitem{dai2020federated}
Z.~Dai, B.~K.~H. Low, and P.~Jaillet.
\newblock Federated {Bayesian} optimization via {Thompson} sampling.
\newblock In {\em Proc. {NeurIPS}}, pages 9687--9699, 2020.

\bibitem{dai2021differentially}
Z.~Dai, B.~K.~H. Low, and P.~Jaillet.
\newblock Differentially private federated {Bayesian} optimization with
  distributed exploration.
\newblock In {\em Proc. {NeurIPS}}, pages 9125--9139, 2021.

\bibitem{dai2022federated}
Z.~Dai, Y.~Shu, A.~Verma, F.~X. Fan, B.~K.~H. Low, and P.~Jaillet.
\newblock Federated neural bandit.
\newblock {arXiv}:2205.14309, 2022.

\bibitem{dai2019}
Z.~Dai, H.~Yu, B.~K.~H. Low, and P.~Jaillet.
\newblock Bayesian optimization meets {Bayesian} optimal stopping.
\newblock In {\em Proc. {ICML}}, pages 1496--1506, 2019.

\bibitem{daxberger2019mixed}
E.~Daxberger, A.~Makarova, M.~Turchetta, and A.~Krause.
\newblock Mixed-variable {Bayesian} optimization.
\newblock In {\em Proc. {IJCAI}}, 2020.

\bibitem{daxberger2017distributed}
E.~A. Daxberger and B.~K.~H. Low.
\newblock Distributed batch {Gaussian} process optimization.
\newblock In {\em Proc. {ICML}}, pages 951--960, 2017.

\bibitem{desautels2014parallelizing}
T.~Desautels, A.~Krause, and J.~W. Burdick.
\newblock Parallelizing exploration-exploitation tradeoffs in {Gaussian}
  process bandit optimization.
\newblock {\em Journal of Machine Learning Research}, 15:3873--3923, 2014.

\bibitem{deshwal2021bayesian}
A.~Deshwal, S.~Belakaria, and J.~R. Doppa.
\newblock Bayesian optimization over hybrid spaces.
\newblock In {\em Proc. {ICML}}, 2021.

\bibitem{eriksson2019scalable}
D.~Eriksson, M.~Pearce, J.~Gardner, R.~D. Turner, and M.~Poloczek.
\newblock Scalable global optimization via local {Bayesian} optimization.
\newblock volume~32, pages 5496--5507, 2019.

\bibitem{fan2021fault}
F.~X. Fan, Y.~Ma, Z.~Dai, W.~Jing, C.~Tan, and B.~K.~H. Low.
\newblock Fault-tolerant federated reinforcement learning with theoretical
  guarantee.
\newblock In {\em Proc. {NeurIPS}}, 2021.

\bibitem{frazier2018tutorial}
P.~I. Frazier.
\newblock A tutorial on {Bayesian} optimization.
\newblock {arXiv}:1807.02811, 2018.

\bibitem{griffiths2020constrained}
R.-R. Griffiths and J.~M. Hern{\'a}ndez-Lobato.
\newblock Constrained {Bayesian} optimization for automatic chemical design
  using variational autoencoders.
\newblock {\em Chemical science}, 11(2):577--586, 2020.

\bibitem{gu2021batched}
Q.~Gu, A.~Karbasi, K.~Khosravi, V.~Mirrokni, and D.~Zhou.
\newblock Batched neural bandits.
\newblock {arXiv}:2102.13028, 2021.

\bibitem{he2020bayesian}
B.~He, B.~Lakshminarayanan, and Y.~W. Teh.
\newblock Bayesian deep ensembles via the neural tangent kernel.
\newblock In {\em Proc. {NeurIPS}}, 2020.

\bibitem{hernandez2017parallel}
J.~M. Hern{\'a}ndez-Lobato, J.~Requeima, E.~O. Pyzer-Knapp, and
  A.~Aspuru-Guzik.
\newblock Parallel and distributed {Thompson} sampling for large-scale
  accelerated exploration of chemical space.
\newblock In {\em Proc. {ICML}}, pages 1470--1479. PMLR, 2017.

\bibitem{NghiaAAAI18}
T.~N. Hoang, Q.~M. Hoang, and B.~K.~H. Low.
\newblock Decentralized high-dimensional {Bayesian} optimization with factor
  graphs.
\newblock In {\em Proc. {AAAI}}, pages 3231--3238, 2018.

\bibitem{jacot2018neural}
A.~Jacot, F.~Gabriel, and C.~Hongler.
\newblock Neural tangent kernel: Convergence and generalization in neural
  networks.
\newblock In {\em Proc. {NeurIPS}}, 2018.

\bibitem{kandasamy2017multi}
K.~Kandasamy, G.~Dasarathy, J.~Schneider, and B.~P{\'o}czos.
\newblock Multi-fidelity {Bayesian} optimisation with continuous
  approximations.
\newblock In {\em Proc. {ICML}}, pages 1799--1808. PMLR, 2017.

\bibitem{kandasamy2018parallelised}
K.~Kandasamy, A.~Krishnamurthy, J.~Schneider, and B.~P{\'o}czos.
\newblock Parallelised {Bayesian} optimisation via {Thompson} sampling.
\newblock In {\em Proc. {AISTATS}}, pages 133--142. PMLR, 2018.

\bibitem{kandasamy2015high}
K.~Kandasamy, J.~Schneider, and B.~P{\'o}czos.
\newblock High dimensional {Bayesian} optimisation and bandits via additive
  models.
\newblock In {\em Proc. {ICML}}, pages 295--304. PMLR, 2015.

\bibitem{kassraie2021neural}
P.~Kassraie and A.~Krause.
\newblock Neural contextual bandits without regret.
\newblock {arXiv}:2107.03144, 2021.

\bibitem{kathuria2016batched}
T.~Kathuria, A.~Deshpande, and P.~Kohli.
\newblock Batched {Gaussian} process bandit optimization via determinantal
  point processes.
\newblock In {\em Proc. NeurIPS}, volume~29, pages 4206--4214, 2016.

\bibitem{dmitrii20b}
D.~Kharkovskii, Z.~Dai, and B.~K.~H. Low.
\newblock Private outsourced {Bayesian} optimization.
\newblock In {\em Proc. {ICML}}, pages 5231--5242, 2020.

\bibitem{dmitrii20a}
D.~Kharkovskii, C.~K. Ling, and B.~K.~H. Low.
\newblock Nonmyopic {Gaussian} process optimization with macro-actions.
\newblock In {\em Proc. AISTATS}, pages 4593--4604, 2020.

\bibitem{korovina2020chembo}
K.~Korovina, S.~Xu, K.~Kandasamy, W.~Neiswanger, B.~Poczos, J.~Schneider, and
  E.~Xing.
\newblock Chembo: {Bayesian} optimization of small organic molecules with
  synthesizable recommendations.
\newblock In {\em Proc. {AISTATS}}, pages 3393--3403. PMLR, 2020.

\bibitem{lakshminarayanan2016simple}
B.~Lakshminarayanan, A.~Pritzel, and C.~Blundell.
\newblock Simple and scalable predictive uncertainty estimation using deep
  ensembles.
\newblock In {\em Proc. {NeurIPS}}, 2017.

\bibitem{lecun2015deep}
Y.~LeCun, Y.~Bengio, and G.~Hinton.
\newblock Deep learning.
\newblock {\em nature}, 521(7553):436--444, 2015.

\bibitem{lee2019wide}
J.~Lee, L.~Xiao, S.~Schoenholz, Y.~Bahri, R.~Novak, J.~Sohl-Dickstein, and
  J.~Pennington.
\newblock Wide neural networks of any depth evolve as linear models under
  gradient descent.
\newblock volume~32, pages 8572--8583, 2019.

\bibitem{li2021gaussian}
Z.~Li and J.~Scarlett.
\newblock Gaussian process bandit optimization with few batches.
\newblock In {\em Proc. {AISTATS}}, pages 92--107, 2022.

\bibitem{ling16}
C.~K. Ling, B.~K.~H. Low, and P.~Jaillet.
\newblock {Gaussian} process planning with {Lipschitz} continuous reward
  functions: Towards unifying {Bayesian} optimization, active learning, and
  beyond.
\newblock In {\em Proc. {AAAI}}, pages 1860--1866, 2016.

\bibitem{lisicki2021empirical}
M.~Lisicki, A.~Afkanpour, and G.~W. Taylor.
\newblock An empirical study of neural kernel bandits.
\newblock In {\em NeurIPS Workshop on Bayesian Deep Learning}, 2021.

\bibitem{matthews2017sample}
A.~G. d.~G. Matthews, J.~Hron, R.~E. Turner, and Z.~Ghahramani.
\newblock Sample-then-optimize posterior sampling for bayesian linear models.
\newblock In {\em NeurIPS Workshop on Advances in Approximate Bayesian
  Inference}, 2017.

\bibitem{mutny2019efficient}
M.~Mutn{\`y} and A.~Krause.
\newblock Efficient high dimensional {Bayesian} optimization with additivity
  and quadrature {Fourier} features.
\newblock In {\em Proc. NeurIPS}, pages 9005--9016. Curran, 2019.

\bibitem{nabati2021online}
O.~Nabati, T.~Zahavy, and S.~Mannor.
\newblock Online limited memory neural-linear bandits with likelihood matching.
\newblock In {\em Proc. {ICML}}, 2021.

\bibitem{nava2022diversified}
E.~Nava, M.~Mutny, and A.~Krause.
\newblock Diversified sampling for batched bayesian optimization with
  determinantal point processes.
\newblock In {\em Proc. {AISTATS}}, pages 7031--7054. PMLR, 2022.

\bibitem{nguyen2021conditional}
Q.~P. Nguyen, Z.~Dai, B.~K.~H. Low, and P.~Jaillet.
\newblock Optimizing conditional value-at-risk of black-box functions.
\newblock In {\em Proc. NeurIPS}, pages 4170--4180, 2021.

\bibitem{nguyen2021value}
Q.~P. Nguyen, Z.~Dai, B.~K.~H. Low, and P.~Jaillet.
\newblock Value-at-risk optimization with {Gaussian} processes.
\newblock In {\em Proc. ICML}, pages 8063--8072, 2021.

\bibitem{topk}
Q.~P. Nguyen, S.~Tay, B.~K.~H. Low, and P.~Jaillet.
\newblock Top-$k$ ranking {Bayesian} optimization.
\newblock In {\em Proc. {AAAI}}, pages 9135--9143, 2021.

\bibitem{TMES}
Q.~P. Nguyen, Z.~Wu, B.~K.~H. Low, and P.~Jaillet.
\newblock Trusted-maximizers entropy search for efﬁcient {Bayesian}
  optimization.
\newblock In {\em Proc. {UAI}}, pages 1486--1495, 2021.

\bibitem{salinas2020quantile}
D.~Salinas, H.~Shen, and V.~Perrone.
\newblock A quantile-based approach for hyperparameter transfer learning.
\newblock In {\em Proc. {ICML}}, pages 8438--8448. PMLR, 2020.

\bibitem{shahriari2016taking}
B.~Shahriari, K.~Swersky, Z.~Wang, R.~P. Adams, and N.~{de Freitas}.
\newblock Taking the human out of the loop: A review of {Bayesian}
  optimization.
\newblock {\em Proceedings of the IEEE}, 104(1):148--175, 2016.

\bibitem{nasi}
Y.~Shu, S.~Cai, Z.~Dai, B.~C. Ooi, and B.~K.~H. Low.
\newblock {NASI}: Label- and data-agnostic neural architecture search at
  initialization.
\newblock In {\em Proc. {ICLR}}, 2022.

\bibitem{nes}
Y.~Shu, Y.~Chen, Z.~Dai, and B.~K.~H. Low.
\newblock Neural ensemble search via {Bayesian} sampling.
\newblock In {\em Proc. {UAI}}, 2022.

\bibitem{hnas2022}
Y.~Shu, Z.~Dai, Z.~Wu, and B.~K.~H. Low.
\newblock Unifying and boosting gradient-based training-free neural
  architecture search.
\newblock In {\em Proc. {NeurIPS}}, 2022.

\bibitem{shu2019understanding}
Y.~Shu, W.~Wang, and S.~Cai.
\newblock Understanding architectures learnt by cell-based neural architecture
  search.
\newblock In {\em Proc. {ICLR}}, 2020.

\bibitem{sim2021collaborative}
R.~H.~L. Sim, Y.~Zhang, B.~K.~H. Low, and P.~Jaillet.
\newblock Collaborative {Bayesian} optimization with fair regret.
\newblock In {\em Proc. {ICML}}, pages 9691--9701, 2021.

\bibitem{srinivas2009gaussian}
N.~Srinivas, A.~Krause, S.~M. Kakade, and M.~Seeger.
\newblock {Gaussian} process optimization in the bandit setting: No regret and
  experimental design.
\newblock In {\em Proc. {ICML}}, pages 1015--1022, 2010.

\bibitem{SebICML22}
S.~S. Tay, C.~S. Foo, D.~Urano, R.~C.~X. Leong, and B.~K.~H. Low.
\newblock Efficient distributionally robust {Bayesian} optimization with
  worst-case sensitivity.
\newblock In {\em Proc. ICML}, 2022.

\bibitem{van2017convolutional}
M.~Van~der Wilk, C.~E. Rasmussen, and J.~Hensman.
\newblock Convolutional {Gaussian} processes.
\newblock In {\em Proc. {NeurIPS}}, 2017.

\bibitem{verma2022bayesian}
A.~Verma, Z.~Dai, and B.~K.~H. Low.
\newblock Bayesian optimization under stochastic delayed feedback.
\newblock In {\em Proc. {ICML}}, pages 22145--22167. PMLR, 2022.

\bibitem{wang2021neural}
Z.~Wang, P.~Awasthi, C.~Dann, A.~Sekhari, and C.~Gentile.
\newblock Neural active learning with performance guarantees.
\newblock {\em Proc. {NeurIPS}}, 34:7510--7521, 2021.

\bibitem{wang2018batched}
Z.~Wang, C.~Gehring, P.~Kohli, and S.~Jegelka.
\newblock Batched large-scale {Bayesian} optimization in high-dimensional
  spaces.
\newblock In {\em Proc. {AISTATS}}, pages 745--754. PMLR, 2018.

\bibitem{wu2022davinz}
Z.~Wu, Y.~Shu, and B.~K.~H. Low.
\newblock Davinz: Data valuation using deep neural networks at initialization.
\newblock In {\em International Conference on Machine Learning}, pages
  24150--24176. PMLR, 2022.

\bibitem{xu2020neural}
P.~Xu, Z.~Wen, H.~Zhao, and Q.~Gu.
\newblock Neural contextual bandits with deep representation and shallow
  exploration.
\newblock {arXiv}:2012.01780, 2020.

\bibitem{zhang2019deep}
S.~Zhang, L.~Yao, A.~Sun, and Y.~Tay.
\newblock Deep learning based recommender system: {A} survey and new
  perspectives.
\newblock {\em ACM Computing Surveys (CSUR)}, 52(1):1--38, 2019.

\bibitem{zhang2020neural}
W.~Zhang, D.~Zhou, L.~Li, and Q.~Gu.
\newblock Neural {Thompson} sampling.
\newblock In {\em Proc. {ICLR}}, 2021.

\bibitem{ZhangUAI19}
Y.~Zhang, Z.~Dai, and B.~K.~H. Low.
\newblock Bayesian optimization with binary auxiliary information.
\newblock In {\em Proc. UAI}, pages 1222--1232, 2019.

\bibitem{yehong17}
Y.~Zhang, T.~N. Hoang, B.~K.~H. Low, and M.~Kankanhalli.
\newblock Information-based multi-fidelity {Bayesian} optimization.
\newblock In {\em Proc. {NIPS} Workshop on {Bayesian} Optimization}, 2017.

\bibitem{zhou2020neural}
D.~Zhou, L.~Li, and Q.~Gu.
\newblock Neural contextual bandits with {UCB}-based exploration.
\newblock In {\em Proc. {ICML}}, pages 11492--11502. PMLR, 2020.

\end{thebibliography}
\bibliographystyle{abbrv}

%%%%%%%%%%%%%%%%%%%%%%%%%%%%%%%%%%%%%%%%%%%%%%%%%%%%%%%%%%%%
\section*{Checklist}

% %%% BEGIN INSTRUCTIONS %%%
% The checklist follows the references.  Please
% read the checklist guidelines carefully for information on how to answer these
% questions.  For each question, change the default \answerTODO{} to \answerYes{},
% \answerNo{}, or \answerNA{}.  You are strongly encouraged to include a {\bf
% justification to your answer}, either by referencing the appropriate section of
% your paper or providing a brief inline description.  For example:
% \begin{itemize}
%   \item Did you include the license to the code and datasets? \answerYes{See Section~\ref{gen_inst}.}
%   \item Did you include the license to the code and datasets? \answerNo{The code and the data are proprietary.}
%   \item Did you include the license to the code and datasets? \answerNA{}
% \end{itemize}
% Please do not modify the questions and only use the provided macros for your
% answers.  Note that the Checklist section does not count towards the page
% limit.  In your paper, please delete this instructions block and only keep the
% Checklist section heading above along with the questions/answers below.
% %%% END INSTRUCTIONS %%%

\begin{enumerate}

\item For all authors...
\begin{enumerate}
  \item Do the main claims made in the abstract and introduction accurately reflect the paper's contributions and scope?
    \answerYes{}
  \item Did you describe the limitations of your work?
    \answerYes{Refer to Section \ref{sec:conclusion}.}
  \item Did you discuss any potential negative societal impacts of your work?
    \answerYes{Refer to Section \ref{sec:conclusion}.}
  \item Have you read the ethics review guidelines and ensured that your paper conforms to them?
    \answerYes{}
\end{enumerate}

\item If you are including theoretical results...
\begin{enumerate}
  \item Did you state the full set of assumptions of all theoretical results?
    \answerYes{Refer to Sections \ref{subsec:theory:infinite} and \ref{subsec:theory:finite}.}
        \item Did you include complete proofs of all theoretical results?
    \answerYes{Refer to Appendices \ref{app:sec:infinite:width} and \ref{app:sec:finite:width}.}
\end{enumerate}

\item If you ran experiments...
\begin{enumerate}
  \item Did you include the code, data, and instructions needed to reproduce the main experimental results (either in the supplemental material or as a URL)?
    \answerYes{The code is uploaded as part of the Supplementary Material.}
  \item Did you specify all the training details (e.g., data splits, hyperparameters, how they were chosen)?
    \answerYes{Refer to Appendix \ref{app:sec:experimental:details}.}
        \item Did you report error bars (e.g., with respect to the random seed after running experiments multiple times)?
    \answerYes{}
        \item Did you include the total amount of compute and the type of resources used (e.g., type of GPUs, internal cluster, or cloud provider)?
    \answerYes{Refer to Appendix \ref{app:sec:experimental:details}.}
\end{enumerate}

\item If you are using existing assets (e.g., code, data, models) or curating/releasing new assets...
\begin{enumerate}
  \item If your work uses existing assets, did you cite the creators?
    \answerYes{Refer to Appendix \ref{app:sec:experimental:details}.}
  \item Did you mention the license of the assets?
    \answerYes{Refer to Appendix \ref{app:sec:experimental:details}.}
  \item Did you include any new assets either in the supplemental material or as a URL?
    \answerNo{}
  \item Did you discuss whether and how consent was obtained from people whose data you're using/curating?
    \answerYes{Refer to Appendix \ref{app:sec:experimental:details}.}
  \item Did you discuss whether the data you are using/curating contains personally identifiable information or offensive content?
    \answerYes{Refer to Appendix \ref{app:sec:experimental:details}.}
\end{enumerate}

\item If you used crowdsourcing or conducted research with human subjects...
\begin{enumerate}
  \item Did you include the full text of instructions given to participants and screenshots, if applicable?
    \answerNA{}
  \item Did you describe any potential participant risks, with links to Institutional Review Board (IRB) approvals, if applicable?
    \answerNA{}
  \item Did you include the estimated hourly wage paid to participants and the total amount spent on participant compensation?
    \answerNA{}
\end{enumerate}

\end{enumerate}

\newpage 
\appendix
% \onecolumn

\section{Justification for The Acquisition Functions Being Sampled from GPs}
\label{app:sec:justification:gp:sampling}
For our Algo.~\ref{algo:1}, the work of~\cite{he2020bayesian} has shown that if $\beta_t=1$, then after running lines 4-7 of Algo.~\ref{algo:1}, the resulting function $f^i_t(\mathbf{x};\theta^i_t)$ corresponds to a function 
sampled from the GP posterior with the NTK as the kernel function: 
% $\mathcal{GP}(\mu_{\text{fb}[t]}(\cdot),\sigma^2_{\text{fb}[t]}(\cdot,\cdot))$. 
$\mathcal{GP}(\mu_{t-1}(\cdot),\sigma^2_{t-1}(\cdot,\cdot))$ conditioned on the $(t-1)\times B$ observations from the first $t-1$ iterations. The GP posterior mean and covariance function are expressed as:
\begin{equation}
% \begin{split}
    \mu_{t-1}(\mathbf{x}) \triangleq \mathbf{k}_{t-1}(\mathbf{x})^\top(\mathbf{K}_{t-1}+\sigma^2\mathbf{I})^{-1}\mathbf{y}_{t-1}\ ,
    % ,\qquad \qquad 
    % \sigma_t^2(\mathbf{x},\mathbf{x}') \triangleq k(\mathbf{x},\mathbf{x}')-\mathbf{k}_{t-1}(\mathbf{x})^\top(\mathbf{K}_{t-1}+\sigma^2\mathbf{I})^{-1}\mathbf{k}_{t-1}(\mathbf{x}').
% \end{split}
\label{gp_posterior}
\end{equation}
\begin{equation}
% \begin{split}
    % \mu_{t-1}(\mathbf{x}) \triangleq \mathbf{k}_{t-1}(\mathbf{x})^\top(\mathbf{K}_{t-1}+\sigma^2\mathbf{I})^{-1}\mathbf{y}_{t-1}\ ,\qquad \qquad 
    \sigma_{t-1}^2(\mathbf{x},\mathbf{x}') \triangleq k(\mathbf{x},\mathbf{x}')-\mathbf{k}_{t-1}(\mathbf{x})^\top(\mathbf{K}_{t-1}+\sigma^2\mathbf{I})^{-1}\mathbf{k}_{t-1}(\mathbf{x}')
% \end{split}
\label{gp_posterior:variance}
\end{equation}
where $\mathbf{k}_{t-1}(\mathbf{x})\triangleq (k(\mathbf{x}, \mathbf{x}^i_{\tau}))^{\top}_{\tau=1,\ldots,t-1,i=1,\ldots,B}$ which is a $(t-1)B-$dimensional vector, $\mathbf{y}_t\triangleq (y^i_{\tau})^{\top}_{\tau=1,\ldots,t-1, i=1,\ldots,B}$ which is also a $(t-1)B-$dimensional vector, and $\mathbf{K}_{t}\triangleq (k(\mathbf{x}^i_{\tau}, \mathbf{x}^{i'}_{\tau'}))_{\tau=1,\ldots,t-1,i=1,\ldots,B;\tau'=1,\ldots,t-1,i'=1,\ldots,B}$ which is a $(t-1)B \times (t-1)B-$dimensional squared matrix.

For our Algo.~\ref{algo:2}, when $\beta_t=1$, because every run of the procedure in lines 4-6 corresponds to running the sample-then-optimize method~\cite{matthews2017sample} while treating the neural tangent features as the input features, therefore, 
the resulting linear function $f^i_t(\mathbf{x};\theta^i_t)$ w.r.t. $\theta^i_t$
%sy:
% the resulting linear function $f^i_t(\mathbf{x};\theta^i_t)$ w.r.t. $\theta$
also corresponds to a sampled function from the GP posterior 
% $\mathcal{GP}(\mu_{\text{fb}[t]}(\cdot),\sigma^2_{\text{fb}[t]}(\cdot,\cdot))$ 
$\mathcal{GP}(\mu_{t-1}(\cdot),\sigma^2_{t-1}(\cdot,\cdot))$ 
with the NTK (if the NN is infinite-width) or empirical NTK (if the NN if finite-width) as the kernel function according to the work of~\cite{matthews2017sample}.

Next, for both Algo.~\ref{algo:1} and Algo.~\ref{algo:2}, when $\beta_t = 2\log(\pi^2t^2|\mathcal{X}|/\delta)$, since we have multiplied the output of NN by $\beta_t$ which corresponds to multiplying the gradient of the NN by $\beta_t$, therefore, the resulting NTK 
% after multiplication by $\beta_t$ 
will be multiplied by $\beta_t^2$. Also note that we have also multiplied the noise variance $\sigma^2$ by $\beta_t^2$ in \eqref{eq:loss:function}. As a result, after plugging these two changes into the equations for GP posterior mean~\eqref{gp_posterior} and variance~\eqref{gp_posterior:variance}, it is easy to verify that the GP posterior variance will be multiplied by $\beta_t^2$ while the GP posteior mean is unchanged.

\section{GP Posterior Variance with NTK}
\label{app:gp:posterior:variance:for:ntk}
When using the NTK as the kernel function, the GP posterior variance~\eqref{gp_posterior:variance} at any input $\mathbf{x}$ can be easily approximated by
\begin{equation}
\sigma^2_{t-1}(\mathbf{x},\mathbf{x}) \approx \nabla_{\theta}f(\mathbf{x};\theta_0)^{\top}\left[\Sigma_{t-1} + \sigma^2 I \right]^{-1} \nabla_{\theta}f(\mathbf{x};\theta_0), 
\label{eq:app:gp:posterior:variance}
\end{equation}
in which
\begin{equation}
    \Sigma_{t-1} = \sum^{t-1}_{\tau=0}\sum^{B}_{i=1} \nabla_{\theta}f(\mathbf{x}^i_{\tau};\theta_0) \nabla_{\theta}f(\mathbf{x}^i_{\tau};\theta_0)^{\top},
\end{equation}
and $\theta_0\sim\text{init}(\cdot)$ are randomly initialized parameters.
Therefore, to run the uncertainty sampling algorithm as the initialization stage, we simply need to sequentially maximize equation~\eqref{eq:app:gp:posterior:variance}, i.e., in iteration $t$ of the initialization stage, we simply choose the next initial input $\mathbf{x}$ by maximizing equation~\eqref{eq:app:gp:posterior:variance}.

\section{Proof of Theorem~\ref{theorem:regret:exact:ntk}}
\label{app:sec:infinite:width}
% We analyze here with the simpler assumption, where we assume that $f$ is sampled from a GP with the NTK covariance function denoted as $k$. Note that with this simpler assumption, we need to additionally assume that the objective function is bounded: $|f(\mathbf{x})|\leq B',\forall\mathbf{x}\in\mathcal{X}$.
Here, to simplify the analysis, we follow the work of~\cite{desautels2014parallelizing} and reparameterize the iterations to view our algorithms in the sequential setting. 
Specifically, in the main text (Algos.~\ref{algo:1} and~\ref{algo:2}), every $B$ function evaluations are counted as an iteration $t$; however, we reparameterize the iterations such that every query selection is counted as an iteration $t$. That is, every time an input query is selected, we increment the number of iterations by $1$.
As a result, before the reparameterization, the cumulative regret is expressed as $R_T=\sum^{T/B}_{t=1}\sum^{B}_{i=1}(f(\mathbf{x}^*)-f(\mathbf{x}^i_t))$; after reparameterization, the same cumulative regret is now expressed as $R_T=\sum^T_{t=1}(f(\mathbf{x}^*)-f(\mathbf{x}_t))$. 
Note that when $B=1$, the two parameterizations are the same.
Therefore, in the entire proof in this section, we index the iterations sequentially by $1,2,\ldots,t,t+1,\ldots,T$.

At iteration $t$, we use $\text{fb[t]}$ to denote the largest iteration index whose observation has been collected. 
% As a result of this definition, at iteration $t$, if the selected query is the first in the batch (i.e., there is no pending observations at this point), then we have that $\text{fb}[t]=t-1$; if the query is the second in the batch (i.e., there is $1$ pending observation at this point), then $\text{fb}[t]=t-2$; if the query is the last in the batch (i.e., there are $B-1$ pending observations at this point), then $\text{fb}[t]=t-B$.
For example, if the batch size is $B=3$, assuming that after the most recent batch of inputs have been collected, we have in total gathered $t-1$ observations; 
then when selecting the input queries in iterations $t$, $t+1$ and $t+2$, we have that $\text{fb}[t]=\text{fb}[t+1]=\text{fb}[t+2]=t-1$, because the index of the most recent observation is fixed at $t-1$ since we do not collect any new observations during this process. Next, when choosing the input query at iterations $t+3$, $t+4$ and $t+5$, we have that $\text{fb}[t+3]=\text{fb}[t+4]=\text{fb}[t+5]=t+2$.
As a result of our reparameterization here, the requirement on the constant $C$ from Theorem~\ref{theorem:regret:exact:ntk} should be slightly modified into:~$\max_{A\subset\mathcal{X},|A|\leq B-1}\mathbb{I}(f;\mathbf{y}_{A} | \mathbf{y}_{1:\text{fb}[t]}) \leq C,\forall t\geq 1$, where $\mathbf{y}_{1:\text{fb}[t]}$ represents the output observations from iterations $1$ to $\text{fb}[t]$.

Our reparameterization mentioned above allows us to derive more general theoretical results which hold for both synchronous and asynchronous batch BO.
In the setting of synchronous batch evaluations which we have focused on in the main text, $\max\{t-B,0\} \leq \text{fb}[t] \leq t-1$.
% Note that in the asynchronous setting, $\text{fb}[t] = \max\{t-B,0\}$, where as for the synchronous setting, $\max\{t-B,0\} \leq \text{fb}[t] \leq t-1$.
% Importantly, i
In this setting, $\text{fb}[t]$ is a deterministic function of $t$, which is determined before the algorithm starts. 
% For the slightly more trickier asynchronous setting, we may need to additionally \emph{assume that the evaluations of the inputs in a batch are completed in the same order as they are selected}, i.e., assume that the evaluation of every input $\mathbf{x}\in\mathcal{X}$ takes the same amount of computational time. But this can be easily relaxed in practice.
Of note, although we focus on the setting of synchronous batch BO in our theoretical analysis, the only requirement of our theoretical analysis on $\text{fb}[t]$ is that $t-\text{fb}[t] \leq B$, i.e., the number of pending observations $t-\text{fb}[t] - 1$ should be upper-bounded by $B-1$. Therefore, our theoretical results also hold in the setting of asynchronous batch BO, because the number of pending observations in asynchronous batch BO is always equal to $B-1$~\cite{desautels2014parallelizing}.

Denote as $\mu_{\text{fb}[t]}$ and $\sigma_{\text{fb}[t]}$ the GP posterior mean and standard deviation conditioned on the observations from iteration $1$ to $\text{fb}[t]$.
Define $\mathcal{F}_{t-1}=\{\mathbf{x}_1,y_1,\ldots,\mathbf{x}_{\text{fb}[t]},y_{\text{fb}[t]},\mathbf{x}_{\text{fb}[t]+1},\ldots,\mathbf{x}_{t-1}\}$ as the history of selected inputs and observed outputs for those completed observations, as well as the selected inputs of those pending observations.
Define $\beta_t = 2\log(\pi^2t^2|\mathcal{X}|/(3\delta))$, and $c_t = \beta_t (1+\sqrt{2\log(|\mathcal{X}|t^2)})$.
% Due to our assumption that $f$ is sampled from a GP, the following Lemma on the concentration of the function $f$ follows immediately.
\begin{lemma}
\label{lemma:event:Ef}
Choose $\delta\in(0,1)$. Define $E^{f}(t)$ as the event that $|\mu_{\text{fb}[t]}(\mathbf{x}) - f(\mathbf{x})| \leq \beta_t\sigma_{\text{fb}[t]}(\mathbf{x}),\forall \mathbf{x}\in\mathcal{X}$. We have that $\mathbb{P}(E^{f}(t)) \geq 1-\delta/2,\forall t\geq1$.
\end{lemma}
The proof of Lemma~\ref{lemma:event:Ef}, which follows from the proof of Lemma 5.1 of~\cite{srinivas2009gaussian}, makes use of our assumption that $f$ is sampled from a GP and relies on simple applications of the concentration of Gaussian distributions and union bounds.
Denote by $f_t$ the acquisition function in iteration $t$, which is sampled from the GP posterior with the NTK as the kernel function: $f_t \sim \mathcal{GP}(\mu_{\text{fb}[t]}(\cdot), \beta_t^2\sigma^2_{\text{fb}[t]}(\cdot,\cdot))$ as we have justified in Appendix~~\ref{app:sec:justification:gp:sampling}.
Note that the query in iteration $t$ is selected by $\mathbf{x}_t={\arg\max}_{\mathbf{x}\in\mathcal{X}}f_t(\mathbf{x})$, which corresponds to line 8 of Algo.~\ref{algo:1} and line 7 of Algo.~\ref{algo:2} respectively.
\begin{lemma}
\label{lemma:event:Eft}
Define $E^{f_t}(t)$ as the event that $|\mu_{\text{fb}[t]}(\mathbf{x}) - f_t(\mathbf{x})| \leq \beta_t \sqrt{2\log(|\mathcal{X}|t^2)} \sigma_{\text{fb}[t]}(\mathbf{x}),\forall \mathbf{x}\in\mathcal{X}$. We have that $\mathbb{P}(E^{f_t}(t)) \geq 1-1/t^2,\forall t\geq1$.
\end{lemma}
The proof of Lemma~\ref{lemma:event:Eft} follows from Lemma 5 of the work of~\cite{chowdhury2017kernelized}.
Importantly, conditioned on both events $E^f(t)$ and $E^{f_t}(t)$, we have that
\begin{equation}
|f(\mathbf{x})-f_t(\mathbf{x})| \leq c_t \sigma_{\text{fb}[t]}(\mathbf{x}).
\end{equation}

We next define the set of \emph{saturated points}, which can be understood as the set of undesirable points in every iteration.
\begin{definition}
\label{def:saturated_set}
Define the set of saturated inputs in iteration $t$ as
\[
S_t = \{ \mathbf{x} \in \mathcal{X} : \Delta(\mathbf{x}) > c_t \sigma_{\text{fb}[t]}(\mathbf{x}) \},
\]
in which $\Delta(\mathbf{x}) = f(\mathbf{x}^*) - f(\mathbf{x})$.
% and $\mathbf{x}^* = \arg\max_{\mathbf{x}\in \mathcal{X}}f(\mathbf{x})$.
\end{definition}
 An important consequence of the definition above is that $\mathbf{x}^*$ is always unsaturated, because $\Delta(\mathbf{x}^*) = 0 < c_t\sigma_{\text{fb}[t]}(\mathbf{x}^*)$.
\begin{lemma}
\label{lemma:uniform_lower_bound}
For any $\mathcal{F}_{t-1}$, conditioned on the events $E^f(t)$, we have that $\forall \mathbf{x}\in \mathcal{X}$,
\begin{equation}
\mathbb{P}\left(f_t(\mathbf{x}) > f(\mathbf{x}) | \mathcal{F}_{t-1}\right) \geq p,
\label{eq:tmp_1}
\end{equation}
in which $p=\frac{1}{4e\sqrt{\pi}}$. 
\end{lemma}
\begin{proof}
\begin{equation}
\begin{split}
\mathbb{P}\left(f_t(\mathbf{x}) > f(\mathbf{x}) | \mathcal{F}_{t-1}\right) &= 
\mathbb{P}\left(\frac{f_t(\mathbf{x})-\mu_{\text{fb}[t]}(\mathbf{x})}{\beta_t\sigma_{\text{fb}[t]}(\mathbf{x})} > \frac{f(\mathbf{x})-\mu_{\text{fb}[t]}(\mathbf{x})}{\beta_t\sigma_{\text{fb}[t]}(\mathbf{x})} \Big| \mathcal{F}_{t-1}\right)\\
&\geq \mathbb{P}\left(\frac{f_t(\mathbf{x})-\mu_{\text{fb}[t]}(\mathbf{x})}{\beta_t\sigma_{\text{fb}[t]}(\mathbf{x})} > \frac{|f(\mathbf{x})-\mu_{\text{fb}[t]}(\mathbf{x})|}{\beta_t\sigma_{\text{fb}[t]}(\mathbf{x})} \Big| \mathcal{F}_{t-1}\right)\\
&\stackrel{(a)}{\geq} \mathbb{P}\left(\frac{f_t(\mathbf{x})-\mu_{\text{fb}[t]}(\mathbf{x})}{\beta_t\sigma_{\text{fb}[t]}(\mathbf{x})} > 1 \Big| \mathcal{F}_{t-1}\right)\\
&\stackrel{(b)}{\geq} \frac{e^{-1}}{4\sqrt{\pi}}.
\end{split}
\end{equation}
$(a)$ follows from Lemma~\ref{lemma:event:Ef}, which holds because we condition on the event $E^{f}(t)$ here. $(b)$ follows since $f_t(\mathbf{x}) \sim \mathcal{N}(\mu_{\text{fb}[t]}(\mathbf{x}), \beta_t^2\sigma^2_{\text{fb}[t]}(\mathbf{x}))$ and makes use of the Gaussian anti-concentration inequality, i.e., $\mathbb{P}(Z > a) \geq \frac{e^{-a^2}}{4\sqrt{\pi}a}$ where $Z$ follows a standard Gaussian distribution.
\end{proof}

The next Lemma shows that the probability that the selected input is unsaturated (i.e., desirable according to Definition~\ref{def:saturated_set}) can be lower-bounded.
\begin{lemma}
\label{lemma:prob_unsaturated}
For any $\mathcal{F}_{t-1}$, conditioned on the event $E^f(t)$, we have that 
% with probability $\geq 1 - \delta/2$,
\[
\mathbb{P}\left(\mathbf{x}_t \in \mathcal{X}\setminus S_t, | \mathcal{F}_{t-1} \right) \geq p - 1/t^2.
\]
\end{lemma}
\begin{proof}
To begin with, we can lower-bounded the probability that the selected $\mathbf{x}_t$ is unsaturated as follows:
\begin{equation}
\mathbb{P}\left(\mathbf{x}_t \in \mathcal{X}\setminus S_t | \mathcal{F}_{t-1} \right) \geq \mathbb{P}\left( f_t(\mathbf{x}^*) > f_t(\mathbf{x}),\forall \mathbf{x} \in S_t | \mathcal{F}_{t-1} \right).
\label{eq:lower_bound_prob_unsaturated}
\end{equation}
The inequality above holds because the event on the right hand sight implies the event on the left hand side. Specifically, because $\mathbf{x}^*$ is always unsaturated (Definition ~\ref{def:saturated_set}), therefore, as long as $f_t(\mathbf{x}^*) > f_t(\mathbf{x}),\forall \mathbf{x} \in S_t$, then the selected $\mathbf{x}_t$ is guaranteed to be unsaturated because it is selected as $\mathbf{x}_t={\arg\max}_{\mathbf{x}\in\mathcal{X}}f_t(\mathbf{x})$.

Next, we assume that both events $E^f(t)$ and $E^{f_t}(t)$ holds, which allows us to derive an upper bound on $f_t(\mathbf{x})$ for all $\mathbf{x}\in S_t$:
\begin{equation}
\begin{split}
    f_t(\mathbf{x}) \stackrel{(a)}{\leq} f(\mathbf{x}) + c_t\sigma_{\text{fb}[t]}(\mathbf{x}) \stackrel{(b)}{\leq} f(\mathbf{x}) + \Delta(\mathbf{x})=f(\mathbf{x}) + f(\mathbf{x}^*) - f(\mathbf{x}) = f(\mathbf{x}^*),
\end{split}
\label{eq:bound_ftx_ftstar}
\end{equation}
in which $(a)$ results from Lemma~\ref{lemma:event:Ef} and Lemma~\ref{lemma:event:Eft} and $(b)$ follows from Definition~\ref{def:saturated_set}.
As a result,equation~\eqref{eq:bound_ftx_ftstar} implies that when both both events $E^f(t)$ and $E^{f_t}(t)$ hold, we have that
\begin{equation}
    \mathbb{P}\left( f_t(\mathbf{x}^*) > f_t(\mathbf{x}),\forall \mathbf{x} \in S_t | \mathcal{F}_{t-1} \right) \geq \mathbb{P}\left( f_t(\mathbf{x}^*) > f(\mathbf{x}^*) | \mathcal{F}_{t-1} \right).
\label{eq:app:proof:lemma:4:lower:bound:by:x:star}
\end{equation}
Next, combining equations~\eqref{eq:lower_bound_prob_unsaturated} and~\eqref{eq:app:proof:lemma:4:lower:bound:by:x:star} and separately considering the cases where the event $E^{f_t}(t)$ is true or false, we have that 
\begin{equation}
\begin{split}
    \mathbb{P}\left(\mathbf{x}_t \in \mathcal{X}\setminus S_t | \mathcal{F}_{t-1} \right) &\geq 
    \mathbb{P}\left( f_t(\mathbf{x}^*) > f_t(\mathbf{x}),\forall \mathbf{x} \in S_t | \mathcal{F}_{t-1} \right)\\
    &\stackrel{(a)}{\geq} \mathbb{P}\left( f_t(\mathbf{x}^*) > f(\mathbf{x}^*) | \mathcal{F}_{t-1} \right) - \mathbb{P}\left(\overline{E^{f_t}(t)} | \mathcal{F}_{t-1}\right)\\
    &\stackrel{(b)}{\geq} p - 1 / t^2.
\end{split}
\label{eq:unsaturated_prob_plugin_1}
\end{equation}
This completes the proof.
\end{proof}

We use $\sigma_{t-1}(\cdot)$ to represent the GP posterior standard deviation conditioned on all selected input queries from iterations $1$ to $t-1$.
\begin{lemma}
\label{lemma:ratio:sigma:C}
We have for all $t\geq 1$ and all $\mathbf{x}\in\mathcal{X}$ that
\end{lemma}
\[
\frac{\sigma_{\text{fb}[t]}(\mathbf{x})}{\sigma_{t-1}(\mathbf{x})} \leq e^C.
\]
\begin{proof}
We use $\mathbf{y}_{1:\text{fb}[t]}$ to denote the output observations from iterations $1$ to $\text{fb}[t]$, use $\mathbf{y}_{\text{fb}[t]+1:t-1}$ to represent the the output observations from iterations $\text{fb}[t]+1$ to $t-1$, and use $\mathbf{y}_{A}$ to denote the vector of observations at a set of inputs $A\subset\mathcal{X}$. We use $H(\cdot)$ to represent the entroy of a random variable.

To begin with, we establish the relationship between the following conditional information gain and the ratio of GP posterior standard deviations which we intend to upper-bound:
\begin{equation}
\begin{split}
\mathbb{I}(f(\mathbf{x});\mathbf{y}_{\text{fb}[t]+1:t-1}|\mathbf{y}_{1:\text{fb}[t]})) &= H(f(\mathbf{x}) | \mathbf{y}_{1:\text{fb}[t]})) - H(f(\mathbf{x}) | \mathbf{y}_{1:t-1}))\\
&=\frac{1}{2}\log(2\pi e \sigma^2_{\text{fb}[t]}(\mathbf{x})) - \frac{1}{2}\log(2\pi e \sigma^2_{t-1}(\mathbf{x})) \\
&= \log\frac{\sigma_{\text{fb}[t]}(\mathbf{x})}{\sigma_{t-1}(\mathbf{x})}.
\end{split}
\end{equation}
The first equality comes from the definition of conditional information gain and the second equality follows immediately from the entroy of Gaussian random variables. The equation above allows us to upper-bound the ratio of GP posterior standard deviations as follows:
\begin{equation}
\begin{split}
\frac{\sigma_{\text{fb}[t]}(\mathbf{x})}{\sigma_{t-1}(\mathbf{x})} &= \exp(\mathbb{I}(f(\mathbf{x});\mathbf{y}_{\text{fb}[t]+1:t-1}|\mathbf{y}_{1:\text{fb}[t]})))\\
&\stackrel{(a)}{\leq} \exp(\mathbb{I}(f;\mathbf{y}_{\text{fb}[t]+1:t-1}|\mathbf{y}_{1:\text{fb}[t]})))\\
&\stackrel{(b)}{\leq} \exp\big(\max_{A\subset\mathcal{X},|A|\leq B-1}\mathbb{I}(f;\mathbf{y}_{A} | \mathbf{y}_{1:\text{fb}[t]})\big) \stackrel{(c)}{\leq} e^C,
\end{split}
\end{equation}
in which $(a)$ is because the information gain about $f$ is larger than that of $f(\mathbf{x})$, $(b)$ follow since the size of $\mathbf{y}_{\text{fb}[t]+1:t-1}$ is at most $B-1$ in our batch setting with a batch size of $B$, and $(c)$ is a result of the definition of the constant $C$. This completes the proof.
\end{proof}

Next, we are ready to prove an upper bound on the expected instantaneous regret $r_t=f(\mathbf{x}^*)-f(\mathbf{x}_t)$.
\begin{lemma}
\label{lemma:upper_bound_expected_regret}
For any $\mathcal{F}_{t-1}$, conditioned on the event $E^{f}(t)$, we have that
\begin{equation*}
\mathbb{E}\left[r_t | \mathcal{F}_{t-1}\right] \leq c_t e^C \left(1+\frac{10}{p}\right)\mathbb{E}\left[\sigma_{t-1}(\mathbf{x}_t) | \mathcal{F}_{t-1}\right] + \frac{2B'}{t^2}.
\end{equation*}
\end{lemma}
\begin{proof}
To begin with, define
\begin{equation}
\overline{\mathbf{x}}_t\triangleq {\arg\min}_{\mathbf{x}\in\mathcal{X}\setminus S_t}\sigma_{\text{fb}[t]}(\mathbf{x}).
\end{equation}
That is, $\overline{\mathbf{x}}_t$ is the unsaturated input with the smallest GP posterior standard deviation.
Note that given a $\mathcal{F}_{t-1}$, $\overline{\mathbf{x}}_t$ is deterministic. Next, we have that
\begin{equation}
\begin{split}
\mathbb{E}[\sigma_{\text{fb}[t]}(\mathbf{x}_t) | \mathcal{F}_{t-1}] &\geq \mathbb{E}\left[\sigma_{\text{fb}[t]}(\mathbf{x}_t) | \mathcal{F}_{t-1}, \mathbf{x}_t\in\mathcal{X}\setminus S_t\right] \mathbb{P}\left(\mathbf{x}_t\in\mathcal{X}\setminus S_t | \mathcal{F}_{t-1}\right)\\
&\geq \sigma_{\text{fb}[t]}(\overline{\mathbf{x}}_t)(p-1/t^2),
\end{split}
\label{eq:app:sigma:bar}
\end{equation}
where the second inequality makes use of Lemma~\ref{lemma:prob_unsaturated}, which holds here because we have also conditioned on the event $E^{f}(t)$ in Lemma~\ref{lemma:prob_unsaturated}. Next, conditioned on both $E^{f}(t)$ and $E^{f_t}(t)$, we have that
\begin{equation}
\begin{split}
r_t &= f(\mathbf{x}^*) - f(\mathbf{x}_t) = f(\mathbf{x}^*) - f(\overline{\mathbf{x}}_t) + f(\overline{\mathbf{x}}_t) - f(\mathbf{x}_t)\\
&\stackrel{(a)}{\leq} \Delta(\overline{\mathbf{x}}_t) + f_t(\overline{\mathbf{x}}_t) + c_t \sigma_{\text{fb}[t]}(\overline{\mathbf{x}}_t) - f_t(\mathbf{x}_t) + c_t \sigma_{\text{fb}[t]}(\mathbf{x}_t)\\
&\stackrel{(b)}{\leq} c_t \sigma_{\text{fb}[t]}(\overline{\mathbf{x}}_t) + c_t \sigma_{\text{fb}[t]}(\overline{\mathbf{x}}_t) + c_t \sigma_{\text{fb}[t]}(\mathbf{x}_t) + f_t(\overline{\mathbf{x}}_t) - f_t(\mathbf{x}_t)\\
&\stackrel{(c)}{\leq} c_t \left( 2\sigma_{\text{fb}[t]}(\overline{\mathbf{x}}_t) + \sigma_{\text{fb}[t]}(\mathbf{x}_t) \right),
\end{split}
\end{equation}
in which $(a)$ makes use of Lemma~\ref{lemma:event:Ef} and Lemma~\ref{lemma:event:Eft}, $(b)$ follows since $\overline{\mathbf{x}}_t$ is unsaturated, and $(c)$ follows from the way in which $\mathbf{x}_t$ is selected: $\mathbf{x}_t={\arg\max}_{\mathbf{x}\in\mathcal{X}}f_t(\mathbf{x})$.
Next, the expected value of $r_t$ can be upper-bounded as follows:
% Inspired by Proposition 1 of~\cite{desautels2014parallelizing}, choose a constant $C$ such that $\sigma_{\text{fb}[t]}(\mathbf{x}) / \sigma_{t-1}(\mathbf{x})=\exp\left(\mathbb{I}\left( f(\mathbf{x});\mathbf{y}_{\text{fb}[t]+1:t-1} | \mathbf{y}_{1:\text{fb}[t]} \right)\right) \leq \exp(C)$. A naive choice for $C$ is simply $\gamma_{B-1}$, but we will use a special initialization scheme, based on uncertainty sampling, to reduce the value of $C$ to be a constant that is independent of the batch size $B$.
\begin{equation}
\begin{split}
\mathbb{E}\left[ r_t | \mathcal{F}_{t-1} \right] &\leq \mathbb{E}\left[c_t \left( 2\sigma_{\text{fb}[t]}(\overline{\mathbf{x}}_t) + \sigma_{\text{fb}[t]}(\mathbf{x}_t) \right) | \mathcal{F}_{t-1}\right] + 2B'\mathbb{P}\left(\overline{E^{f_t}(t)} | \mathcal{F}_{t-1}\right)\\
&\stackrel{(a)}{\leq} \mathbb{E}\left[c_t \left(\frac{2}{p-1/t^2} \sigma_{\text{fb}[t]}(\mathbf{x}_t) + \sigma_{\text{fb}[t]}(\mathbf{x}_t) \right) | \mathcal{F}_{t-1}\right] + \frac{2B'}{t^2}\\
&= \mathbb{E}\left[c_t \left(1+\frac{2}{p-1/t^2}\right)\sigma_{\text{fb}[t]}(\mathbf{x}_t) | \mathcal{F}_{t-1}\right] + \frac{2B'}{t^2}\\
&\stackrel{(b)}{\leq} c_t \left(1+\frac{2}{p-1/t^2}\right)\mathbb{E}\left[e^C\sigma_{t-1}(\mathbf{x}_t) | \mathcal{F}_{t-1}\right] + \frac{2B'}{t^2}\\
&\stackrel{(c)}{\leq} c_t e^C \left(1+\frac{10}{p}\right)\mathbb{E}\left[\sigma_{t-1}(\mathbf{x}_t) | \mathcal{F}_{t-1}\right] + \frac{2B'}{t^2},
\end{split}
\end{equation}
where $(a)$ follows from equation~\eqref{eq:app:sigma:bar} and $(b)$ makes use of Lemma~\ref{lemma:ratio:sigma:C}. 
$(c)$ follows since $2 / (p-1/t^2) \leq 10/p$, which holds because (i) $p-1/t^2<0$ for $t<5$, (ii) $2 / (p-1/t^2) \leq 10/p$ for $t=5$, and (iii) $2 / (p-1/t^2)$ is decreasing as $t$ increases when $t\geq5$.

\end{proof}

\begin{definition}
Define $Y_0=0$, and for all $t=1,\ldots,T$,
\[
\overline{r}_t=r_t \mathbb{I}\{E^{f}(t)\},
\]
\[
X_t = \overline{r}_t - c_t e^C \left(1+\frac{10}{p}\right) \sigma_{t-1}(\mathbf{x}_t) - \frac{2B'}{t^2}
\]
\[
Y_t=\sum^t_{s=1}X_s.
\]
\end{definition}

\begin{lemma}
\label{lemma:proof:sup:martingale}
Conditioned on the event $E^{f}(t)$, $(Y_t:t=0,\ldots,T)$ is a super-martingale with respect to the filtration $\mathcal{F}_t$.
\end{lemma}
\begin{proof}

\begin{equation}
\begin{split}
\mathbb{E}[Y_t - Y_{t-1} | \mathcal{F}_{t-1}] &= \mathbb{E}[X_t | \mathcal{F}_{t-1}]\\
&= \mathbb{E}[\overline{r}_t - c_t e^C \left(1+\frac{10}{p}\right) \sigma_{t-1}(\mathbf{x}_t) - \frac{2B'}{t^2} | \mathcal{F}_{t-1}]\\
&= \mathbb{E}[\overline{r}_t | \mathcal{F}_{t-1} ] - \left(c_t e^C \left(1+\frac{10}{p}\right) \mathbb{E}[\sigma_{t-1}(\mathbf{x}_t)| \mathcal{F}_{t-1}] + \frac{2B'}{t^2}\right) \leq 0.
\end{split}
\end{equation}
If the event $E^{f}(t)$ holds, then $\overline{r}_t= r_t$ and the inequality follows from Lemma~\ref{lemma:upper_bound_expected_regret}. If $E^{f}(t)$ does not hold, $\overline{r}_t=0$ and the inequality holds trivially.
\end{proof}

Lastly, we can apply the Azuma-Hoeffding's inequality to the martingale $(Y_t:t=0,\ldots,T)$ to derive the upper bound on the cumulative regret $R_T$.
\begin{lemma}
\label{lemma:bound:cum:regret}
Define $C_1\triangleq \frac{2}{\log(1+\sigma^{-2})}$.
With probability of $\geq 1-\delta$, we have that
\begin{equation}
R_T \leq c_T e^C \left(1+\frac{10}{p}\right)\sqrt{C_1 T \gamma_T} + \frac{B'\pi^2}{3} + \left[4B' + c_T e^C \left(1+\frac{10}{p}\right)K_0\right]\sqrt{2T\log(2/\delta)}.
\end{equation}
\end{lemma}
\begin{proof}
To begin with, note that
\begin{equation}
\begin{split}
|Y_t - Y_{t-1}| &= |X_t| \leq |\overline{r}_t| + c_t e^C \left(1+\frac{10}{p}\right) \sigma_{t-1}(\mathbf{x}_t) + \frac{2B'}{t^2}\\
&\leq 2B' +  c_t e^C \left(1 + \frac{10}{p}\right) K_0 + 2B'\\
&= 4B' +  c_t e^C \left(1 + \frac{10}{p}\right)K_0,
\end{split}
\end{equation}
where we have made use of our assumption that $\Theta(\mathbf{x},\mathbf{x}')\leq K_0$ in the second inequality.
% Define $C_1\triangleq \frac{2}{\log(1+\sigma^{-2})}$.
Next, applying the Azuma-Hoeffding's inequality to $(Y_t:t=0,\ldots,T)$ with a probability of $\delta/2$, 
we have with probability $\geq 1-\delta/2$ that
\begin{equation}
\begin{split}
\sum^T_{t=1}\overline{r}_t &\leq \sum^T_{t=1} c_t e^C \Big(1+\frac{10}{p}\Big)\sigma_{t-1}(\mathbf{x}_t) + \sum^T_{t=1}\frac{2B'}{t^2} + \sqrt{2\log(2/\delta) \sum^T_{t=1}\Big( 4B' +  c_t e^C \big(1 + \frac{10}{p}\big)K_0 \Big)^2 }\\
&\stackrel{(a)}{\leq} c_T e^C \Big(1+\frac{10}{p}\Big)\sum^T_{t=1}\sigma_{t-1}(\mathbf{x}_t) + \frac{B'\pi^2}{3} + \left[4B' + c_T e^C \left(1+\frac{10}{p}\right)K_0\right]\sqrt{2T\log(2/\delta)}\\
&\stackrel{(b)}{\leq} c_T e^C \Big(1+\frac{10}{p}\Big)\sqrt{C_1 T \gamma_T} + \frac{B'\pi^2}{3} + \left[4B' + c_T e^C \left(1+\frac{10}{p}\right)K_0\right]\sqrt{2T\log(2/\delta)}.
\end{split}
\end{equation}
$(a)$ follows since $c_t$ is increasing in $t$, and $(b)$ follows from the proof of Lemma 5.4 in the work of~\cite{srinivas2009gaussian}. Next, note that $\overline{r}_t=r_t,\forall t\geq1$ with probability of $\geq1-\delta/2$ according to Lemma~\ref{lemma:event:Ef}. Therefore, the upper bound derived in the equation above is an upper bound on $R_T=\sum^T_{t=1}r_t$ (with probability of $\geq1-\delta$), and the proof is completed.
% \begin{equation}
% \begin{split}
% \sum^T_{t=1}\overline{r}_t &\leq c_T e^C \left(1+\frac{10}{p}\right)\sum^T_{t=1}\sigma_{t-1}(\mathbf{x}_t) + \frac{B\pi^2}{3} + \left[4B + c_T e^C \left(1+\frac{10}{p}\right)\right]\sqrt{2T\log(1/\delta')}\\
% &\leq c_T e^C \left(1+\frac{10}{p}\right)\sqrt{C_1 T \gamma_T} + \frac{B\pi^2}{3} + \left[4B + c_T e^C \left(1+\frac{10}{p}\right)\right]\sqrt{2T\log(1/\delta')}
% \end{split}
% \end{equation}

\end{proof}

 Note that $c_T = \mathcal{O}(\log^2 T)$.
From Lemma~\ref{lemma:bound:cum:regret}, we have that
\begin{equation}
R_T = \mathcal{O}\left(e^C (\log^2T) \sqrt{T} \left(1+\sqrt{\gamma_T}\right)\right) = \widetilde{\mathcal{O}}\left(e^C \sqrt{T} \left(1+\sqrt{\gamma_T}\right)\right).
\end{equation}

\section{Proof of Theorem~\ref{theorem:regret:approc:ntk}}
\label{app:sec:finite:width}
In this section, the main technical challenge is to rigorously account for the mismatch between the kernel with which we assume the objective function $f$ is sampled (i.e., the exact NTK $\Theta$) and the kernel with which the acquisition function is sampled (i.e., the empirical NTK $\widetilde{\Theta}$).
For ease of exposition, we use $k$ and $\widetilde{k}$ (instead of $\Theta$ and $\widetilde{\Theta}$) to represent the exact and empirical NTK in the proof in this section. Similarly, we also use $\,\widetilde{}\,$ to indicate that a term is associated with the empirical NTK $\widetilde{k}$. For example, we use $\widetilde{\mu}_{\text{fb}[t]}(\cdot)$ and $\widetilde{\sigma}^2_{\text{fb}[t]}(\cdot)$ to represent the GP posterior mean and variance calculated using the empirical NTK $\widetilde{k}$.

For simplicity, we assume that the event in Proposition~\ref{prop:arora} holds throughout the entire proof, which happens with probability of $\geq 1-\delta/4$. That is, the approximation error between exact and empirical NTKs is bounded:
\begin{equation}
|\widetilde{k}(\mathbf{x},\mathbf{x}') - k(\mathbf{x},\mathbf{x}')|=
\left| \langle \nabla_{\theta}f(\mathbf{x},\widetilde{\theta}), \nabla_{\theta}f(\mathbf{x}',\widetilde{\theta}) \rangle - \Theta(\mathbf{x}, \mathbf{x}') \right| \leq (L+1) \varepsilon, \qquad \forall \mathbf{x},\mathbf{x}'\in\mathcal{X}.
\label{eq:app:finite:width:appro:guarantee}
\end{equation}
% Note that throughout all proofs in this section, we assume that the event in Proposition~\ref{prop:arora} holds, and we will incorporate the associated error probability $\delta_{\text{ntk}}$ at the end of our proof.

To begin with, we use the following lemma to bound the difference between the GP posterior standard deviations calculated using the exact and empirical NTKs.
Here, to simplify the derivations and results, we assume that $(L+1)\varepsilon \leq 1$ and $\sigma^2\leq 1$. Note that these assumptions are not essential to the proof but are only used get cleaner expressions.
Here, again for ease of expositions, we define $\hat{K}_0\triangleq\max\{1,K_0\}$, and $\hat{K}^2_0\triangleq\max\{1,K^2_0\}$.
\begin{lemma}
\label{lemma:bound:difference:gp:post:std}
We have $\forall t\geq1,\forall \mathbf{x}\in\mathcal{X}$ that
\[
|\sigma_{\text{fb}[t]}(\mathbf{x}) - \widetilde{\sigma}_{\text{fb}[t]}(\mathbf{x})| \leq \sqrt{(L+1)\varepsilon \left( 1+\frac{4 \hat{K}_0^2 t^2}{\sigma^4} \right)}.
\]
% \[
% |\sigma_{\text{fb}[t]}(\mathbf{x}) - \widetilde{\sigma}_{\text{fb}[t]}(\mathbf{x})| \leq 2t\sqrt{\frac{(L+1)\varepsilon}{\sigma^2}}
% \]
\end{lemma}
\begin{proof}
Denote by $\boldsymbol{K}_t$ the $\text{fb}[t]\times \text{fb}[t]$-dimensional gram matrix of exact NTK covariance values calculated using all $\text{fb}[t]$ observations up to iteration $\text{fb}[t]$, and use $\widetilde{\boldsymbol{K}}_t$ to represent the corresponding gram matrix calculated using the empirical NTK $\widetilde{k}$.
If we define $A=(\boldsymbol{K}_t + \sigma^2)^{-1}$ or $A=(\widetilde{\boldsymbol{K}}_t + \sigma^2)^{-1}$, then for both values of $A$, we have that
\begin{equation}
\norm{A}_2 = \sqrt{\max[\text{eig}(A^{\top}A)]}=\sqrt{\max[\text{eig}(A)^2]} \leq \frac{1}{\sigma^2}.
\end{equation}
% The equation above still holds if we replace the matrix $A$ by $A=(\widetilde{\boldsymbol{K}}_t + \sigma^2)^{-1}$.
This allows us to derive the following equation:
\begin{equation}
\begin{split}
\norm{(\boldsymbol{K}_t + \sigma^2)^{-1} - (\widetilde{\boldsymbol{K}}_t + \sigma^2)^{-1}}_2 &\leq \norm{(\boldsymbol{K}_t + \sigma^2)^{-1}}_2 \norm{(\widetilde{\boldsymbol{K}}_t + \sigma^2)^{-1}}_2 \norm{\boldsymbol{K}_t - \widetilde{\boldsymbol{K}}_t}_2\\
&\leq \frac{1}{\sigma^2} \times \frac{1}{\sigma^2} \times t(L+1)\varepsilon = \frac{t(L+1)\varepsilon}{\sigma^4}.
\end{split}
\end{equation}
Define the $\text{fb}[t]$-dimensional vectors $\boldsymbol{k}_t(\mathbf{x})=[k(\mathbf{x},\mathbf{x}_\tau)]_{\tau=1,\ldots,\text{fb}[t]}$ and $\widetilde{\boldsymbol{k}}_t(\mathbf{x})=[\widetilde{k}(\mathbf{x},\mathbf{x}_\tau)]_{\tau=1,\ldots,\text{fb}[t]}$.
Then making use of the approximation guarantee from equation~\eqref{eq:app:finite:width:appro:guarantee}, we define $\widetilde{\boldsymbol{k}}_t(\mathbf{x}) = \boldsymbol{k}_t(\mathbf{x}) + (L+1)\varepsilon \boldsymbol{\nu}(\mathbf{x})$, where $\boldsymbol{\nu}(\mathbf{x})$ is an $\text{fb}[t]-$dimensional vector where every element satisfies $|\boldsymbol{\nu}(\mathbf{x})_i| \leq 1, \forall i\in[\text{fb}[t]]$.
Now we can use these definitions to derive the following upper bound.
\begin{equation}
\begin{split}
|&\sigma^2_{\text{fb}[t]}(\mathbf{x}) - \widetilde{\sigma}_{\text{fb}[t]}^2(\mathbf{x})| = | k(\mathbf{x},\mathbf{x}) - \boldsymbol{k}_t(\mathbf{x})^{\top} (\boldsymbol{K}_t+\sigma^2I)^{-1}\boldsymbol{k}_t(\mathbf{x}) \\
&\qquad\qquad - \widetilde{k}(\mathbf{x},\mathbf{x}) + \widetilde{\boldsymbol{k}}_t(\mathbf{x})^{\top} (\widetilde{\boldsymbol{K}}_t+\sigma^2I)^{-1}\widetilde{\boldsymbol{k}}_t(\mathbf{x}) |\\
&\leq |k(\mathbf{x},\mathbf{x})-\widetilde{k}(\mathbf{x},\mathbf{x})| + | \boldsymbol{k}_t(\mathbf{x})^{\top} (\boldsymbol{K}_t+\sigma^2I)^{-1}\boldsymbol{k}_t(\mathbf{x}) - \widetilde{\boldsymbol{k}}_t(\mathbf{x})^{\top} (\widetilde{\boldsymbol{K}}_t+\sigma^2I)^{-1}\widetilde{\boldsymbol{k}}_t(\mathbf{x})|\\
&\leq (L+1)\varepsilon + \Big|\boldsymbol{k}_t(\mathbf{x})^{\top} \left((\boldsymbol{K}_t+\sigma^2I)^{-1} - (\widetilde{\boldsymbol{K}}_t+\sigma^2I)^{-1}\right)\boldsymbol{k}_t(\mathbf{x}) \\
&- 2(L+1)\varepsilon \boldsymbol{\nu}(\mathbf{x})^{\top}(\widetilde{\boldsymbol{K}}_t+\sigma^2I)^{-1}\boldsymbol{k}_t(\mathbf{x}) - (L+1)^2\varepsilon^2\boldsymbol{\nu}(\mathbf{x})^{\top}(\widetilde{\boldsymbol{K}}_t+\sigma^2I)^{-1}\boldsymbol{\nu}(\mathbf{x})\Big| \\
&\leq (L+1)\varepsilon + \norm{\boldsymbol{k}_t(\mathbf{x})}_2 \norm{(\boldsymbol{K}_t + \sigma^2)^{-1} - (\widetilde{\boldsymbol{K}}_t + \sigma^2)^{-1}}_2 \norm{\boldsymbol{k}_t(\mathbf{x})}_2 + \\
&2(L+1)\varepsilon \norm{\boldsymbol{\nu}(\mathbf{x})}_2 \norm{(\widetilde{\boldsymbol{K}}_t + \sigma^2)^{-1}}_2 \norm{\boldsymbol{k}_t(\mathbf{x})}_2 + (L+1)^2\varepsilon^2\norm{\boldsymbol{\nu}(\mathbf{x})}_2\norm{(\widetilde{\boldsymbol{K}}_t+\sigma^2I)^{-1}}_2\norm{\boldsymbol{\nu}(\mathbf{x}}_2)\\
&\leq (L+1)\varepsilon + K_0\sqrt{t} \frac{t(L+1)\varepsilon}{\sigma^4} K_0 \sqrt{t} + 2(L+1)\varepsilon \sqrt{t} \frac{1}{\sigma^2} K_0\sqrt{t} + (L+1)^2\varepsilon^2 \sqrt{t} \frac{1}{\sigma^2} \sqrt{t}\\
&= (L+1)\varepsilon + K_0^2\frac{t^2(L+1)\varepsilon}{\sigma^4} + 2K_0(L+1)\varepsilon\frac{t}{\sigma^2} + (L+1)^2\varepsilon^2\frac{t}{\sigma^2}\\
&\leq (L+1)\varepsilon + 4\hat{K}_0^2\frac{t^2(L+1)\varepsilon}{\sigma^4} \\
&\leq (L+1)\varepsilon \left( 1 + \frac{4 \hat{K}_0^2 t^2}{\sigma^4} \right).
\end{split}
\end{equation}
% To simplify the results, we have assumed that $(L+1)\varepsilon \leq 1$ and $\sigma^2\leq 1$.
Elementary calculation tells us that for $a,b,c>0$, if $a^2-b^2\leq c^2$, then $a\leq \sqrt{b^2+c^2}\leq b+c$, which leads to $a-b\leq c$.
As a result, the equation above tells us that $|\sigma_{\text{fb}[t]}(\mathbf{x})-\widetilde{\sigma}_{\text{fb}[t]}(\mathbf{x})| \leq \sqrt{(L+1)\varepsilon \left( 1+\frac{4 \hat{K}_0^2 t^2}{\sigma^4} \right)}$.
\end{proof}

The next Lemma gives an upper bound on the difference between the GP posterior means calculated using the exact and empirical NTKs.
\begin{lemma}
\label{lemma:bound:difference:gp:post:mean}
With probability of $\geq 1-\delta/4$, we have $\forall t\geq 1,\forall \mathbf{x}\in\mathcal{X}$ that
\[
|\mu_{\text{fb}[t]}(\mathbf{x}) - \widetilde{\mu}_{\text{fb}[t]}(\mathbf{x})| \leq 2\hat{K}_0\frac{t^2(L+1)\varepsilon}{\sigma^4}\left(B' + \sigma\sqrt{2\log(4T/\delta)}\right).
\]
\end{lemma}
\begin{proof}
Define $\mathbf{y}_t=[y_{\tau}]_{\tau=1,\ldots,\text{fb}[t]}$. We have that $y_{\tau} = f(\mathbf{x}_{\tau}) + \epsilon$ where $\epsilon \sim \mathcal{N}(0,\sigma^2)$. Standard Gaussian concentration tells us that $|\epsilon| \leq z\sigma$ with probability of $\geq 1-\exp(-z^2/2)$. Substituting $z=\sqrt{2\log(4T/\delta)}$ and making use of the assumption that $|f(\mathbf{x})|\leq B',\forall \mathbf{x}\in\mathcal{X}$, we have that $|y_{\tau}|\leq B' + \sigma\sqrt{2\log(4T/\delta)}$ with probability of $\geq 1-\delta/(4T)$.
Now taking a union bound over all $T$ iterations, we have that $|y_{\tau}|\leq B' + \sigma\sqrt{2\log(4T/\delta)},\forall \tau=1,\ldots,T$ with probability of $\geq 1-\delta/4$.
This further implies that $\norm{\mathbf{y}_t}_2 =\sqrt{\sum^{\text{fb}[t]}_{\tau=1}y_{\tau}^2} \leq \sqrt{t} \left( B' + \sigma\sqrt{2\log(4T/\delta)} \right)$.
Now we are ready to bound the term in question:
\begin{equation}
\begin{split}
|&\mu_{\text{fb}[t]}(\mathbf{x}) - \widetilde{\mu}_{\text{fb}[t]}(\mathbf{x})| = |\boldsymbol{k}_t(\mathbf{x})^{\top}(K_t+\sigma^2 I)^{-1}\mathbf{y}_t - \widetilde{\boldsymbol{k}}_t(\mathbf{x})^{\top}(\widetilde{K}_t+\sigma^2 I)^{-1}\mathbf{y}_t|\\
&= |\boldsymbol{k}_t(\mathbf{x})^{\top}(K_t+\sigma^2 I)^{-1}\mathbf{y}_t - \boldsymbol{k}_t(\mathbf{x})^{\top}(\widetilde{K}_t+\sigma^2 I)^{-1}\mathbf{y}_t - (L+1)\varepsilon \boldsymbol{\nu}(\mathbf{x})^{\top}(\widetilde{K}_t+\sigma^2 I)^{-1}\mathbf{y}_t|\\
&\leq \norm{\boldsymbol{k}_t(\mathbf{x})}_2 \norm{(K_t+\sigma^2 I)^{-1} - (\widetilde{K}_t+\sigma^2 I)^{-1}} \norm{\mathbf{y}_t}_2 +  (L+1)\varepsilon \norm{\boldsymbol{\nu}(\mathbf{x})}_2\norm{(\widetilde{K}_t+\sigma^2 I)^{-1}}_2\norm{\mathbf{y}_t}_2\\
&\leq K_0\sqrt{t} \frac{t(L+1)\varepsilon}{\sigma^4} \sqrt{t} \left( B' + \sigma\sqrt{2\log(4T/\delta)} \right) + (L+1)\varepsilon \sqrt{t}  \frac{1}{\sigma^2} \sqrt{t} \left( B' + \sigma\sqrt{2\log(4T/\delta)} \right)\\
&\leq 2\hat{K}_0\frac{t^2(L+1)\varepsilon}{\sigma^4}\left(B' + \sigma\sqrt{2\log(4T/\delta)}\right).
\end{split}
\end{equation}
\end{proof}

Next, similar to the proof in Appendix~\ref{app:sec:infinite:width}, here we also need a lemma showing the concentration of the function $f$.
Define $\beta_t = 2\log(2\pi^2t^2|\mathcal{X}|/(3\delta))$, and $c_t = \beta_t (1+\sqrt{2\log(|\mathcal{X}|t^2)})$. Note that the value of $\beta_t$ defined here is slightly different due to the use of different error probabilities (i.e., we have used an error probability of $\delta/2$ in the proof in Appendix~\ref{app:sec:infinite:width} yet $\delta/4$ in this section).
\begin{lemma}
\label{lemma:event:Ef:finite:width}
% Choose $\delta\in(0,1)$. 
% Define $E^{f}(t)$ as the event that 
$|\mu_{\text{fb}[t]}(\mathbf{x}) - f(\mathbf{x})| \leq \beta_t\sigma_{\text{fb}[t]}(\mathbf{x}),\forall \mathbf{x}\in\mathcal{X}$, with probability of $\geq 1-\delta/4,\forall t\geq1$.
% We have that $\mathbb{P}(E^{f}(t)) \geq 1-\delta/4,\forall t\geq1$.
\end{lemma}
The proof of Lemma~\ref{lemma:event:Ef:finite:width} is the same as that of Lemma~\ref{lemma:event:Ef}.
The next Lemma proves the concentration of the objective function $f$ around the GP posterior mean calculated using the empirical NTK $\widetilde{k}$, which consists of an additional error term $\epsilon_{m,t}$ due to the use of the empirical NTK compared with Lemma~\ref{lemma:event:Ef:finite:width} above.
\begin{lemma}
\label{lemma:event:Ef:tilde}
% Choose $\delta\in(0,1)$. 
Define 
\[\epsilon_{m,t} \triangleq 2\hat{K}_0\frac{t^2(L+1)\varepsilon}{\sigma^4}\left(B' + \sigma\sqrt{2\log(4T/\delta)}\right) + \beta_t\sqrt{(L+1)\varepsilon \left( 1+\frac{4 \hat{K}_0^2 t^2}{\sigma^4} \right)}.
\]
Define $E^{\widetilde{f}}(t)$ as the event that $|\widetilde{\mu}_{\text{fb}[t]}(\mathbf{x}) - f(\mathbf{x})| \leq \beta_t\widetilde{\sigma}_{\text{fb}[t]}(\mathbf{x}) + \epsilon_{m,t},\forall \mathbf{x}\in\mathcal{X}$. We have that $\mathbb{P}(E^{\widetilde{f}}(t)) \geq 1-\delta/2,\forall t\geq1$.
\end{lemma}
\begin{proof}
\begin{equation}
\begin{split}
|&\widetilde{\mu}_{\text{fb}[t]}(\mathbf{x}) - f(\mathbf{x})| \leq |\widetilde{\mu}_{\text{fb}[t]}(\mathbf{x}) - \mu_{\text{fb}[t]}(\mathbf{x})| + |\mu_{\text{fb}[t]}(\mathbf{x}) - f(\mathbf{x})|\\
&\stackrel{(a)}{\leq} 2\hat{K}_0\frac{t^2(L+1)\varepsilon}{\sigma^4}\left(B' + \sigma\sqrt{2\log(4T/\delta)}\right) + \beta_t \sigma_{\text{fb}[t]}(\mathbf{x})\\
&\stackrel{(b)}{\leq} 2\hat{K}_0\frac{t^2(L+1)\varepsilon}{\sigma^4}\left(B' + \sigma\sqrt{2\log(4T/\delta)}\right) + \beta_t \left(\widetilde{\sigma}_{\text{fb}[t]}(\mathbf{x}) +  \sqrt{(L+1)\varepsilon \left( 1+\frac{4 \hat{K}_0^2 t^2}{\sigma^4} \right)}\right)\\
&= \beta_t \widetilde{\sigma}_{\text{fb}[t]}(\mathbf{x}) + 2\hat{K}_0\frac{t^2(L+1)\varepsilon}{\sigma^4}\left(B' + \sigma\sqrt{2\log(4T/\delta)}\right) + \beta_t\sqrt{(L+1)\varepsilon \left( 1+\frac{4 \hat{K}_0^2 t^2}{\sigma^4} \right)}\\
&= \beta_t \widetilde{\sigma}_{\text{fb}[t]}(\mathbf{x}) + \epsilon_{m,t}
\end{split}
\end{equation}
$(a)$ follows from Lemma~\ref{lemma:bound:difference:gp:post:mean} and Lemma~\ref{lemma:event:Ef:finite:width} and hence holds with probability of $\geq 1-\delta/4-\delta/4=1-\delta/2$, and $(b)$ results from Lemma~\ref{lemma:bound:difference:gp:post:std}.
\end{proof}

Denote by $\widetilde{f}_t$ the sampled function in iteration $t$ using the empirical NTK, i.e., $\widetilde{f}_t \sim \mathcal{GP}(\widetilde{\mu}_{\text{fb}[t]}(\cdot), \beta_t^2 \widetilde{\sigma}^2_{\text{fb}[t]}(\cdot,\cdot))$.
\begin{lemma}
\label{lemma:event:Eft:finite:width}
Define $E^{\widetilde{f}_t}(t)$ as the event that $|\widetilde{\mu}_{\text{fb}[t]}(\mathbf{x}) - \widetilde{f}_t(\mathbf{x})| \leq \beta_t \sqrt{2\log(|\mathcal{X}|t^2)} \widetilde{\sigma}_{\text{fb}[t]}(\mathbf{x}),\forall \mathbf{x}\in\mathcal{X}$. We have that $\mathbb{P}(E^{\widetilde{f}_t}(t)) \geq 1-1/t^2,\forall t\geq1$.
\end{lemma}
Lemma~\ref{lemma:event:Eft:finite:width} is the counterpart to Lemma~\ref{lemma:event:Eft} in Appendix~\ref{app:sec:infinite:width} and can be proved using the same techniques.
Of note, conditioned on both events $E^{\widetilde{f}}(t)$ and $E^{\widetilde{f}_t}(t)$, we have that
\begin{equation}
\begin{split}
|f(\mathbf{x}) &- \widetilde{f}_t(\mathbf{x})| \leq |f(\mathbf{x}) - \widetilde{\mu}_{\text{fb}[t]}(\mathbf{x})| + |\widetilde{\mu}_{\text{fb}[t]}(\mathbf{x}) - \widetilde{f}_t(\mathbf{x})|\\
&\leq \beta_t\widetilde{\sigma}_{\text{fb}[t]}(\mathbf{x}) + \epsilon_{m,t} + \beta_t \sqrt{2\log(|\mathcal{X}|t^2)} \widetilde{\sigma}_{\text{fb}[t]}(\mathbf{x})\\
&= c_t \widetilde{\sigma}_{\text{fb}[t]}(\mathbf{x}) + \epsilon_{m,t}.
\end{split}
\end{equation}

% \begin{equation}
% \begin{split}
% \sum^T_{t=1}\nu_{t,\varepsilon} &= \sum^T_{t=1} 2\frac{t^2(L+1)\varepsilon}{\sigma^4}\left[B' + \sigma\sqrt{2\log(T/\delta_{\mu})}\right] + \sum^T_{t=1} \beta_t \sqrt{2\log(|\mathcal{X}|t^2)} (L+1) \sqrt{\varepsilon \left( 1+\frac{4t^2}{\sigma^4} \right)}\\
% &\leq 2\frac{T^3(L+1)\varepsilon}{\sigma^4}\left[B' + \sigma\sqrt{2\log(T/\delta_{\mu})}\right] + T\beta_T \sqrt{2\log(|\mathcal{X}|T^2)} (L+1) \sqrt{\varepsilon \left( 1+\frac{4T^2}{\sigma^4} \right)}\\
% &\leq 2\frac{T^3(L+1)\varepsilon}{\sigma^4}\left[B' + \sigma\sqrt{2\log(T/\delta_{\mu})}\right] + T\beta_T \sqrt{2\log(|\mathcal{X}|T^2)} (L+1) \sqrt{\varepsilon \frac{5T^2}{\sigma^4}}\\
% &\leq 2\frac{T^3(L+1)\varepsilon}{\sigma^4}\left[B' + \sigma\sqrt{2\log(T/\delta_{\mu})}\right] + T\beta_T \sqrt{2\log(|\mathcal{X}|T^2)} (L+1) \frac{T}{\sigma^2}\sqrt{5\varepsilon}\\
% &\leq 2\frac{T^3(L+1)\varepsilon}{\sigma^4}\left[B' + \sigma\sqrt{2\log(T/\delta_{\mu})}\right] + T^2\beta_T \sqrt{2\log(|\mathcal{X}|T^2)}  \frac{(L+1)\sqrt{5\varepsilon}}{\sigma^2}
% \end{split}
% \end{equation}
% Note that according to Proposition~\ref{prop:arora}, we have that
% $\varepsilon=\overline{C}L^{3/2}\log^{1/4}(L/\delta_{\text{ntk}})m^{-1/4}$. So, we can choose the width of the NN $m$ to be large enough such that the two terms above can both be upper-bounded by constants.

Next, we similarly define the set of saturated inputs.
\begin{definition}
\label{def:saturated_set:new}
Define the set of saturated inputs in iteration $t$ as
\[
S_t = \{ \mathbf{x} \in \mathcal{X} : \Delta(\mathbf{x}) > c_t \widetilde{\sigma}_{\text{fb}[t]}(\mathbf{x}) + 2\epsilon_{m,t} \},
\]
in which $\Delta(\mathbf{x}) = f(\mathbf{x}^*) - f(\mathbf{x})$.
% and $\mathbf{x}^* = \arg\max_{\mathbf{x}\in \mathcal{X}}f(\mathbf{x})$.
\end{definition}
Again, $\mathbf{x}^*$ is always unsaturated.

\begin{lemma}
\label{lemma:uniform_lower_bound:new}
For any $\mathcal{F}_{t-1}$, conditioned on the events $E^{\widetilde{f}}(t)$, we have that $\forall \mathbf{x}\in \mathcal{X}$,
\begin{equation}
\mathbb{P}\left(\widetilde{f}_t(\mathbf{x}) + \epsilon_{m,t} > f(\mathbf{x}) | \mathcal{F}_{t-1}\right) \geq p,
\end{equation}
in which $p=\frac{1}{4e\sqrt{\pi}}$.
\end{lemma}
\begin{proof}

\begin{equation}
\begin{split}
\mathbb{P}\left(\widetilde{f}_t(\mathbf{x}) + \epsilon_{m.t} > f(\mathbf{x}) | \mathcal{F}_{t-1}\right) &= 
\mathbb{P}\left(\frac{\widetilde{f}_t(\mathbf{x})-\widetilde{\mu}_{\text{fb}[t]}(\mathbf{x})  + \epsilon_{m.t}}{\beta_t\widetilde{\sigma}_{\text{fb}[t]}(\mathbf{x})} > \frac{f(\mathbf{x})-\widetilde{\mu}_{\text{fb}[t]}(\mathbf{x})}{\beta_t\widetilde{\sigma}_{\text{fb}[t]}(\mathbf{x})} \Big| \mathcal{F}_{t-1}\right)\\
&\geq \mathbb{P}\left(\frac{\widetilde{f}_t(\mathbf{x})-\widetilde{\mu}_{\text{fb}[t]}(\mathbf{x})  + \epsilon_{m.t}}{\beta_t\widetilde{\sigma}_{\text{fb}[t]}(\mathbf{x})} > \frac{|f(\mathbf{x})-\widetilde{\mu}_{\text{fb}[t]}(\mathbf{x})|}{\beta_t\widetilde{\sigma}_{\text{fb}[t]}(\mathbf{x})} \Big| \mathcal{F}_{t-1}\right)\\
&= \mathbb{P}\left(\frac{\widetilde{f}_t(\mathbf{x})-\widetilde{\mu}_{\text{fb}[t]}(\mathbf{x})}{\beta_t\widetilde{\sigma}_{\text{fb}[t]}(\mathbf{x})} > \frac{|f(\mathbf{x})-\widetilde{\mu}_{\text{fb}[t]}(\mathbf{x})| - \epsilon_{m.t}}{\beta_t\widetilde{\sigma}_{\text{fb}[t]}(\mathbf{x})} \Big| \mathcal{F}_{t-1}\right)\\
&\stackrel{(a)}{\geq} \mathbb{P}\left(\frac{\widetilde{f}_t(\mathbf{x})-\widetilde{\mu}_{\text{fb}[t]}(\mathbf{x})}{\beta_t\widetilde{\sigma}_{\text{fb}[t]}(\mathbf{x})} > 1 \Big| \mathcal{F}_{t-1}\right)\\
&\stackrel{(b)}{\geq} \frac{\exp(-1)}{4\sqrt{\pi}}.
\end{split}
\end{equation}
$(a)$ follows from Lemma~\ref{lemma:event:Ef:tilde}, and $(b)$ follows because $\widetilde{f}_t(\mathbf{x}) \sim \mathcal{N}(\widetilde{\mu}_{\text{fb}[t]}(\mathbf{x}), \beta_t^2\widetilde{\sigma}^2_{\text{fb}[t]}(\mathbf{x}))$ and makes use of the Gaussian anti-concentration inequality.

\end{proof}

Next, we again prove a lower bound on the probability that the selected input is unsaturated.
\begin{lemma}
\label{lemma:prob_unsaturated:new}
For any $\mathcal{F}_{t-1}$, conditioned on the event $E^{\widetilde{f}}(t)$, we have that 
% with probability $\geq 1 - \delta/2$,
\[
\mathbb{P}\left(\mathbf{x}_t \in \mathcal{X}\setminus S_t, | \mathcal{F}_{t-1} \right) \geq p - 1/t^2.
\]
\end{lemma}
\begin{proof}
The proof here follows similar steps as the proof of Lemma~\ref{lemma:prob_unsaturated}. To begin with, we have the following relationship.
\begin{equation}
\mathbb{P}\left(\mathbf{x}_t \in \mathcal{X}\setminus S_t | \mathcal{F}_{t-1} \right) \geq \mathbb{P}\left( \widetilde{f}_t(\mathbf{x}^*) > \widetilde{f}_t(\mathbf{x}),\forall \mathbf{x} \in S_t | \mathcal{F}_{t-1} \right),
% \label{eq:lower_bound_prob_unsaturated}
\end{equation}
The validity of this equation can be justified in a similar way as equation~\eqref{eq:lower_bound_prob_unsaturated} in the proof of Lemma~\ref{lemma:prob_unsaturated}, i.e., the event on the right hand sight implies the event on the left hand side.

Next, we assume that both events $E^{\widetilde{f}}(t)$ and $E^{\widetilde{f}}(t)$ are true, which allows us to derive an upper bound on $\widetilde{f}_t(\mathbf{x})$ for all $\mathbf{x}\in S_t$:
\begin{equation}
\begin{split}
    \widetilde{f}_t(\mathbf{x}) \stackrel{(a)}{\leq} f(\mathbf{x}) + c_t\widetilde{\sigma}_{\text{fb}[t]}(\mathbf{x}) + \epsilon_{m,t} \stackrel{(b)}{\leq} f(\mathbf{x}) + \Delta(\mathbf{x}) - \epsilon_{m,t} = f(\mathbf{x}^*) - \epsilon_{m,t},
\end{split}
\label{eq:bound_ftx_ftstar:new}
\end{equation}
in which $(a)$ follows from Lemmas~\ref{lemma:event:Ef:tilde} and~\ref{lemma:event:Eft:finite:width}, and $(b)$ is a result of Definition~\ref{def:saturated_set:new}.

Therefore,~\eqref{eq:bound_ftx_ftstar:new} implies that when both both events $E^{\widetilde{f}}(t)$ and $E^{\widetilde{f}_t}(t)$ hold,
\begin{equation}
    \mathbb{P}\left( \widetilde{f}_t(\mathbf{x}^*) > \widetilde{f}_t(\mathbf{x}),\forall \mathbf{x} \in S_t | \mathcal{F}_{t-1} \right) \geq \mathbb{P}\left( \widetilde{f}_t(\mathbf{x}^*) > f(\mathbf{x}^*) - \epsilon_{m,t} | \mathcal{F}_{t-1} \right).
\end{equation}

 Next, we can show that 
\begin{equation}
\begin{split}
    \mathbb{P}\left(\mathbf{x}_t \in \mathcal{X}\setminus S_t | \mathcal{F}_{t-1} \right) &\geq 
    \mathbb{P}\left( \widetilde{f}_t(\mathbf{x}^*) > \widetilde{f}_t(\mathbf{x}),\forall \mathbf{x} \in S_t | \mathcal{F}_{t-1} \right)\\
    &\geq \mathbb{P}\left( \widetilde{f}_t(\mathbf{x}^*) > f(\mathbf{x}^*) - \epsilon_{m,t} | \mathcal{F}_{t-1} \right) - \mathbb{P}\left(\overline{E^{\widetilde{f}_t}(t)} | \mathcal{F}_{t-1}\right)\\
    &\geq p - 1 / t^2,
\end{split}
% \label{eq:unsaturated_prob_plugin_1}
\end{equation}
where the last inequality makes use of Lemma~\ref{lemma:uniform_lower_bound:new}.
\end{proof}

The next Lemma derives an upper bound on the expected instantaneous regret $r_t=f(\mathbf{x}^*)-f(\mathbf{x}_t)$.
\begin{lemma}
\label{lemma:upper_bound_expected_regret_new}
For any $\mathcal{F}_{t-1}$, conditioned on the event $E^{\widetilde{f}}(t)$, we have that
\begin{equation*}
\mathbb{E}\left[r_t | \mathcal{F}_{t-1}\right] \leq c_t e^C \left(1+\frac{10}{p}\right)\mathbb{E}\left[\sigma_{t-1}(\mathbf{x}_t) | \mathcal{F}_{t-1}\right] + \epsilon'_{m,t} + \frac{2B'}{t^2},
\end{equation*}
where 
\begin{equation}
\epsilon'_{m,t} \triangleq 3 c_t \sqrt{(L+1)\varepsilon \left( 1+\frac{4 \hat{K}_0^2 t^2}{\sigma^4} \right)} + 4\epsilon_{m,t}.
\end{equation}
% \begin{equation}
% \begin{split}
% \epsilon'_{m,t} &\triangleq c_t (L+1) \sqrt{\varepsilon \left( 1+\frac{4t^2}{\sigma^4} \right)} + \epsilon_{m,t} \\
% &= 2\frac{t^2(L+1)\varepsilon}{\sigma^4}\left[B' + \sigma\sqrt{2\log(T/\delta_{\mu})}\right] + (c_t + \beta_t) \sqrt{(L+1) \varepsilon \left( 1+\frac{4t^2}{\sigma^4} \right)}.
% \end{split}
% \end{equation}
\end{lemma}
\begin{proof}
Define
\begin{equation}
\overline{\mathbf{x}}_t\triangleq {\arg\min}_{\mathbf{x}\in\mathcal{X}\setminus S_t}\sigma_{\text{fb}[t]}(\mathbf{x}).
\end{equation}
Note that given a $\mathcal{F}_{t-1}$, $\overline{\mathbf{x}}_t$ is deterministic. Next, this definiton also leads to:
\begin{equation}
\begin{split}
\mathbb{E}[\sigma_{\text{fb}[t]}(\mathbf{x}_t) | \mathcal{F}_{t-1}] &\geq \mathbb{E}\left[\sigma_{\text{fb}[t]}(\mathbf{x}_t) | \mathcal{F}_{t-1}, \mathbf{x}_t\in\mathcal{X}\setminus S_t\right] \mathbb{P}\left(\mathbf{x}_t\in\mathcal{X}\setminus S_t | \mathcal{F}_{t-1}\right)\\
&\geq \sigma_{\text{fb}[t]}(\overline{\mathbf{x}}_t)(p-1/t^2),
\end{split}
\label{eq:app:sigma:bar:new}
\end{equation}
where the last inequality makes use of Lemma~\ref{lemma:prob_unsaturated:new}.

Next, conditioned on both $E^{\widetilde{f}}(t)$ and $E^{\widetilde{f}_t}(t)$, we have that
\begin{equation}
\begin{split}
r_t &= f(\mathbf{x}^*) - f(\mathbf{x}_t) = f(\mathbf{x}^*) - f(\overline{\mathbf{x}}_t) + f(\overline{\mathbf{x}}_t) - f(\mathbf{x}_t)\\
&\stackrel{(a)}{\leq} \Delta(\overline{\mathbf{x}}_t) + \widetilde{f}_t(\overline{\mathbf{x}}_t) + c_t \widetilde{\sigma}_{\text{fb}[t]}(\overline{\mathbf{x}}_t) + \epsilon_{m,t} - \widetilde{f}_t(\mathbf{x}_t) + c_t \widetilde{\sigma}_{\text{fb}[t]}(\mathbf{x}_t) + \epsilon_{m,t}\\
&\stackrel{(b)}{\leq} c_t \widetilde{\sigma}_{\text{fb}[t]}(\overline{\mathbf{x}}_t) + 2\epsilon_{m,t} + c_t \widetilde{\sigma}_{\text{fb}[t]}(\overline{\mathbf{x}}_t) + c_t \widetilde{\sigma}_{\text{fb}[t]}(\mathbf{x}_t) + 2\epsilon_{m,t} + \widetilde{f}_t(\overline{\mathbf{x}}_t) - \widetilde{f}_t(\mathbf{x}_t)\\
&\stackrel{(c)}{\leq} c_t \left( 2\widetilde{\sigma}_{\text{fb}[t]}(\overline{\mathbf{x}}_t) + \widetilde{\sigma}_{\text{fb}[t]}(\mathbf{x}_t) \right) + 4\epsilon_{m,t}\\
&\stackrel{(d)}{\leq} c_t \left( 2\sigma_{\text{fb}[t]}(\overline{\mathbf{x}}_t) + \sigma_{\text{fb}[t]}(\mathbf{x}_t) \right) + 3 c_t \sqrt{(L+1)\varepsilon \left( 1+\frac{4 \hat{K}_0^2 t^2}{\sigma^4} \right)} + 4\epsilon_{m,t}\\
&= c_t \left( 2\sigma_{\text{fb}[t]}(\overline{\mathbf{x}}_t) + \sigma_{\text{fb}[t]}(\mathbf{x}_t) \right) + \epsilon'_{m,t}.
\end{split}
\end{equation}
$(a)$ follows from Lemmas~\ref{lemma:event:Ef:tilde} and~\ref{lemma:event:Eft:finite:width}, $(b)$ follows from the definition of unsaturated inputs (Definition~\ref{def:saturated_set:new}), $(c)$ results from the way in which $\mathbf{x}_t$ is selected: $\mathbf{x}_t={\arg\max}_{\mathbf{x}\in\mathcal{X}}\widetilde{f}_t(\mathbf{x})$, $(d)$ makes use of Lemma~\ref{lemma:bound:difference:gp:post:std}.

% Inspired by Proposition 1 of~\cite{desautels2014parallelizing}, choose a constant $C$ such that $\sigma_{\text{fb}[t]}(\mathbf{x}) / \sigma_{t-1}(\mathbf{x})=\exp\left(\mathbb{I}\left( f(\mathbf{x});\mathbf{y}_{\text{fb}[t]+1:t-1} | \mathbf{y}_{1:\text{fb}[t]} \right)\right) \leq \exp(C)$. A naive choice for $C$ is simply $\gamma_{B-1}$, but we will use a special initialization scheme, based on uncertainty sampling, to reduce the value of $C$ to be a constant that is independent of the batch size $B$.
Next, we can upper-bound the expected instantaneous regret:
\begin{equation}
\begin{split}
\mathbb{E}\left[ r_t | \mathcal{F}_{t-1} \right] &\leq \mathbb{E}\left[c_t \left( 2\sigma_{\text{fb}[t]}(\overline{\mathbf{x}}_t) + \sigma_{\text{fb}[t]}(\mathbf{x}_t) \right) + \epsilon'_{m,t} | \mathcal{F}_{t-1}\right] + 2B'\mathbb{P}\left(\overline{E^{\widetilde{f}_t}(t)} | \mathcal{F}_{t-1}\right)\\
&\stackrel{(a)}{\leq} \mathbb{E}\left[c_t \left(\frac{2}{p-1/t^2} \sigma_{\text{fb}[t]}(\mathbf{x}_t) + \sigma_{\text{fb}[t]}(\mathbf{x}_t) \right) + \epsilon'_{m,t} | \mathcal{F}_{t-1}\right] + \frac{2B'}{t^2}\\
&= \mathbb{E}\left[c_t \left(1+\frac{2}{p-1/t^2}\right)\sigma_{\text{fb}[t]}(\mathbf{x}_t) + \epsilon'_{m,t} | \mathcal{F}_{t-1}\right] + \frac{2B'}{t^2}\\
&\stackrel{(b)}{\leq} c_t \left(1+\frac{2}{p-1/t^2}\right)\mathbb{E}\left[e^C\sigma_{t-1}(\mathbf{x}_t) | \mathcal{F}_{t-1}\right] + \epsilon'_{m,t} + \frac{2B'}{t^2}\\
&\stackrel{(c)}{\leq} c_t e^C \left(1+\frac{10}{p}\right)\mathbb{E}\left[\sigma_{t-1}(\mathbf{x}_t) | \mathcal{F}_{t-1}\right] + \epsilon'_{m,t} + \frac{2B'}{t^2}.
\end{split}
\end{equation}
$(a)$ follows from equation~\eqref{eq:app:sigma:bar:new}, $(b)$ makes use of Lemma~\ref{lemma:ratio:sigma:C}, and $(c)$ follows since $2 / (p-1/t^2) \leq 10/p$.
This completes the proof.
\end{proof}

We similarly define the following stochastic process, which will be shown to be a super-martingale in the subsequent Lemma.
\begin{definition}
Define $Y_0=0$, and for all $t=1,\ldots,T$,
\[
\overline{r}_t=r_t \mathbb{I}\{E^{\widetilde{f}}(t)\},
\]
\[
X_t = \overline{r}_t - c_t e^C \left(1+\frac{10}{p}\right) \sigma_{t-1}(\mathbf{x}_t) - \epsilon'_{m,t} - \frac{2B'}{t^2}
\]
\[
Y_t=\sum^t_{s=1}X_s.
\]
\end{definition}

\begin{lemma}
\label{lemma:proof:sup:martingale:finite:width}
Conditioned on the event $E^{f}(t)$, $(Y_t:t=0,\ldots,T)$ is a super-martingale with respect to the filtration $\mathcal{F}_t$.
\end{lemma}
The proof of Lemma~\ref{lemma:proof:sup:martingale:finite:width} above follows closely the proof of Lemma~\ref{lemma:proof:sup:martingale} and is hence omitted.

\begin{lemma}
\label{lemma:bound:cum:regret_new}
Define $C_1\triangleq \frac{2}{\log(1+\sigma^{-2})}$.
With probability of $\geq 1-\delta$,
\begin{equation}
R_T \leq c_T e^C \Big(1+\frac{10}{p}\Big)\sqrt{C_1 T \gamma_T} + T\epsilon'_{m,T} + \frac{B'\pi^2}{3} +  \Big(  4B' + c_T e^C \big(1+\frac{10}{p}\big) K_0 + \epsilon'_{m,T} \Big) \sqrt{2T\log(4/\delta)}.
\end{equation}
\end{lemma}
\begin{proof}
To begin with, we have that
\begin{equation}
\begin{split}
|Y_t - Y_{t-1}| &= |X_t| \leq |\overline{r}_t| + c_t e^C \left(1+\frac{10}{p}\right) \sigma_{t-1}(\mathbf{x}_t) + \epsilon'_{m,t} + \frac{2B'}{t^2}\\
&\leq 2B' + c_t e^C \left(1+\frac{10}{p}\right) K_0 + \epsilon'_{m,t} + 2B'\\
&= 4B' + c_t e^C \left(1+\frac{10}{p}\right) K_0 + \epsilon'_{m,t}.
\end{split}
\end{equation}
% \begin{equation}
% \begin{split}
% |Y_t - Y_{t-1}| &= |X_t| \leq c_t e^C \left(1+\frac{10}{p}\right)\sigma_{t-1}(\mathbf{x}_t) + 3\epsilon'_{m,t} + \frac{2B'}{t^2}\\
% &\leq c_t e^C \left(1+\frac{10}{p}\right) + 3\epsilon'_{m,t} + 2B'
% \end{split}
% \end{equation}

Using the Azuma-Hoeffding's inequality with an error probability of $\delta/4$, we have that
\begin{equation}
\begin{split}
\sum^T_{t=1}\overline{r}_t &\leq \sum^T_{t=1} c_t e^C \left(1+\frac{10}{p}\right) \sigma_{t-1}(\mathbf{x}_t) + \sum^T_{t=1}\epsilon'_{m,t} +
\sum^T_{t=1}\frac{2B'}{t^2} + \\
&\qquad \sqrt{2\log(4/\delta) \sum^T_{t=1}\left( 4B' + c_t e^C \left(1+\frac{10}{p}\right) K_0 + \epsilon'_{m,t} \right)^2}\\
&\stackrel{(a)}{\leq} c_T e^C \left(1+\frac{10}{p}\right)\sum^T_{t=1}\sigma_{t-1}(\mathbf{x}_t) + \sum^T_{t=1}\epsilon'_{m,t} + \frac{B'\pi^2}{3} + \\
&\qquad \left( 4B' + c_T e^C \left(1+\frac{10}{p}\right) K_0 + \epsilon'_{m,T} \right) \sqrt{2T\log(4/\delta)}\\
&\stackrel{(b)}{\leq} c_T e^C \left(1+\frac{10}{p}\right)\sqrt{C_1 T \gamma_T} + \sum^T_{t=1}\epsilon'_{m,t} + \frac{B'\pi^2}{3} + \\
&\qquad \left( 4B' + c_T e^C \left(1+\frac{10}{p}\right) K_0 + \epsilon'_{m,T} \right) \sqrt{2T\log(4/\delta)}\\
&\leq c_T e^C \left(1+\frac{10}{p}\right)\sqrt{C_1 T \gamma_T} + T\epsilon'_{m,T} + \frac{B'\pi^2}{3} + \\
&\qquad  \left(  4B' + c_T e^C \left(1+\frac{10}{p}\right) K_0 + \epsilon'_{m,T} \right) \sqrt{2T\log(4/\delta)}.
\end{split}
\end{equation}
% \begin{equation}
% \begin{split}
% \sum^T_{t=1}\overline{r}_t &\leq c_T e^C \left(1+\frac{10}{p}\right)\sum^T_{t=1}\sigma_{t-1}(\mathbf{x}_t) + 3\sum^T_{t=1}\epsilon'_{m,t} + \frac{B'\pi^2}{3} + \\
% &\qquad \sqrt{2\log(1/\delta') \sum^T_{t=1}\left( c_t e^C \left(1+\frac{10}{p}\right) + 3\epsilon'_{m,t} + 2B' \right)^2}\\
% &\leq c_T e^C \left(1+\frac{10}{p}\right)\sum^T_{t=1}\sigma_{t-1}(\mathbf{x}_t) + 3\sum^T_{t=1}\epsilon'_{m,t} + \frac{B'\pi^2}{3} + \\
% &\qquad \left( c_T e^C \left(1+\frac{10}{p}\right) + 3\epsilon'_{m,T} + 2B' \right) \sqrt{2T\log(1/\delta')}\\
% &\leq c_T e^C \left(1+\frac{10}{p}\right)\sqrt{C_1 T \gamma_T} + 3\sum^T_{t=1}\epsilon'_{m,t} + \frac{B'\pi^2}{3} + \\
% &\qquad \left( c_T e^C \left(1+\frac{10}{p}\right) + 3\epsilon'_{m,T} + 2B' \right) \sqrt{2T\log(1/\delta')}\\
% &\leq c_T e^C \left(1+\frac{10}{p}\right)\sqrt{C_1 T \gamma_T} + 3T\epsilon'_{m,T} + \frac{B'\pi^2}{3} +  \left[ c_T e^C \left(1+\frac{10}{p}\right) + 3\epsilon'_{m,T} + 2B' \right] \sqrt{2T\log(\frac{1}{\delta'})}
% \end{split}
% \end{equation}
$(a)$ follows since $c_t$ is increasing in $t$, and $(b)$ follows from the proof of Lemma 5.4 in the work of~\cite{srinivas2009gaussian}. 
Next, note that $\overline{r}_t=r_t,\forall t\geq1$ with probability of $\geq1-\delta/2$ according to Lemma~\ref{lemma:event:Ef:tilde}. 
% Therefore, the upper bound derived in the equation above is an upper bound on $R_T=\sum^T_{t=1}r_t$.
Also recall that throughout the entire proof in this section, we have conditioned on the event in Proposition~\ref{prop:arora}, which also holds with probability of $\geq1-\delta/4$.
Therefore, also taking into account the error probability of $\delta/4$ from the Azuma-Hoeffding's inequality, the upper bound derived above is an upper bound on the cumulative regret $R_T=\sum^T_{t=1}r_t$ with probability of $\geq 1-\delta/2-\delta/4-\delta/4=1-\delta$.
% Lastly, note that the event $E^{f}(t)$ holds with probability of $\geq 1-\delta-\delta_{\mu}$ and we also need to take into account the probability of $\delta_{\text{ntk}}$ from Proposition~\ref{prop:arora}. Therefore, the regret upper bound holds with probability of $\geq 1-\delta-\delta_{\mu}-\delta_{\text{ntk}}-\delta'$.
\end{proof}

Now let's analyze the asymptotic scaling of the regret upper bound derived above. Firstly, note that $c_T = \mathcal{O}(\log^2 T)$.
Next, recall that we have in the main text that $(L+1)\varepsilon = C_{\text{ntk}}(L+1)L^{3/2}\log^{1/4}(4L|\mathcal{X}|^2/\delta)m^{-1/4}$. This allows us to analyze the scaling of $\epsilon'_{m,T}$.
\begin{equation}
\begin{split}
\epsilon'_{m,T} &= 3 c_T \sqrt{(L+1)\varepsilon \left( 1+\frac{4 \hat{K}_0^2 T^2}{\sigma^4} \right)} + 4 \Big(
2\hat{K}_0\frac{T^2(L+1)\varepsilon}{\sigma^4}\left(B' + \sigma\sqrt{2\log(4T/\delta)}\right) + \\
&\qquad \beta_T\sqrt{(L+1)\varepsilon \left( 1+\frac{4 \hat{K}_0^2 T^2}{\sigma^4} \right)}
\Big)\\
&=\widetilde{\mathcal{O}}\left(
(\log T)^2 \sqrt{(L+1)\varepsilon} T + T^2 (L+1)\varepsilon + \log T\sqrt{(L+1)\varepsilon} T
\right)\\
&=\widetilde{\mathcal{O}}\left(
T^2 \sqrt{(L+1) \varepsilon}
\right)\\
&=\widetilde{\mathcal{O}}\left(T^2m^{-1/8}(L+1)^{5/4}\right)
\end{split}
\label{eq:epsilon:prime:m:T}
\end{equation}
% \begin{equation}
% \begin{split}
% \epsilon'_{m,T} &= 2\frac{t^2(L+1)\varepsilon}{\sigma^4}\left[B' + \sigma\sqrt{2\log(T/\delta_{\mu})}\right] + (c_t + \beta_t) \sqrt{(L+1) \varepsilon \left( 1+\frac{4t^2}{\sigma^4} \right)}\\
% &=\widetilde{\mathcal{O}}\left(T^2m^{-1/8}(L+1)^{5/4}\right)
% \end{split}
% \end{equation}

This allows us to analyze the asymptotic scaling of our regret upper bound (ignoring all log factors)
\begin{equation}
\begin{split}
R_T &= \widetilde{\mathcal{O}}\left(
e^C\sqrt{T\gamma_T} + T\epsilon'_{m,T} + (e^C+\epsilon'_{m,T})\sqrt{T}
\right)\\
&= \widetilde{\mathcal{O}}\left(e^C \sqrt{T} (\sqrt{\gamma_T}+1) + T\epsilon'_{m,T} + \sqrt{T} \epsilon'_{m,T} \right)\\
&= \widetilde{\mathcal{O}}\left(e^C \sqrt{T} (\sqrt{\gamma_T}+1) + T^3m^{-1/8}(L+1)^{5/4}\right).
\end{split}
\end{equation}

% From Lemma~\ref{lemma:bound:cum:regret}, we have that
% \begin{equation}
% R_T = \mathcal{O}\left(e^C \log^2T \sqrt{T} \left(1+\sqrt{\gamma_T}\right)\right) = \widetilde{\mathcal{O}}\left(e^C \sqrt{T} \left(1+\sqrt{\gamma_T}\right)\right).
% \end{equation}

% \section{Analysis for the RKHS Assumption}
% \label{app:proof:rkhs:assumption}

\section{Extension to Continuous Input Domains}
\label{app:sec:extension:continuous:domain}
To extend our theoretical results to cases where the input domain $\mathcal{X}$ is continuous, we can follow the techniques discussed in Section 3.1 of the work of \cite{li2021gaussian}.
We assume that $\mathcal{X}\subset[0,1]^d$.
To begin with, we need to additionally assume that the objective function $f$ is Lipschitz continuous with a Lipschitz constant $L>0$.
Next, we can construct a finite sub-domain $\widetilde{\mathcal{X}}$ of the continuous domain $\mathcal{X}$, where $\widetilde{X}$ has equal spacing of $\frac{1}{\sqrt{T}}$ in each dimension.
As a result, the finite sub-domain $\widetilde{\mathcal{X}}$ contains $T^{d/2}$ points, i.e., $|\widetilde{\mathcal{X}}|=T^{d/2}$.
Then, we can simply run our algorithms (Algo.~\ref{algo:1} and Algo.~\ref{algo:2}) on this finite sub-domain $\widetilde{X}$.
% , after which all our theoretical results immediately hold with $|\widetilde{X}|=T^{d/2}$.

As a consequence, for \textbf{Theorem~\ref{theorem:regret:exact:ntk}}, we only need to make two changes to our theoretical results.
Firstly, we need to modify $\beta_t$ to be $\beta_t = 2\log(\pi^2t^2|\widetilde{\mathcal{X}}|/(3\delta))=2\log(\pi^2t^2 T^{d/2} / (3\delta))$, which will only introduce an additional dependence on $\mathcal{O}(d\log T)$ into $c_T$ (Appendix \ref{app:sec:infinite:width}) and hence \emph{an additional multiplicative factor of} $\mathcal{O}(d\log T)=\widetilde{\mathcal{O}}(d)$ into the regret upper bound in Theorem~\ref{theorem:regret:exact:ntk}.
% (since the $\log T$ is hidden by $\widetilde{\mathcal{O}}$).
Secondly, due to the Lipschitz continuity of $f$ and the fact that every input $\mathbf{x}\in\mathcal{X}$ has a neighbor in the finite sub-domain $\widetilde{\mathcal{X}}$ whose distance to it is less than $\frac{d}{\sqrt{T}}$, we have that $f(\mathbf{x}^*) \leq \max_{\widetilde{\mathbf{x}}\in\widetilde{\mathcal{X}}}f(\widetilde{\mathbf{x}}) + \mathcal{O}(\frac{Ld}{\sqrt{T}})$.
As a result, this will introduce \emph{an additional additive term of} $\mathcal{O}(T \times \frac{Ld}{\sqrt{T}})=\mathcal{O}(Ld \sqrt{T})$ to the final upper bound on the cumulative regret.

For \textbf{Proposition \ref{prop:arora}}, an additional multiplicative factor of $d\log T$ will be introduced into the condition on $m$.
For \textbf{Theorem \ref{theorem:regret:approc:ntk}}, to begin with, same as the analysis of Theorem \ref{theorem:regret:exact:ntk} in the paragraph above, \emph{an additional additive factor of} $\mathcal{O}(Ld \sqrt{T})$ will be introduced, and \emph{an additional multiplicative factor of} $\mathcal{O}(d\log T)=\widetilde{\mathcal{O}}(d)$ will be introduced into the first term in Theorem \ref{theorem:regret:approc:ntk}. Moreover, as a result of the additional factor of $d\log T$ in the condition on $m$ (Proposition \ref{prop:arora}), an additional multiplicative factor of $d\log T$ will also be introduced into the approximation quality of $(L+1)\varepsilon$ (Sec.~\ref{subsec:theory:finite}).
As a result, in the proof of Theorem \ref{theorem:regret:approc:ntk} (Appendix \ref{app:sec:finite:width}), \emph{an additional multiplicative factor of} $\sqrt{d}$ will be introduced into the term $\epsilon'_{m,T}$ (see \eqref{eq:epsilon:prime:m:T}) and hence into the second term in the upper bound in Theorem \ref{theorem:regret:approc:ntk}.

Of note, the modified results discussed above do not affect the scaling of our theoretical results (Theorem \ref{theorem:regret:exact:ntk} and Theorem \ref{theorem:regret:approc:ntk}) in $T$ (we ignore all dependencies on $\log T$), because the only additional term depending on $T$ for both theorems is an additive term of $\widetilde{\mathcal{O}}(\sqrt{T})$.

\section{More Experimental Details}
\label{app:sec:experimental:details}
In all experiments, for simplicity, we set $\beta_t=1,\forall t\geq1$, which is consistent with many previous papers on BO which have found the theoretical values of $\beta_t$ to be overly conservative~\cite{srinivas2009gaussian}.
For fair comparisons, in every experiment, all methods under comparison use the same set of initial inputs which are selected by random search.
We use the ERF activation function in the synthetic experiment (Sec.~\ref{sec:exp:synth}) because the synthetic function we have adopted is very smooth.
In all real-world experiments (Secs.~\ref{subsec:exp:automl},~\ref{subsec:exp:rl} and~\ref{sec:exp:images}), we use the ReLU activation function since it has been found to be very effective in modeling complicated real-world functions.
% For simplicity, we use random search to choose the initial input queries instead of using uncertainty sampling method discussed in Sec.~\ref{subsec:theory:infinite}.

For all methods under comparison, when maximizing the acquisition function to choose an input query, if the domain is discrete, we simply evaluate the acquisition function value at every input in the domain and then choose the input that maximizes it.
When the domain is continuous,
we firstly use random search to randomly sample $10,000$ inputs in the domain to evaluate their acquisition function values, and then use L-BFGS-B with $100$ random restarts to refine the search.
When the domain is mixed (i.e., consisting of both continuous and discrete inputs), we treat it as a continuous domain and after finding the input the maximizes the acquisition function, we round the discrete inputs to the nearest integer.
For Neural UCB \cite{zhou2020neural}, we treat the UCB value calculated for each arm (input) as the acquisition function; for Neural TS, we treat the reward sampled for each arm (input) as the acquisition function \cite{zhang2020neural}.

We implement the training of the surrogate model $f^{i}_t(\mathbf{x};\theta)$ for both Algos.~\ref{algo:1} and~\ref{algo:2} based on the implementations from the work of~\cite{he2020bayesian}, and we adopt all their default parameter settings (refer to the implementations of~\cite{he2020bayesian} for the specific parameter settings, available at \url{https://github.com/bobby-he/bayesian-ntk}) and only vary the architecture of the NN surrogate model (e.g., the depth and width of the NN, we replace the NN with a CNN for our experiments in Sec.~\ref{sec:exp:images}) as we have mentioned in the main text.
For both Neural UCB and Neural TS, we adopt the implementations from the work of~\cite{zhang2020neural}, use all their default parameter settings, and only modify the architecture of their NN surrogate model for a fair comparison with our methods. As we have mentioned in the main text, to apply Neural UCB and Neural TS for problems with continuous domains, we adapt their implementations such that we optimize their acquisition functions in the same way as our methods (i.e., through a combination of random search and L-BFGS-B as discussed above).
Our experiments are performed using a computing cluster where each machine has an NVIDIA A100 GPU and 96 CPUs.

\subsection{Synthetic Experiments}
In the synthetic experiment, the objective function $f$ is sampled from a GP with an SE kernel using a lengthscale of $0.1$. The domain of $f$ is a uniform grid of size $1,000$ in the interval of $[0,1]$.

\subsection{Real-world Experiments on Automated ML}
\label{app:subsec:exp:automl}
In this section, we give more details on the three hyperparameter tuning experiments in Sec.~\ref{subsec:exp:automl}.

\paragraph{Hyperparamter Tuning of Random Forest.}
Here we tune $6$ categorical hyperparameters of random forest: 
\begin{itemize}
    \item the maximum depth of any individual tree (integer within $[1,10]$), 
    \item the minimum number of samples required to split an internal node (integer within $[2,10]$), 
    \item the minimum number of samples required to be at a leaf node (integer within $[1,10]$), 
    \item the maximum number of features to consider when looking for the best split (integer within $[1,8]$), 
    \item the criterion to measure the quality of a split (binary, "entropy" or "gini"), 
    \item whether bootstrap samples are used when building trees (binary, True or False).
\end{itemize}
We use the publicly available diabetes prediction dataset which can be accessed from~\url{https://www.kaggle.com/uciml/pima-indians-diabetes-database} and has the CC0 License.
This dataset does not contains personally identifiable information or offensive content.
It consists of 768 data instances, each containing $8$ input features.
We use $70 \%$ of the dataset as the training set and the remaining $30\%$ as the validation set. We use random search to choose $5$ initial inputs as the set of initialization, which is shared among all methods under comparison.

\paragraph{Hyperparameter Tuning of XGBoost.}
The MNIST dataset is publicly available and associated with the GNU General Public License, and can be obtained from the Keras Package\footnote{\url{https://keras.io/}}.
It does not contain personally identifiable information or offensive content.
In this experiment, we use the MNIST dataset to tune 9 hyperparameters of XGBoost~\cite{chen2016xgboost}: 
\begin{itemize}
    \item gamma which represents the minimum loss reduction required to make a further partition on a leaf node of the tree (continuous, $[0,10]$), 
    \item the learning rate (continuous, $[10^{-6}, 1]$), 
    \item the maximum depth of any individual tree (integer, $[1,15]$), 
    \item which booster to use (binary, "dart" or "gbtree"), 
    \item the grow policy which controls the way new nodes are added to the tree (binary, "depthwise" or "lossguide"), 
    \item the objective (bianry, "multi:softprob" or "multi:softmax"), 
    \item the tree construction method (binary, "exact" or "hist"), 
    \item alpha which is the L1 regularization term on the weights (continuous, $[0,10]$), and 
    \item lambda which is the L2 regularization term on the weights (continuous, $[0,10]$).
\end{itemize}

\paragraph{Hyperparameter Tuning of Convolutional Neural Networks.}
Here we use the MNIST dataset to tune 9 hyperparameters of convolutional neural networks (CNN). The CNN consists of one convolutional layer, followed by a max pooling layer and subsequently a fully connected layer. The 9 hyperparameters are:
\begin{itemize}
    \item the learning rate (continuous, $[10^{-4}, 0.1]$),
    \item the weight decay (continuous, $[10^{-6}, 10^{-2}]$),
    \item the batch size (integer, $[64, 512]$),
    \item the max pooling size (integer, $[3,5]$),
    \item the number of neurons in the convolutional layer (integer, ${4, 16}$),
    \item the size of the convolutional kernel (integer, $[3, 5]$),
    \item the number of neurons in the fully connected layer (integer, $[4, 16]$),
    \item which activation function to use (binary, ReLU or Tanh),
    \item which optimization method to use (binary, ADAM or RMSprop).
\end{itemize}

\begin{figure}[t]
     \centering
     \begin{tabular}{cc}
         \includegraphics[width=0.40\linewidth]{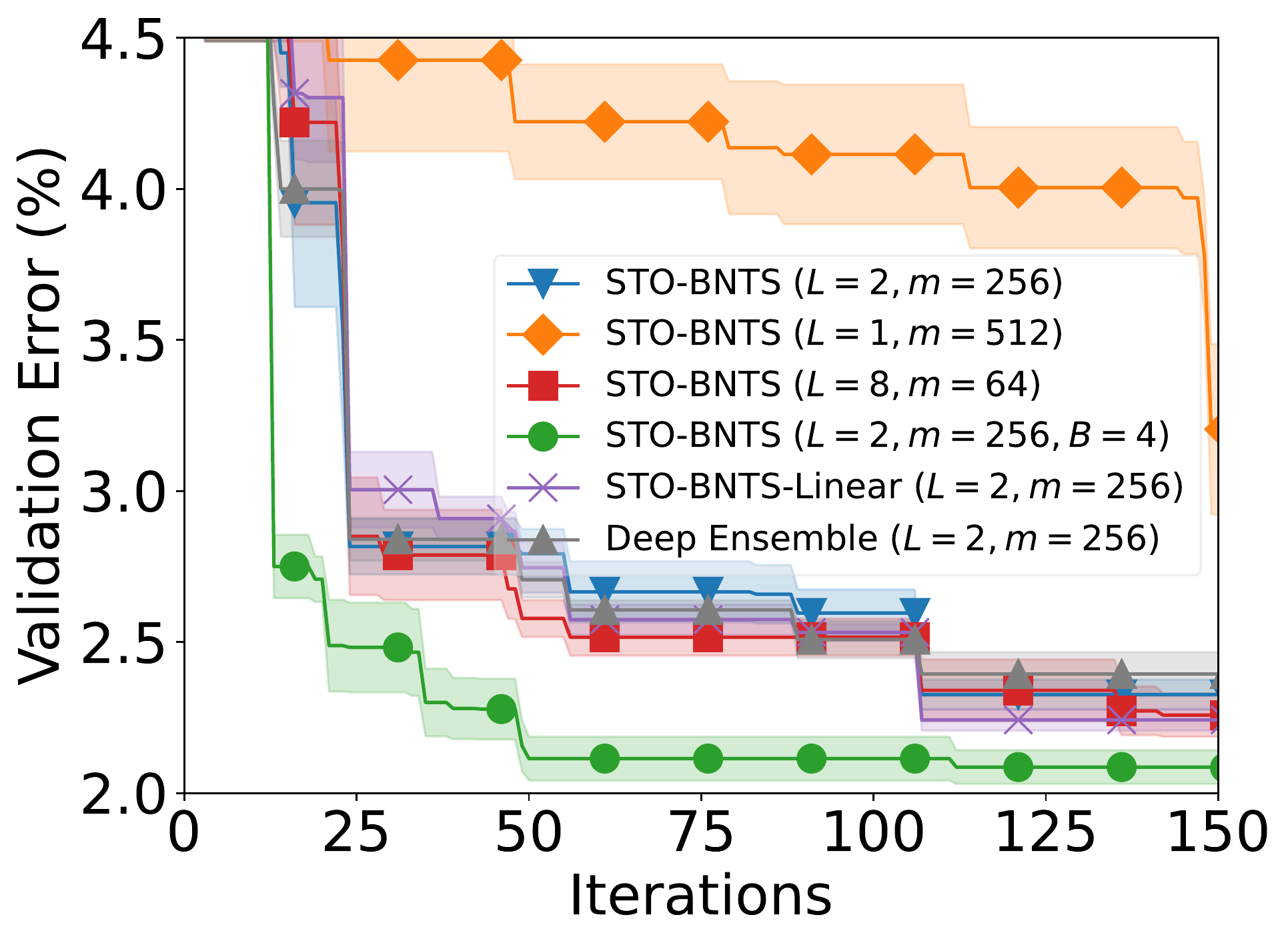} & \hspace{-4mm}
         \includegraphics[width=0.40\linewidth]{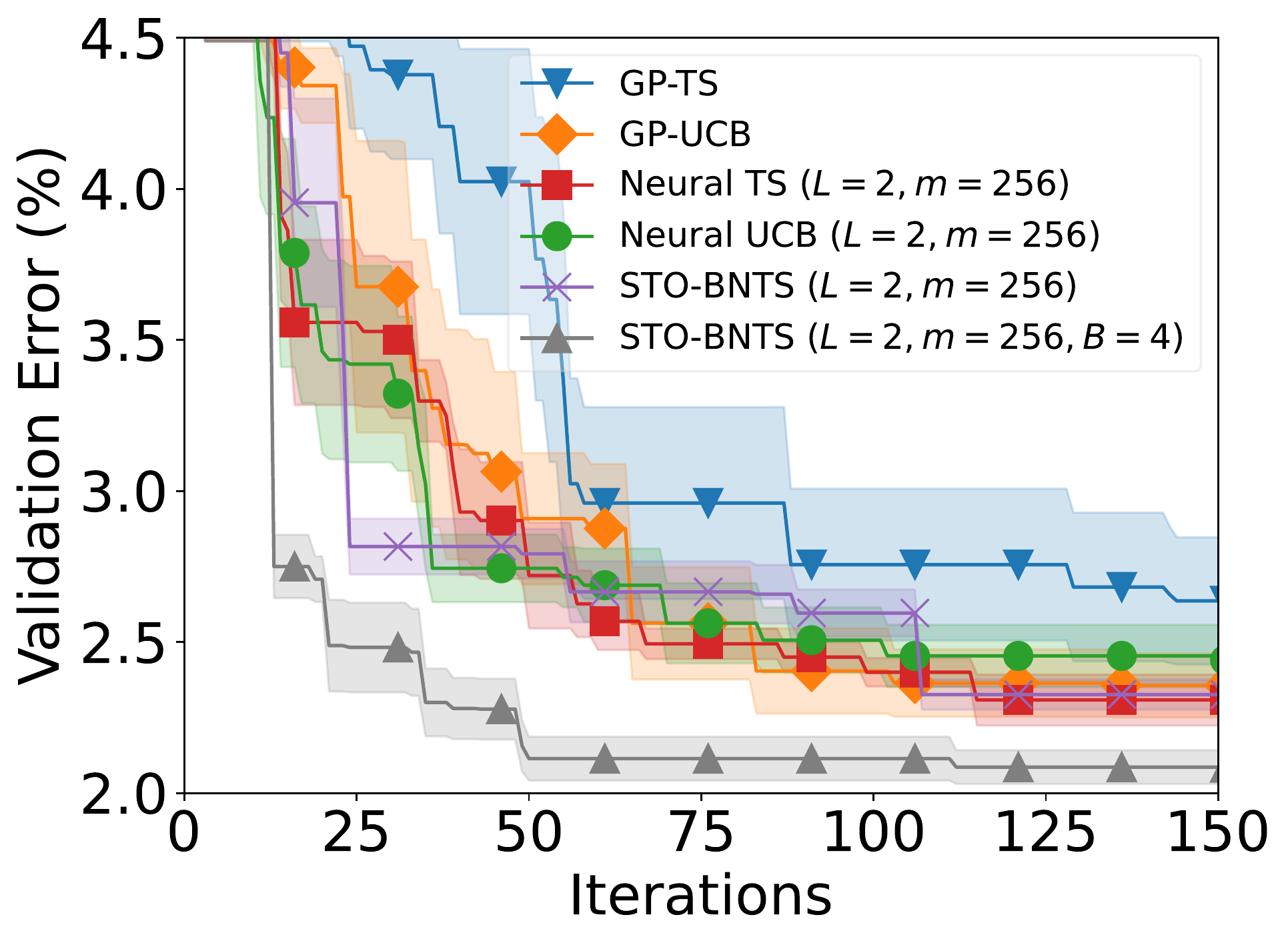}\\
         {(a)} & {(b)}
     \end{tabular}
% \vspace{-3mm}
     \caption{
    %  Results for AutoML experiments. 
     Validation errors for hyperparameter tuning of CNN. $B=1$ unless specified otherwise.}
     \label{fig:exp:automl:cnn}
% \vspace{-0.1in}
\end{figure}

\subsection{Real-world Experiments on Reinforcement Learning}
The lunar lander task involves tuning $12$ parameters of a heuristic controller which is used to control the LunarLander-v2 environment from OpenAI Gym~\cite{brockman2016openai}. The heuristic controller is provided by OpenAI Gym and can be found at \url{https://github.com/openai/gym/blob/8a96440084a6b9be66b2216b984a1c170e4a061c/gym/envs/box2d/lunar_lander.py#L447}.
OpenAI Gym\footnote{\url{https://github.com/openai/gym}} is open-sourced and under the MIT License.
The (14-dimensional) robots pushing and (20-dimensional) rover trajectory planning tasks were firstly introduced by the work of~\cite{wang2018batched} where the detailed experimental settings can be found.
% , and have both also been adopted by other previous works on BO~\cite{eriksson2019scalable}.
Both tasks are publicly available at \url{https://github.com/zi-w/Ensemble-Bayesian-Optimization} and are under the MIT license.
Due to the large number of iterations ($500$) of these three experiments (which is necessary as a result of the high dimensionality of the input spaces), standard GP-UCB and GP-TS become too computationally costly to run. Therefore, we applied random Fourier features approximations \cite{dai2020federated,dai2021differentially} to the GP using a large number $1,000$ of random features, with which GP-UCB and GP-TS perform well and are still computationally feasible to run.

Using the Lunar-Lander experiment, we have also compared our STO-BNTS and STO-BNTS-Linear with batch versions of GP-TS and Neural TS. The results are shown in Fig.~\ref{fig:app:exp:lunar:compare:batch:methods:and:flatten:batch} (a), in which all methods use the same batch size of $B=4$. The figure shows that our STO-BNTS and STO-BNTS-Linear are still able to significantly outperform the other baseline methods when the same batch size is used.

We also use the Lunar-Lander experiment to show an alternative visualization of the performances of our algorithms with batch evaluations in Fig.~\ref{fig:app:exp:lunar:compare:batch:methods:and:flatten:batch} (b). Specifically, the horizontal axis in Fig.~\ref{fig:app:exp:lunar:compare:batch:methods:and:flatten:batch} (b) is the number of function evaluations, in contrast to iterations in Fig.~\ref{fig:exp:rl}a. Same as Fig.~\ref{fig:exp:rl}a, this figure also shows the benefit of batch evaluations, because compared with the sequential algorithms ($B=1$, purple and yellow curves),  our algorithms with a batch size of $B=4$ only suffer slight degradations of the per-function evaluation performances.

\begin{figure}[t]
     \centering
     \begin{tabular}{cc}
         \includegraphics[width=0.40\linewidth]{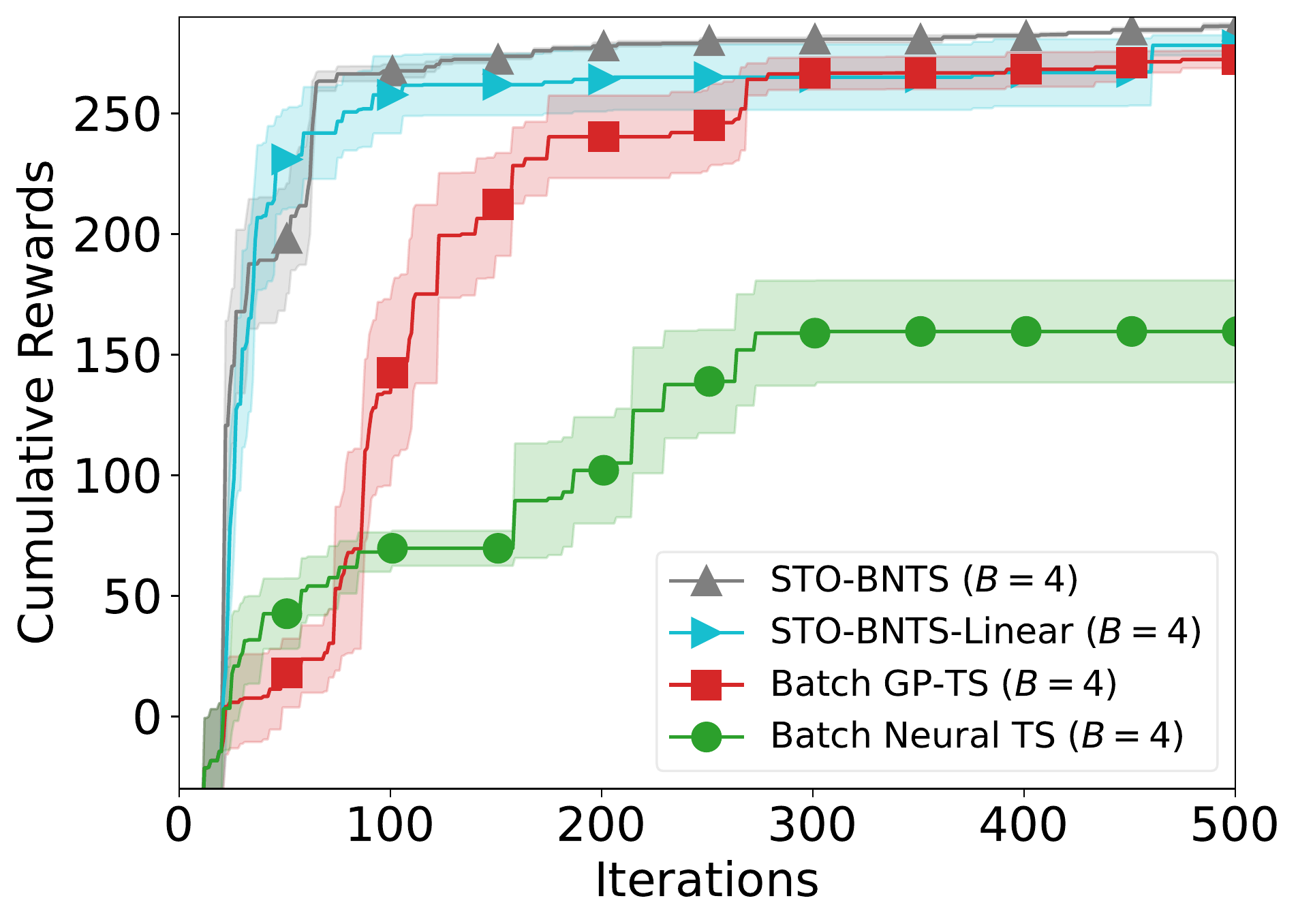} & \hspace{-4mm}
         \includegraphics[width=0.40\linewidth]{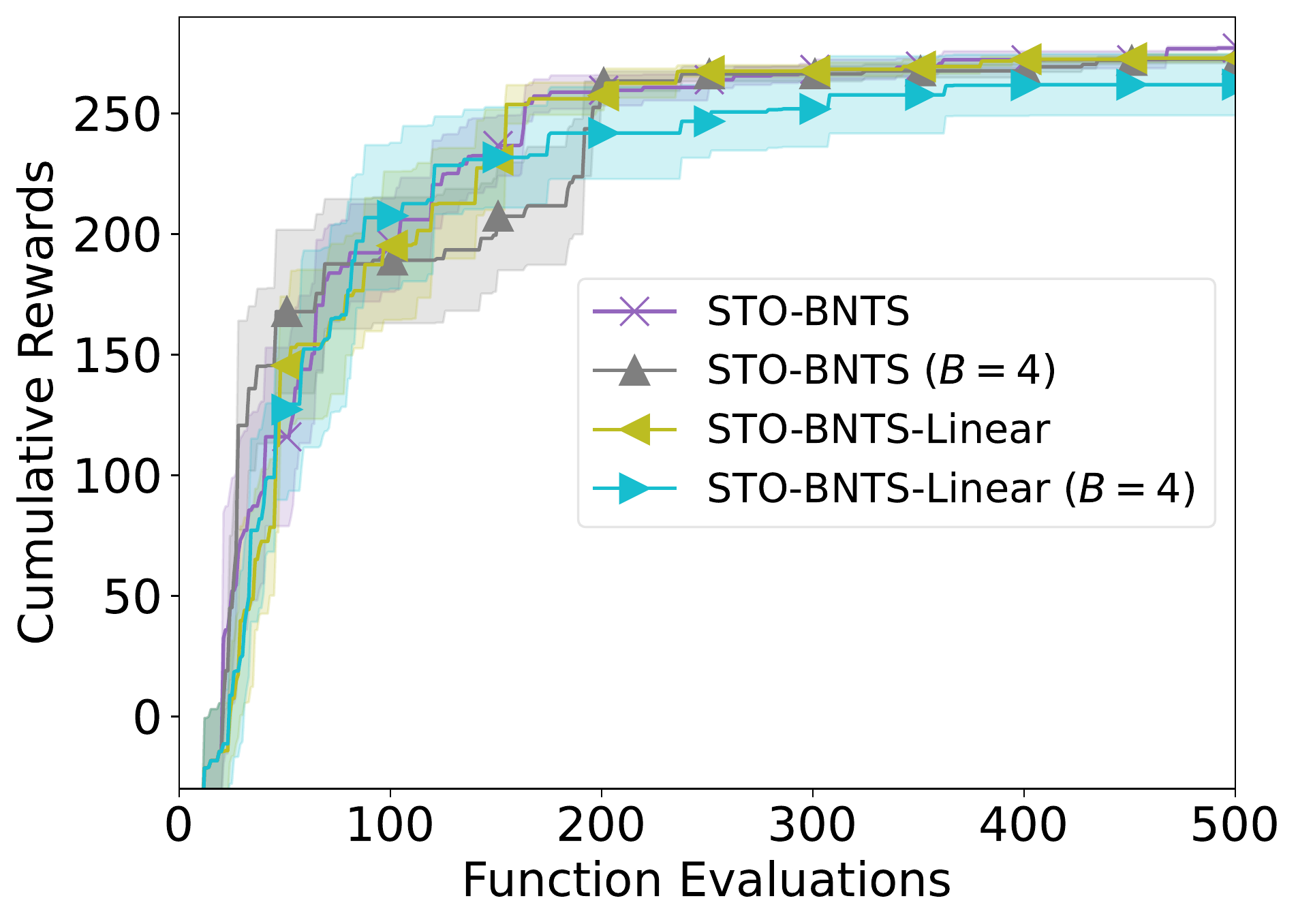}\\
         {(a)} & {(b)}
     \end{tabular}
% \vspace{-3mm}
     \caption{
(a) Comparison of different algorithms with the same batch size of $B=4$, including batch versions of our STO-BNTS and STO-BNTS-Linear, as well as batch GP-TS and batch Neural TS.
(b) An alternative visualization of the performance of our algorithms with batch evaluations, using the 12-D Lunar-Lander experiment (Fig.~\ref{fig:exp:rl}a). The horizontal axis here is the number of function evaluations, in contrast to iterations in Fig.~\ref{fig:exp:rl}a.
}
     \label{fig:app:exp:lunar:compare:batch:methods:and:flatten:batch}
% \vspace{-0.1in}
\end{figure}

\subsection{Optimization over Images}
\label{app:sec:exp:images}
In this experiment, to construct the score function (Fig.~\ref{fig:exp:rl}d), we firstly use the training set of the MNIST dataset (consisting of $60,000$ images) to train a CNN, and then use the trained CNN to predict the class probabilities for the $10$ different classes using the testing set of $10,000$ images.
Next, we use the predicted probability of class $0$ for the $10,000$ testing images as the score function.
As a result of our construction of the score function, similar images in general have similar score values since they share similar representations from the CNN, and images of $0$ overall have much larger scores than images from the other classes.
This can simulate the real-world scenario of image recommender system, in which the user may prefer a certain type of images and hence give higher ratings to them.

For all three CNN-based methods in Fig.~\ref{fig:exp:rl}d, we use a CNN with one convolutional layer (with convolutional kernels size of $3$), followed by a max pooling layer (with a pooling size of $3$), and then followed by a fully connected layer. We have used the ReLU activation function. The width of both the convolutional and fully connected layers are $m=64$.
Fig.~\ref{fig:app:exp:images} plots the results for STO-BNTS-Linear in this experiment, which shows that its performance can be dramatically improved if we increase the width $m$ of the NN surrogate model (Sec.~\ref{sec:exp:images}).

\begin{figure}
     \centering
     \begin{tabular}{c}
         \includegraphics[width=0.4\linewidth]{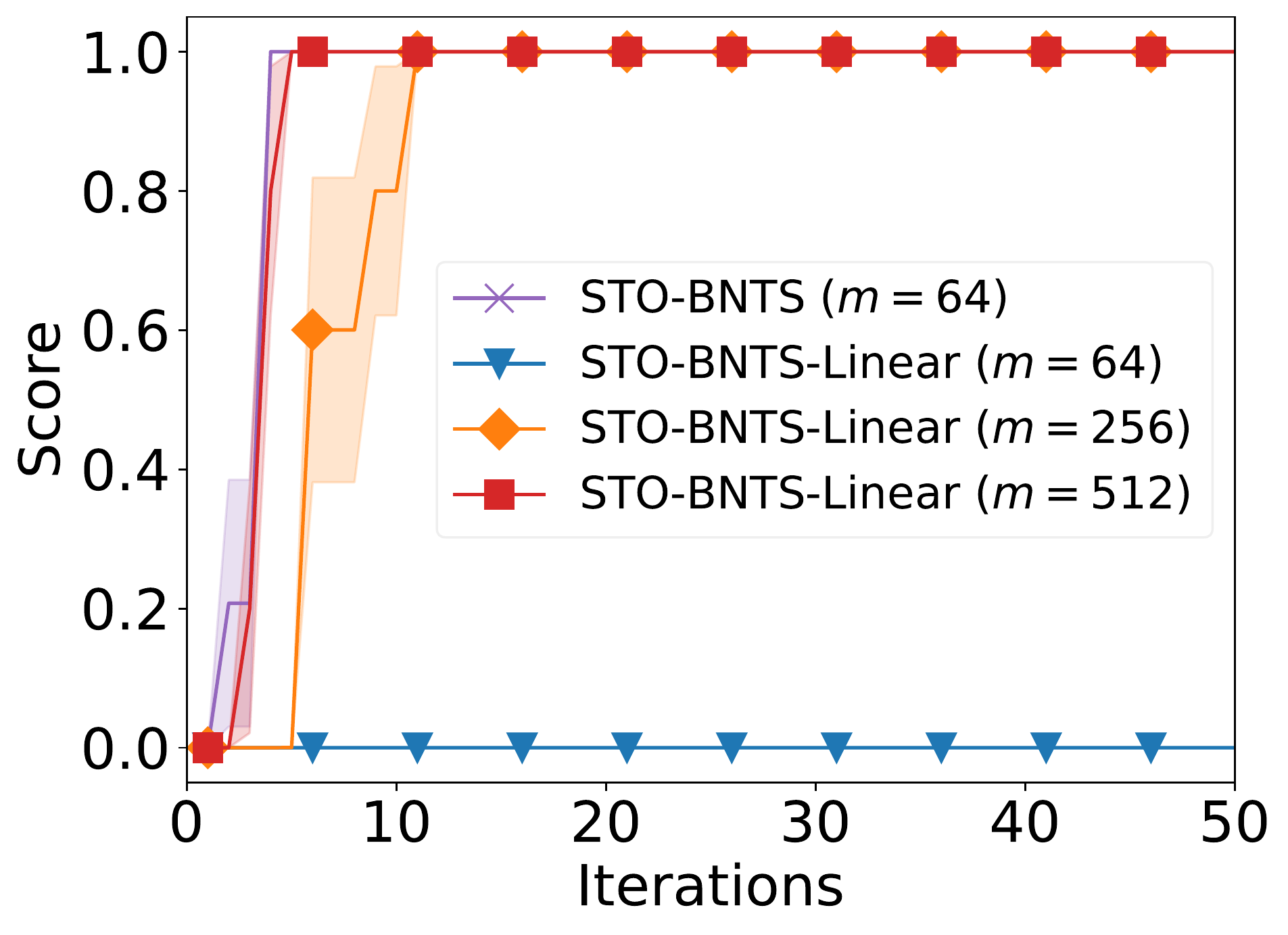}
     \end{tabular}
     \caption{
    Results of STO-BNTS-Linear in the experiments on optimization over images (Sec.~\ref{sec:exp:images}).
    %  Results for problems with continuous (or mixed) domains.
    %  (a,b) Validation errors for hyperparameter tuning of CNN on MNIST. (c,d) Cumulative rewards for the $12$-D Lunar-Lander task. (e,f) Rewards for the $14$-D robot pushing task.
     }
     \label{fig:app:exp:images}
% \vspace{-5mm}
\end{figure}

% \begin{figure}
%      \centering
%      \begin{tabular}{c}
%          \includegraphics[width=0.4\linewidth]{}
%      \end{tabular}
%      \caption{
% Comparison of different algorithms with the same batch size of $B=4$, including batch versions of our STO-BNTS and STO-BNTS-Linear, as well as batch GP-TS and batch Neural TS.
%      }
%      \label{fig:app:exp:lunar:compare:batch:methods}
% % \vspace{-5mm}
% \end{figure}

% \begin{figure}
%      \centering
%      \begin{tabular}{c}
%          \includegraphics[width=0.4\linewidth]{}
%      \end{tabular}
%      \caption{
%      An alternative visualization of the performance of our algorithms with batch evaluations, using the 12-D Lunar-Lander experiment (Fig.~\ref{fig:exp:rl}a).
%      The horizontal axis here is the number of function evaluations, in contrast to iterations in Fig.~\ref{fig:exp:rl}a.
%      }
%      \label{fig:app:lunar:flatten:batch}
% \end{figure}

\end{document}